\documentclass{article}
\DeclareUnicodeCharacter{FF0C}{,}
\usepackage[margin=1in]{geometry}
\usepackage{xcolor}
\usepackage{booktabs}
\usepackage{tabularx}
\usepackage{makecell}
\usepackage{array}
\newcommand{\rev}[1]{\textcolor{black}{#1}}

\newcommand{\revII}[1]{\textcolor{black}{#1}}
\usepackage{rotating}
\usepackage{array}
\usepackage{makecell}
\usepackage{tikz}
\usetikzlibrary{backgrounds,arrows.meta,positioning,fit}

\usepackage{titlesec}
\titleformat*{\section}{\large\bfseries}
\titleformat*{\subsection}{\bfseries}
\titleformat*{\subsubsection}{\itshape}
\usepackage{threeparttable}
\usepackage{authblk}

\usepackage{lineno}

\usepackage{array}
\newcolumntype{P}[1]{>{\centering\arraybackslash}p{#1}}

\usepackage{amsmath, amssymb, amsthm}
\usepackage{mleftright}
\usepackage{bm, bbm}
\usepackage{mathtools}

\newtheorem{proposition}{Proposition}

\newtheorem{definition}{Definition}

\newtheorem{example}{Variant}

\newtheorem{property_nl}{NL Property}
\newtheorem{property_Nested_Alt-GNN}{Nested Alt-GNN Property}

\usepackage{graphicx}
\usepackage{booktabs, float, multirow, threeparttable}
\usepackage[skip=5pt]{caption}
\usepackage[skip=5pt]{subcaption}
\usepackage{array}
\usepackage[detect-all]{siunitx}
\newrobustcmd{\B}{\DeclareFontSeriesDefault[rm]{bf}{b}\bfseries}

\usepackage{enumitem}
\setlist{nosep}

\usepackage{soul}
\usepackage{xcolor}

\usepackage{comment}
\usepackage{microtype}

\usepackage[title]{appendix}

\usepackage[colorlinks, allcolors=blue]{hyperref}
\usepackage[nameinlink, capitalise]{cleveref}
\usepackage{natbib}

\title{Alternative Graph Neural Networks: \rev{Bridging} GEV Models and Deep Learning for Travel Mode Choice Modeling}

\author[a]{Yuqi Zhou}
\author[a]{Zhanhong Cheng}
\author[d]{Dingyi Zhuang}
\author[b]{Lingqian Hu}
\author[c]{Yuheng Bu}
\author[a,*]{Shenhao Wang}

\affil[a]{Department of Urban and Regional Planning, University of Florida}
\affil[b]{College of Architecture, Texas A\&M University}
\affil[c]{Department of Computer Science, University of California, Santa Barbara}
\affil[d]{Department of Civil and Environmental Engineering, Massachusetts Institute of Technology}

\date{}

\begin{document}
\allowdisplaybreaks
\maketitle
\vspace*{-2em}

\begin{abstract}
\noindent Generalized extreme value (GEV) models capture dependence among choice alternatives in discrete choice modeling, but require this dependence to be predefined, symmetric, and shared uniformly across individuals. Recent efforts to connect discrete choice models with deep neural networks have improved predictive performance but still cannot explicitly represent alternative dependence within neural architectures. To address these gaps, we introduce the alternative graph - a graph in which nodes represent choice alternatives and edges encode their dependence - and propose Alternative Graph Neural Networks (Alt-GNNs), a family of GNN-based discrete choice models. Theoretically, Alt-GNNs incorporate multinomial logit, nested logit, and ASU-DNN as special cases and enable innovative model designs, including Nested Alt-GNN, Top-K Alt-GNN, Complete Alt-GNN, and Attention Alt-GNN. Alt-GNNs are consistent with random utility maximization theory, \rev{enable behavior-informed graph design}, and offer a novel graph-based interpretation of utility functions similar to Mother Logit. Empirically, on two travel mode choice datasets from London and Chicago, Alt-GNNs significantly improve predictive performance over all benchmark models in mode choice modeling because of their flexible alternative graph design and vast hyperparameter space. Even the simplest Alt-GNN variant - Nested Alt-GNN - generalizes the nested logit model while preserving its unique two-layer substitution properties, enabling \rev{graph-based model design} over otherwise unconstrained behavioral patterns from deep neural networks. In addition, Alt-GNNs also lead to graph-based, individualized, and asymmetric interpretation of alternative dependence, transcending pre-defined nesting structures dominant in GEV models. Overall, this study provides a framework that bridges classical GEV models with modern GNNs, advancing both predictive accuracy and behavioral interpretation in travel mode choice modeling.
\end{abstract}


\thispagestyle{empty}
\clearpage

\setcounter{page}{1}

\section{Introduction}
Discrete choice models (DCMs) provide a powerful theoretical framework to explain individual decisions and serve as the foundation for analyzing travel mode choice in transportation planning \citep{ben-akivaMultinomialChoice1985, koppelmanTRAVELCHOICEBEHAVIORMODELS1980, trainAutoOwnershipUse1986}. Among the rich DCM families, the generalized extreme value (GEV) framework encodes dependence among alternatives through generating functions, yielding substitution patterns richer than those permitted by the independence of irrelevant alternatives (IIA) constraint \citep{benAkivaFrancois1983, manskiStructureRandomUtility1977, ben-akivaDiscreteChoiceMethods1999, mcfaddenMixedMNLModels2000}. It is essential to leverage this alternative dependence to improve predictive performance and model interpretation, and misspecified dependence structures distort market shares, elasticities, or other policy implications \citep{wangDeepNeuralNetworks2020,hanNeuralembeddedDiscreteChoice2022,vancranenburghArtificialNeuralNetwork2019}. More recently, researchers have begun synergizing DCMs and deep neural networks (DNNs)  by designing innovative model architectures \citep{wangDeepNeuralNetworks2020a, hanNeuralembeddedDiscreteChoice2022,wongResLogitResidualNeural2021} or imposing behavioral constraints \citep{fengDeepNeuralNetworks2024, kimNewFlexiblePartially2024}. These synergistic approaches demonstrate improved predictive performance by automating utility specification without requiring analysts to prespecify functional forms \citep{vancranenburghArtificialNeuralNetwork2019}.

\rev{However, existing approaches have not connected neural architectures with the broader GEV family by explicitly capturing the dependence among alternatives.} On one hand, researchers have leveraged DNNs for flexible utility specification, e.g., DNN architecture with alternative-specific utility functions (ASU-DNN), Residual Logit (ResLogit), Multinomial Logit with a Taste Network (TasteNet-MNL) and Embeddings Learning Multinomial Logit (EL-MNL). But these models still treat alternatives as largely independent inputs: any inter-alternative dependence is absorbed into shared feedforward weights without structural controls. As a result, alternative dependence is either hidden in the feedforward structure without identifying the substitution mechanism \citep{hanNeuralembeddedDiscreteChoice2022, wangDeepNeuralNetworks2021, wongResLogitResidualNeural2021, sifringerEnhancingDiscreteChoice2020,arkoudiCombiningDiscreteChoice2023}, or captured only under the overly restrictive IIA constraint, as in ASU-DNN \citep{wangDeepNeuralNetworks2020a}. On the other hand, GEV models routinely exploit alternative dependence to improve performance and capture rich substitution patterns, with examples including nested logit (NL), cross-nested logit, generalized nested logit, paired combinatorial logit, and spatially correlated logit \citep{mcfaddenModelingChoiceResidential1978, wenGeneralizedNestedLogit2001, ben-akivaDiscreteChoiceMethods1999, bhatMixedSpatiallyCorrelated2004}. However, in such GEV models, the alternative dependence is mostly specified a priori, shared uniformly across all individuals, constrained to symmetric interactions, and limited by the form of the chosen generating function. This landscape motivates the design of frameworks that further connect these two paradigms by endowing DNNs with explicit, structurally controlled alternative dependence while extending GEV models with data-driven, flexible, and individualized dependence structures.

To achieve this goal, this study introduces the \textbf{alternative graph}, a graph in which nodes represent choice alternatives and edges encode their dependence, as a new representation of the choice set. We propose \textbf{Alternative Graph Neural Networks} (Alt-GNNs), a family of GNN-based discrete choice models that embed alternative dependence to enable controls over behavioral patterns from DNNs. The Alt-GNN framework achieves significant expressive power through a rich hyperparameter space. It incorporates multinomial logit (MNL), NL, and ASU-DNN as special cases and extends to new variants including Nested Alt-GNN, Top-K Alt-GNN, Complete Alt-GNN, and Attention Alt-GNN. We evaluate all variants on travel mode choice datasets from London Passenger Mode Choice (LPMC) \citep{hillelRecreatingPassengerMode2018} and Chicago Metropolitan Agency for Planning (CMAP)  \citep{HouseholdTravelSurvey}, exploring 370 configurations across graph structure, message function, aggregation function, and readout function. \rev{Overall, this study develops Alt-GNN as a framework that bridges GEV and GNN models through alternative graph representations and GNN algorithms.} Together, this study addresses four related research questions:

\begin{itemize}
    \item \textbf{RQ1 - Theoretical Framework}:  \rev{Theoretically, how do Alt-GNNs provide a bridge between GEV and GNN models while enabling new model designs and behavioral interpretations grounded in the random utility maximization (RUM) framework?}
    
    \item \textbf{RQ2 - Predictive Power}: Empirically, to what extent do Alt-GNNs improve the predictive performance of benchmark DCMs, and what mechanisms explain these improvements?
    
    \item \rev{\textbf{RQ3 - Behavior-Informed Graph Design}: How can alternative graph structures preserve desirable elasticity and substitution properties from classical discrete choice models while improving predictive performance?}
    
    \item \rev{\textbf{RQ4 - Graph Interpretation}: How can predefined and learned alternative graph representations be interpreted in Alt-GNNs?
}
\end{itemize}

The remainder of this paper is organized to address the four research questions. \cref{sec:literature_review} presents the research landscape by reviewing the literature on classical DCMs, deep learning for choice modeling, and graph neural networks. \cref{sec:methodology} answers RQ1 by introducing the alternative graph concept, defining the Alt-GNN framework, presenting its key variants, and establishing their theoretical properties. \cref{sec:experiment_setup} describes the experiment design, and \cref{sec:results} presents empirical results along three dimensions: predictive performance in \cref{sec:results_prediction} (RQ2), elasticity and substitution patterns in \cref{sec:results_interpretation_substitution} (RQ3), and graph interpretation in \cref{sec:results_interpretation_graph} (RQ4). \cref{sec:conclusions} concludes and outlines future research directions. Following open science practices, our code and data are available at \url{https://github.com/yuqi-zhou537/Alt-GNN_for_Travel_Mode}.

\section{Literature Review}
\label{sec:literature_review}

\subsection{Discrete choice models for alternative dependence}
\label{sec:literature_review_nl}

It is both a theoretical challenge and an opportunity to enrich classical DCMs by leveraging alternative dependence. In the MNL model, the unobserved utility components are assumed to be i.i.d., so substitution patterns are proportional across all alternatives---a restriction known as the Independence of Irrelevant Alternatives (IIA) property \citep{luceIndividualChoiceBehavior1959, mcfaddenConditionalLogitAnalysis1974}. To address IIA, the GEV framework captures alternative dependence through structured generating functions that induce correlated errors among related alternatives, first derived by \cite{mcfaddenModelingChoiceResidential1978} and later generalized by \cite{benAkivaFrancois1983}. As the simplest GEV case, the NL model relaxes IIA by allowing alternatives within the same nest to share correlated error components \citep{williams1977, mcfaddenModelingChoiceResidential1978, dalyZachary1978}. The paired combinatorial logit model extends NL by forming a nest for each pair of alternatives, enabling richer substitution patterns through overlapping two-alternative nests \citep{mcfaddenModelingChoiceResidential1978, chu1989, koppelmanWen2000}. The spatially correlated logit model further restricts these paired nests to adjacent alternatives, so that only neighboring alternatives share unobserved components \citep{bhatMixedSpatiallyCorrelated2004}. The cross-nested logit model generalizes NL by allowing alternatives to belong to multiple nests through allocation parameters \citep{vovsha1997, ben-akivaDiscreteChoiceMethods1999}, and the generalized nested logit model extends cross-nested logit with nest-specific dissimilarity parameters and heterogeneous scale parameters \citep{wenGeneralizedNestedLogit2001}. \cite{bierlaireNetworkGEVModel2002} introduces the network GEV model, the most general member of the GEV family, which defines the generating function recursively over a directed acyclic graph and interprets choice as a hierarchical process of selecting nests before alternatives. However, GEV models capture alternative dependence subject to several limitations. The nest structures must be specified a priori, requiring strong assumptions about which alternatives are similar. Within any nest, a single scale parameter enforces symmetric and uniform dependence among all members, preventing the model from representing asymmetric cross-alternative effects. Finally, the dependence structure is homogeneous across decision makers: individual variation in substitution patterns cannot be represented without augmenting the model with additional components.

Regarding interpretation, the GEV family encodes alternative dependence through the inclusive value, in which the utility of one alternative is affected by the utilities of correlated alternatives in the same nest. While this structure is typically interpreted through a hierarchical decision-making mechanism, it is also analogous to the cross-utility effects in the mother logit model \citep{timmermansMotherLogitAnalysis1991}, which parameterizes cross-alternative influences directly in the systematic utility terms. Therefore, both can be understood as instances of a broader class of models in which an alternative's utility interacts with those of other alternatives, thus validating our behavioral interpretation of the Alt-GNNs. Regarding applications, GEV models are widely adopted for analyzing travel demand, including route choice \citep{bovyFactorRevisitedPath2008, liuUnderstandingRouteChoice2022, mepparambathNovelModellingApproach2023}, mode choice \citep{dingExploringInfluenceBuilt2017, batesPivotingKnownBase2024}, activity choice \citep{bowmanActivitybasedDisaggregateTravel2001}, and location choice \citep{perez-lopezSpatiallyCorrelatedNested2022}, as all such choice sets contain graph structures. 

Different from GEV models, simulation-based models capture alternative dependence through the joint distribution of random coefficients, with probit and mixed logit model as leading examples. The Probit model assumes that $\varepsilon_n$ follows a multivariate normal distribution, so that alternative dependence is directly encoded in the covariance matrix \citep{blissMethodProbits1934,thurstoneLawComparativeJudgment1974,hausmanConditionalProbitModel1978}. In Mixed logit model, the primary role of the mixing distribution $f(\beta)$ is to represent unobserved taste heterogeneity, and alternative dependence arises indirectly when random coefficients are correlated across alternatives \citep{mcfaddenMixedMNLModels2000, hessCorrelationScaleMixed2017, bansalBayesianEstimationMixed2020}. Both models are widely applied across travel behavior contexts, including mode choice \citep{volakakisWordsMatterAutonomous2025}, activity choice \citep{ghaderCopulabasedContinuousCrossnested2021, wangProbitBasedDiscreteContinuousChoice2023}, and location choice \citep{sahaModelingBicyclistsDestination2025}. However, simulation-based models are categorically different from the GEV family in model specification and training. Simulation-based models evaluate choice probabilities through numerical integration via stochastic simulation, whereas GEV models have closed-form choice probabilities and are estimated by deterministic optimization---leading to substantial differences in computational cost and convergence properties. Nonetheless, the two approaches are complementary as the coefficients in GEV models can be correlated through random covariance matrices \citep{bhatMixedSpatiallyCorrelated2004, ghaderCopulabasedContinuousCrossnested2021}. Similarly, while our Alt-GNNs are more pertinent to the GEV paradigm as they rely on deterministic training without simulation, they can be integrated with the simulation-based approaches.

\subsection{Synergizing deep learning and discrete choice models}
\label{sec:literature_review_dnn}
In light of the parametric restrictions in DCMs, researchers started to connect DCMs and DNNs, thus enhancing model prediction and interpretation. One approach is architectural integration, which incorporates neural components into the utility specification or parameterization. For example, \cite{wangDeepNeuralNetworks2020a} proposed ASU-DNN, which replaces the linear-in-parameters utility specification in the MNL model with a nonlinear neural network (NN), while retaining the MNL structure and its IIA property. ResLogit augments the utility function with a neural residual, which captures nonlinear cross-effects across alternatives and relaxes the restrictive substitution patterns implied by standard logit models \citep{wongResLogitResidualNeural2021}. TasteNet-MNL uses a NN to generate individual-specific taste parameters within the utility function, enabling flexible preference heterogeneity \citep{hanNeuralembeddedDiscreteChoice2022}. Embeddings Multinomial Logit (E-MNL) achieves interpretability by assigning each embedding dimension to a specific alternative, while EL-MNL extends this structure with an additional NN term that improves predictive performance but lacks direct behavioral interpretation \citep{arkoudiCombiningDiscreteChoice2023}. On the other hand, researchers have also attempted to improve the interpretability of DNN-based choice models by imposing behavioral constraints or adding regularization terms to the loss function to enforce economically meaningful properties. For example, \cite{fengDeepNeuralNetworks2024} introduces gradient regularization to enforce monotonic relationships between attributes and choice probabilities. \cite{haj-yahiaIncorporatingDomainKnowledge2025} incorporate behavioral domain knowledge through loss-based regularization by penalizing violations of economically meaningful properties, such as monotonicity and willingness-to-pay consistency, rather than directly regularizing gradients. Lattice network-based discrete choice model (DCM-LN) represents the utility function using a lattice network and guarantees partial monotonicity by design through constraints on the network parameters \citep{kimNewFlexiblePartially2024}. In addition, some studies connect DCMs and DNNs at the computational level. They represent choice models as computational graphs with differentiable components, which allows unified estimation of econometric models and neural networks \citep{kimComputationalGraphbasedFramework2022, kimComputationalGraphbasedMathematical2024}.



The synergistic DNN-DCM models achieve higher predictive performance across various travel behavior applications, such as mode choice \citep{hanNeuralembeddedDiscreteChoice2022,arkoudiCombiningDiscreteChoice2023,kimNewFlexiblePartially2024,haj-yahiaIncorporatingDomainKnowledge2025}, route choice \citep{yaoVariationalAutoencoderApproach2022}, traffic flow and demand–supply interaction \citep{kimComputationalGraphbasedMathematical2024,zhouFlowthroughTensorsUnified2025}, as well as sequences of joint activity–destination–mode choices over the day \citep{fredrikssonJointContextawareNeural2026}. While many studies demonstrated that DNNs can capture more flexible elasticity values and substitution patterns \citep{wangDeepNeuralNetworks2020a,wongResLogitResidualNeural2021}, \rev{the main challenge remains how to develop behaviorally informed model designs that introduce more structure into learned choice patterns, while still leveraging the expressive power of DNN architectures}. 

\rev{GNNs can potentially provide a bridge between DCMs and neural network models because graph representations and message passing provide an algorithmic perspective related to the structure of GEV choice probabilities. However, such connections between GEV models and GNN architectures have not been systematically investigated.} In GNNs, the message passing mechanism constitutes the algorithmic core \citep{corsoGraphNeuralNetworks2024}: each node aggregates information from its neighboring nodes and updates its representation accordingly. Through multi-layer message passing, the model effectively captures long-range dependencies, enabling a more comprehensive understanding of the underlying graph structure. Typical GNN examples include Graph Attention Networks and Graph Convolutional Networks: Graph Attention Networks learn the importance of neighboring alternatives through data-driven attention weights, while Graph Convolutional Networks model alternative dependence using adjacency-based convolution and Laplacian regularization to enforce smoothness in node representations \citep{velickovicGraphAttentionNetworks2018, kipfSemiSupervisedClassificationGraph2017a}. GNNs also have vast hyperparameter space, enabling enormous model design potential, as \cite{youDesignSpaceGraph2020} constructed a 12-dimensional design space to combine diverse architectural components (e.g., aggregation functions, activation functions, layer connection methods, etc.). Different from feedforward neural networks capturing input dependency only implicitly, GNNs use graph topology to explicitly represent the dependence among nodes. Such capabilities of graph representation have been used for travel demand prediction and choice modeling \citep{liGraphNeuralNetwork2022,fanGraphGuidedNeuralNetwork2023,tomlinsonGraphbasedMethodsDiscrete2024,villarragaDesigningGraphConvolutional2025a}. However, in previous choice modeling studies, graph structures mainly describe individuals’ social networks, while the alternative dependence through the graph representation has largely been overlooked. \cite{chengGraphNeuralNetworks2025} is, to our knowledge, the only prior work that applied GNNs to capture alternative dependence. Nonetheless, our study tackles a smaller alternative set in travel mode choice, as opposed to residential location \citep{chengGraphNeuralNetworks2025}, which enables us to conduct a more nuanced analysis of graph-based interpretation and substitution patterns in the Alt-GNN framework. Overall, the integration of GNNs and GEV models remains largely an open research question, warranting further exploration. 

In summary, classical DCMs rely on predefined nesting structures with parametric assumptions, symmetric utility specifications, and uniform alternative representation among individuals. The synergistic DNN-DCM approaches achieved higher predictive power, but used mainly feedforward architectures without leveraging graph representations to explicitly capture alternative dependence or impose behavioral constraints. Even the studies most pertinent to ours mainly focused on social network effects rather than alternative graphs \citep{villarragaDesigningGraphConvolutional2025a, tomlinsonGraphbasedMethodsDiscrete2024}. \rev{To address these limitations, we introduce the concept of the alternative graph and propose Alt-GNN models to explore how alternative-level graph representations can bridge GEV models and GNN architectures, as shown below.}

\section{Methodology}
\label{sec:methodology}

\subsection{Representation and definition of Alt-GNN}
\label{sec:methodology_alter_graph}
The alternative graph is denoted as $\mathcal{G} = (\mathcal{V}, \mathcal{E})$, where $\mathcal{V}$ represents the set of choice alternatives, and $\mathcal{E}$ is a set of edges representing the relationship among choice alternatives. Two related alternatives are connected by an edge. The edge set $\mathcal{E}$ can be represented by an adjacency matrix $\mathcal{A}_{\left|\mathcal{V}\right| \times \left|\mathcal{V}\right|}$, where $a_{ij} = 1$ if alternative $i$ is connected to alternative $j$, and $a_{ij} = 0$ otherwise. 
\begin{figure}[htbp]
\centering

\begin{subfigure}[b]{0.24\linewidth}
    \includegraphics[width=\linewidth]{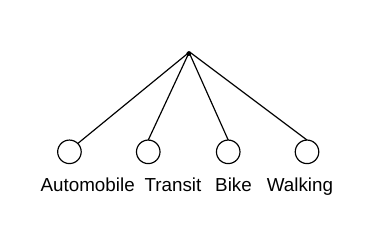}
    \caption{nesting structure 0}
    \label{sfig:nonenest}
\end{subfigure}
\begin{subfigure}[b]{0.24\linewidth}
    \includegraphics[width=\linewidth]{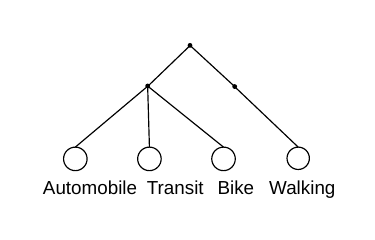}
    \caption{nesting structure 1}
    \label{sfig:0001nest}
\end{subfigure}
\begin{subfigure}[b]{0.24\linewidth}
    \includegraphics[width=\linewidth]{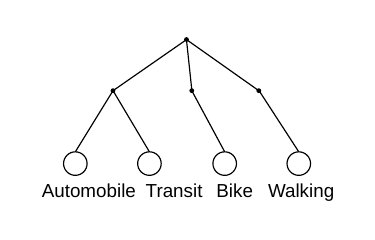}
    \caption{nesting structure 2}
    \label{sfig:0012nest}
\end{subfigure}
\begin{subfigure}[b]{0.24\linewidth}
    \includegraphics[width=\linewidth]{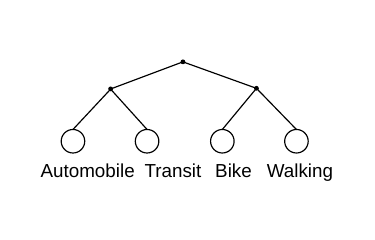}
    \caption{nesting structure 3}
    \label{sfig:0011nest}
\end{subfigure}

\begin{subfigure}[b]{0.24\linewidth}
    \includegraphics[width=\linewidth]{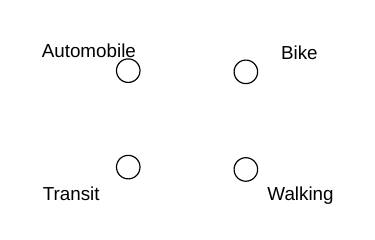}
    \caption{nested alternative graph 0 }
    \label{sfig:nonegraph}
\end{subfigure}
\begin{subfigure}[b]{0.24\linewidth}
    \includegraphics[width=\linewidth]{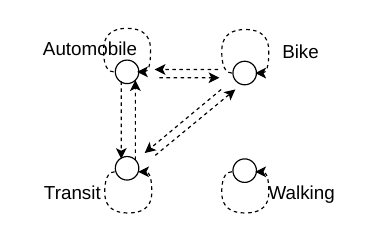}
    \caption{nested alternative graph 1}
    \label{sfig:0001graph}
\end{subfigure}
\begin{subfigure}[b]{0.24\linewidth}
    \includegraphics[width=\linewidth]{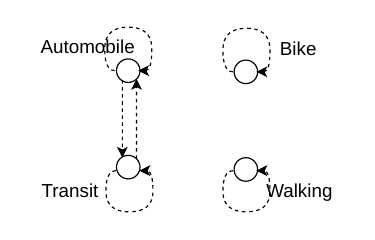}
    \caption{nested alternative graph 2}
    \label{sfig:0012graph}
\end{subfigure}
\begin{subfigure}[b]{0.24\linewidth}
    \includegraphics[width=\linewidth]{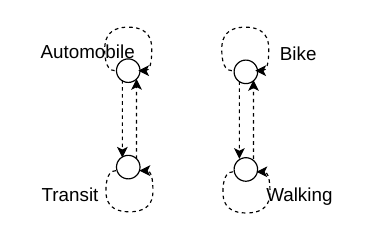}
    \caption{nested alternative graph 3}
    \label{sfig:0011graph}
\end{subfigure}

\caption{Two representations of alternative dependence: nesting structures and alternative graphs}
\label{fig:nesting_structures}

\end{figure}

To better understand the concept of alternative graphs, \cref{fig:nesting_structures} shows four nesting structures in the NL model and their corresponding alternative graphs, referred to as nested alternative graphs.
In nesting structures, alternatives are grouped into nests, while in alternative graphs, alternatives are represented as nodes and their relationships are captured by edges. The alternative graph is constructed by connecting alternatives within the same nest. For example, in nesting structure 1 (\cref{sfig:0001nest}), automobile, transit, and bike are grouped into the same nest, while walking forms another nest. In the corresponding nested alternative graph 1 (\cref{sfig:0001graph}), the three nodes of automobile, transit, and bike form a fully connected subgraph, while walking forms its own subgraph. Similarly, the nested alternative graphs 2-3 are constructed with the same approach. A unique case is nested alternative graph 0 without any edge (\cref{sfig:nonegraph}). This corresponds to nesting structure 0, which represents four independent alternatives as disconnected nodes. The nesting structures and nested alternative graphs have clear bi-directional mapping because they can be transformed to each other precisely. 

\begin{figure}[htbp]
\centering
\captionsetup[subfigure]{font=scriptsize,skip=2pt}

\begin{subfigure}[b]{0.45\linewidth}
    \centering
    \includegraphics[width=\linewidth]{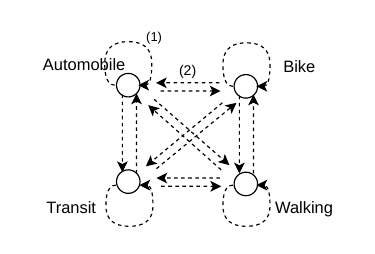}
    \caption{complete alternative graph}
    \label{sfig:complete_graph}
\end{subfigure}
\hfill
\begin{subfigure}[b]{0.45\linewidth}
    \centering
    \includegraphics[width=\linewidth]{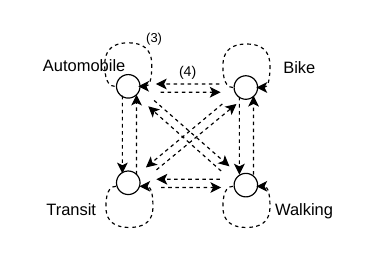}
    \caption{attention alternative graph}
    \label{sfig:attention_graph}
\end{subfigure}

\caption{Complete and attention-based alternative graphs. In the message aggregation stage, edge 1 denotes the message from alternative $i$ to itself, $\exp(\mathrm{MLP}(x_{ni}^{(0)}))$;
edge 2 denotes the message passing from alternative $j$ to $i$, 
$\exp(\mathrm{MLP}(x_{nj}^{(0)}))$; edge 3 denotes the attention-weighted message from alternative $i$ to itself, $\alpha_{nii}\exp(\mathrm{MLP}(x_{ni}^{(0)}))$; edge
4 denotes the attention-weighted message from alternative $j$ to $i$, $\alpha_{nij}\exp(\mathrm{MLP}(x_{nj}^{(0)}))$.}
\label{fig:complete_attention_graph_compare}
\end{figure}

\rev{The concept of alternative graph enables us to initialize structures beyond predefined nests and learn alternative interaction patterns from data. Here, graph edges specify the computational paths through which alternative-level representations are exchanged, while the edge weights control the strength of this information aggregation. These quantities characterize the computational structure of Alt-GNN rather than the random utility error correlation used in GEV models or the substitution patterns derived from choice probabilities.} For example, \cref{sfig:complete_graph} visualizes a graph with all the nodes connected. This complete alternative graph (\cref{sfig:complete_graph}) serves as the most general structure, where the graph is fully connected with uniform edge weights, i.e., $a_{ij} = 1$ for all $i, j \in \mathcal{V}$. In fact, the nested alternative graphs can be viewed as special cases with restricted edge patterns: the edge weights $a_{ij}$ are constrained to binary values determined by the nesting structure. If the nesting structure reflects the underlying data generating process, then such a structure should be learned from the data. On top of the complete graph, we further introduce an attention alternative graph (\cref{sfig:attention_graph}), which retains the full connectivity of the complete graph but replaces the uniform weights with individualized weights learned from data through an attention mechanism, where $a_{ij}$ takes continuous values in $[0,1]$, enabling more flexible and data-driven interaction patterns.


The representation of the alternative graph has advantages over classical concepts for at least three reasons. 
First, they provide a more general representation of dependence structures. Nesting structures impose a predefined hierarchical grouping with symmetric interactions, while alternative graphs allow non-nested and directional connections between alternatives. Second, alternative graphs support continuous edge weights and can vary across individuals. Automatic learning of alternative graphs allows interaction strengths to vary continuously and capture heterogeneous dependence patterns across individuals. Third, from a representation perspective, alternative graphs provide a more direct description of the message passing process than nesting structures. As shown in \cref{fig:complete_attention_graph_compare}, alternative graphs encode dependence through edges, where each edge represents a message in the aggregation step, allowing the contribution of each alternative to be interpreted through edge weights.

\begin{definition}
An alternative graph neural network (Alt-GNN) applies a message passing algorithm to an alternative graph, where each node represents an alternative and edges encode potential interactions among alternatives. 
The model is characterized by a set of hyperparameters
$\mathcal{H} = \{L, M, A, U, R\}$, 
corresponding to the number of layers, and message, aggregation, update, and readout functions. 
\end{definition}
Let $\mathcal{N}(i) \subseteq \mathcal{V}$ denote the neighborhood of node $i$,
defined as $\mathcal{N}(i) = \{\, j \in \mathcal{V} \mid \alpha_{ij} > 0 \,\},$
that is, node $j$ is included in $\mathcal{N}(i)$ if there exists a directed edge from $j$ to $i$.
Let $t = 0, 1, \dots, L$ index the representation at each message-passing stage, with $t = 0$ denoting the input representation. A general message passing algorithm for layer $(t+1)$ consists of three steps:
\begin{align*}
\mathbf{m}_{i,j}^{(t + 1)} &= M^{(t)}\bigl(x_{i}^{(t)}, x_j^{(t)}, a_{ij}\bigr), \\
\mathbf{a}_{i}^{(t + 1)} &= A \left( \left\{ \mathbf{m}_{i,j}^{(t+1)} \mid j \in \mathcal{N}(i) \right\} \right), \\
x_{i}^{(t+1)} &= U \left(x_{i}^{(t)}, \mathbf{a}_{i}^{(t+1)} \right).
\end{align*}

The three functions serve three goals. The first equation creates an edge message $\mathbf{m}_{i,j}^{(t + 1)}$ by combining two node features; the second equation aggregates the edge messages $\mathbf{m}_{i,j}^{(t + 1)}$ around the neighborhoods of node $i$; the third equation updates the node feature $x_{i}^{(t)}$ as $x_{i}^{(t+1)}$ by integrating the initial node feature $x_{i}^{(t)}$ and the aggregated message $\mathbf{a}_{i}^{(t + 1)}$. The algorithm is initialized with node features $x_i^{(0)}$ and ends with $x_i^{(L)}$ after $L$-layer updating. The final node feature $x_i^{(L)}$ could be used for predicting output $y_i$ with a readout function $R(x_i^{(L)})$, which  can take a simple linear form (e.g., $y_i = w^{\top} x_i^{(L)}$). This message passing algorithm can successfully aggregate the graph information into individual nodes, and it has been shown to generalize many existing GNN algorithms.

\subsection{Variants of Alt-GNNs}
\label{sec:methodology_alt_gnn_variants}
Although the Alt-GNN framework is built upon relatively standard message passing algorithms, it introduces a rich and flexible family of models for choice modeling. In particular, Alt-GNN provides a general representation that recovers existing models, including MNL, NL, and ASU-DNN, as special cases, thereby bridging established GEV models and GNN architectures. Building on these connections, we further introduce new Alt-GNN variants that use different alternative graph designs to represent dependence among alternatives. These variants include Nested, Complete, Top-K, and Attention Alt-GNNs, illustrating how alternative graph design expands the model design space of Alt-GNNs.

\begin{example}
MNL model is a unique case of Alt-GNNs when the hyperparameter space $\mathcal{H}$ takes the following form: $\{L = 0,\, M_{i}(\cdot) = \emptyset,\; A(\cdot) = \emptyset,\; U(\cdot) = \emptyset,\; R(x_i) = w_i^{\top}x_i^{(0)}  \}$.
\end{example}

\rev{As illustrated by Alternative Graph 0 in \cref{sfig:nonegraph}, the MNL model can be interpreted as a degenerate Alt-GNN defined on a disconnected alternative graph. Each alternative is represented by an isolated node, and the graph contains no edges between alternatives. Consequently, no information is exchanged across alternative nodes, and the utility of each alternative depends only on its own initial node features.} The utility is obtained directly from the initial node features through a linear readout function,
without any aggregation or update operations. In MNL, the utility for alternative $i$ and individual $n$, together with the corresponding choice probability, are given by:
\begin{equation}
V_{ni} = w_i^\top x_{ni},
\end{equation}
\begin{equation}
P_{ni} =
\frac{\exp(w_i^\top x_{ni})}
{\sum_{j \in \mathcal V} \exp(w_j^\top x_{nj})}.
\label{eq:gnn_probability_mnl}
\end{equation}

\begin{example}
ASU-DNN model is a unique case of Alt-GNNs when the hyperparameter space $\mathcal{H}$ takes the following form:
$\{ L = 0,\; M_{i,j}(\cdot) = \emptyset,\; A(\cdot) = \emptyset,\; U(\cdot) = \emptyset,\; R_i(x_i^{(0)}) = \mathrm{MLP}_i(x_i^{(0)}) \}$.
\end{example}

The ASU-DNN model also corresponds to a degenerate Alt-GNN using the alternative graph in \cref{sfig:nonegraph}.
Different from the linear utility specification in the MNL model, the readout function here is given by an alternative-specific multilayer perceptron, which is the same as \cite{wangDeepNeuralNetworks2020a}'s work. 
The utility and the corresponding choice probability functions are given by:
\begin{equation}
V_{ni} = \mathrm{MLP}_i(x_{ni}),
\end{equation}
\begin{equation}
P_{ni} =
\frac{\exp(\mathrm{MLP}_i(x_{ni}))}
{\sum_{j \in \mathcal V} \exp(\mathrm{MLP}_j(x_{nj}))}.
\end{equation}
In both variants, the MNL and ASU-DNN models follow the IIA constraint and thus exhibit the proportional substitution pattern among alternatives.

\begin{example}
NL model is a unique case of Alt-GNNs when the hyperparameter space $\mathcal{H}$ takes the following form: $L = 1$, $M_{i,j}(x_j^{(0)}) = w_j^\top x_j^{(0)} /\mu_k$,
$A(m_i) = (\mu_k - 1) LSE_{j \in \mathcal{N}(i)}(M_{i,j}(x_j^{(0)}))$,
$U(x_i, a_i) = M_{i,i}(x_i^{(0)}) + A(m_i)$, $R(x^{(1)}_i) = x_i^{(1)}$.
\label{example:gnn_nl}
\end{example}
The NL model can also be decomposed into standard Alt-GNN steps, with the alternative graph specified as Figures \ref{sfig:0001graph}, \ref{sfig:0011graph}, or \ref{sfig:0012graph}. 
An edge message is first constructed using a linear transformation,
$M_{i,j}(x_j^{(0)}) = w_j^\top x_j^{(0)}$.
Then the edge messages are aggregated using the log-sum-exponential (LSE) function over the neighborhood of the target node $i$, as $A(m_i) = (\mu_k - 1)\,\mathrm{LSE}_{j \in \mathcal{N}(i)}(M_{i,j}(x_j^{(0)}))$.
Third, the node feature $x_i$ is updated by combining the self-message $M_{i}(x_i^{(0)})$ and the aggregated neighborhood message $A(m_i)$, as $U(x_i, a_i) = M_{i}(x_i^{(0)}) + A(m_i)$, where $M_{i,i}(x_i^{(0)}) = w_i^\top x_i^{(0)} $ gives the self contribution of alternative $i$. Finally, the updated node representation $x_i^{(1)}$ is used directly as a utility value to calculate the choice probability. Combining these steps, the utility value and the choice probability are given by:

\begin{equation}
    V_{ni} = w_i^{\top}x_{ni}/\mu_k + ((\mu_k - 1) LSE_{j \in \mathcal{N}(i)}(w_j^{\top}x_{nj}/\mu_k)),
\end{equation}
\begin{equation}
 P_{ni} = \frac{\exp \left( w_i^{\top}x_{ni}/\mu_k + ((\mu_k - 1) LSE_{j \in \mathcal{N}(i)}(w_j^{\top}x_{nj}/\mu_k))\right)}{\sum_{m \in \mathcal{V}} \exp \left( w_m^{\top}x_{nm}/\mu_l + ((\mu_l - 1) LSE_{j \in \mathcal{N}(m)}(w_j^{\top}x_{nj}/\mu_l)) \right)}.
\label{eq:nsu_gnn_probability_nl}
\end{equation}

\begin{example}
Nested Alt-GNN can be specified as an example of Alt-GNN with high-dimensional message passing along edges when its hyperparameter space $\mathcal{H}$ takes the following form:
$L = 1,\;
M_{i,j}(x_j^{(0)}) = W_j^\top x_j^{(0)},\;
A(m_i) = LSE_{j \in \mathcal{N}(i)}(M_{i,j}(x_j^{(0)})),\;
U(x_i, a_i) = CONCAT(M_{i}(x_i^{(0)}), a_i),\;
R(x_i^{(1)}) = w_i^\top x_i^{(1)}.$
\label{example:gnn_high_dim_message}
\end{example}

Here we introduce new Alt-GNN variants that do not exist in past modeling practices. In contrast to the standard NL model, the Nested Alt-GNNs extend message passing to a high-dimensional feature space with flexible transformation and aggregation functions. The message transformation function $M_{i,j}(\cdot)$ may take either a linear form, such as $M_{i,j}(x_j^{(0)}) = W_j^\top x_j^{(0)}$, or a nonlinear form implemented by a multi-layer perceptron, like $M_{i,j}(x_j^{(0)}) = \mathrm{MLP}_j(x_j^{(0)})$. The aggregation is also performed over the alternative graph structures illustrated in \cref{sfig:0001graph}, \cref{sfig:0011graph}, or \cref{sfig:0012graph}. However, different from the NL model, here the LSE aggregation is applied in an element-wise manner, thus enabling the message passing in a high-dimensional space. The update function concatenates the transformed input feature of node $i$ and its edge message $a_i$, incorporating both individual and contextual information.

\revII{To preserve the within-neighborhood probability-ratio property, alternatives within the same nest share identical neighborhoods and aggregate the same sender-specific messages. Moreover, the contextual readout parameters are shared within each nest, such that $w_{2i}=w_{2j}$ whenever $\mathcal{N}(i)=\mathcal{N}(j)$. This ensures that the aggregated neighborhood information contributes an identical scalar component to the utilities of alternatives within the same nest.}

Only in the final readout function is the node feature $x_i$ reduced to one dimension for computing choice probabilities. The utility and the choice probability functions of an example Nested Alt-GNN are:

\begin{equation}
    V_{ni}
    =
    w_{1i}^{\top}(W_i^{\top}x_{ni})
    +
    w_{2i}^{\top}
    LSE_{j \in \mathcal{N}(i)}
    (W_j^{\top}x_{nj}),
    \label{equation:utility_nested_alt_gnn}
\end{equation}

\begin{equation}
P_{ni}
=
\frac{
\exp \left(
w_{1i}^{\top}(W_i^{\top}x_{ni})
+
w_{2i}^{\top}
LSE_{j \in \mathcal{N}(i)}
(W_j^{\top}x_{nj})
\right)
}{
\sum_{m \in \mathcal{V}}
\exp \left(
w_{1m}^{\top}(W_m^{\top}x_{nm})
+
w_{2m}^{\top}
LSE_{j \in \mathcal{N}(m)}
(W_j^{\top}x_{nj})
\right)
}.
\label{equation:choice_prob_nested_alt_gnn}
\end{equation}


\begin{example}
A Complete Alt-GNN with MLP readout is a unique case of Alt-GNNs when the hyperparameter space $\mathcal{H}$ takes the following form:
$
L = 1,\;
M_{i,j}(x_j^{(0)}) = \mathrm{MLP}^{(M)}(x_j^{(0)}),\;
A(m_i) = \mathrm{LSE}_{j \in \mathcal{V}}\!\left(M_{i,j}(x_j^{(0)})\right),
U(x_i, a_i) = \mathrm{CONCAT}\!\left(M_{i}(x_i^{(0)}),\, a_i\right),\;
R(x_i^{(1)}) = \mathrm{MLP}^{(R)}(x_i^{(1)}).$
\label{example:completegnn}
\end{example}
In the complete-graph GNN, message aggregation is performed over all alternatives, meaning that each alternative aggregates messages from all other alternatives in the choice set. 
This corresponds to a fully connected graph without predefined nesting structures. 
As illustrated in \cref{sfig:complete_graph}, message passing occurs along all pairwise paths between alternatives. In addition, the readout function is no longer restricted to be linear or shared within nests, and is instead implemented as a node-specific multi-layer perceptron.
This specification allows unrestricted interactions among alternatives and serves as a general benchmark without substitution constraints. The utility and the choice probability functions of the Complete Alt-GNN are:

\begin{equation}
V_{ni}=\mathrm{MLP}^{(R)}_{i}\!\Big(\mathrm{CONCAT}\big(\mathrm{MLP}^{(M)}(x_{ni}),\;\mathrm{LSE}_{j \in \mathcal{V}}\!\big(\mathrm{MLP}^{(M)}(x_{nj})\big)\big)\Big),
\end{equation}

\begin{equation}
P_{ni}
=
\frac{\exp\!\Big(
\mathrm{MLP}^{(R)}_{i}\!\Big(
\mathrm{CONCAT}\big(
\mathrm{MLP}^{(M)}(x_{ni}),
\;
\mathrm{LSE}_{j \in \mathcal{V}}\!\big(\mathrm{MLP}^{(M)}(x_{nj})\big)
\big)
\Big)
\Big)}
{\sum_{m \in \mathcal{V}} \exp\!\Big(
\mathrm{MLP}^{(R)}_{m}\!\Big(
\mathrm{CONCAT}\big(
\mathrm{MLP}^{(M)}(x_{nm}),
\;
\mathrm{LSE}_{j \in \mathcal{V}}\!\big(\mathrm{MLP}^{(M)}(x_{nj})\big)
\big)
\Big)
\Big)
}.
\label{eq:completegnn_probability}
\end{equation}
\begin{example}
\rev{Top-K Alt-GNN is an Alt-GNN variant that learns a sparse alternative graph by retaining the top-$k$ edges for each alternative, specified by the following hyperparameter space $\mathcal{H}$:
$L=1,\;
M_{i,j}(x_j^{(0)})=W_j^\top x_j^{(0)},\;
\mathcal{N}_{\Theta}^{(k)}(i)
=
\operatorname{TopK}_{j\in\mathcal{V}}
\left(\theta_{ij},k\right),\;
A(m_i)
=
LSE_{j\in\mathcal{N}_{\Theta}^{(k)}(i)}
(M_{i,j}(x_j^{(0)})),\;
U(x_i,a_i)
=
CONCAT(M_i(x_i^{(0)}),a_i),\;
\\ R(x_i^{(1)})=w_i^\top x_i^{(1)} $.}
\label{example:discrete_alt_gnn}
\end{example}

Unlike Nested Alt-GNNs, the Top-K Alt-GNN jointly learns its graph structure and utility function from data. The edge scores form a population-shared parameter matrix
$\boldsymbol{\Theta}=[\theta_{ij}]
\in\mathbb{R}^{|\mathcal{V}|\times|\mathcal{V}|}$,
where rows index target alternatives and columns index source alternatives. For each target alternative $i$, the model retains the $k$ source alternatives with the largest edge scores:
\begin{equation}
\mathcal{N}_{\Theta}^{(k)}(i)
=
\operatorname{TopK}_{j\in\mathcal{V}}
\left(\theta_{ij},k\right).
\label{discrete_edge_selection}
\end{equation}

\revII{Because the hard Top-$k$ selection is not differentiable with
respect to the edge scores, we use a straight-through softmax estimator
to jointly learn the graph structure and utility-function parameters.
Let
$h_{ij}=\mathbb{I}[j\in\mathcal{N}_{\Theta}^{(k)}(i)]$
denote the hard adjacency indicator, and let
$\widetilde{h}_{ij}
=
k\,\operatorname{softmax}(\theta_{i:}/\tau)_j$
denote its continuous relaxation. The edge gate used during training is
\begin{equation}
g_{ij}^{\mathrm{ST}}
=
h_{ij}
+
\widetilde{h}_{ij}
-
\operatorname{sg}
\left(
\widetilde{h}_{ij}
\right),
\label{equation:topk_straight_through}
\end{equation}
where $\operatorname{sg}(\cdot)$ denotes the stop-gradient operator.
Consequently, the forward pass uses exactly the hard Top-$k$ graph,
whereas the backward pass differentiates through the softmax
relaxation, allowing the edge scores to be updated jointly with the
other model parameters. The hard neighborhood is recomputed from the
current edge scores in every forward pass, and we set $\tau=1$ in all
experiments.}

The message transformation function $M_{i,j}(\cdot)$ may take either a
linear form, such as
$M_{i,j}(x_j^{(0)})=W_j^\top x_j^{(0)}$, or a nonlinear form implemented
by a multi-layer perceptron, like
$M_{i,j}(x_j^{(0)})=\mathrm{MLP}_j(x_j^{(0)})$. For each alternative
$i$, the model retains the $k$ incoming edges with the highest learned
scores $\theta_{ij}$, forming the learned neighborhood
$\mathcal{N}_{\Theta}^{(k)}(i)$. The LSE aggregation is applied in an
element-wise manner over this learned neighborhood, thus enabling
message passing in a high-dimensional space. The update function
concatenates the transformed input feature of node $i$ and its edge
message $a_i$, incorporating both individual and contextual
information. Only in the final readout function is the node feature
$x_i$ reduced to one dimension for computing choice probabilities.
The utility and choice probability functions of an example Top-K
Alt-GNN are:

\begin{equation}
\rev{
V_{ni}
=
w_{1i}^{\top}(W_i^{\top}x_{ni})
+
w_{2i}^{\top}
LSE_{j\in\mathcal{N}_{\Theta}^{(k)}(i)}
(W_j^{\top}x_{nj}).
}
\label{equation:utility_discrete_alt_gnn}
\end{equation}

\begin{equation}
\rev{
P_{ni}
=
\frac{
\exp\left(
w_{1i}^{\top}(W_i^{\top}x_{ni})
+
w_{2i}^{\top}
LSE_{j\in\mathcal{N}_{\Theta}^{(k)}(i)}
(W_j^{\top}x_{nj})
\right)
}{
\sum_{m\in\mathcal{V}}
\exp\left(
w_{1m}^{\top}(W_m^{\top}x_{nm})
+
w_{2m}^{\top}
LSE_{j\in\mathcal{N}_{\Theta}^{(k)}(m)}
(W_j^{\top}x_{nj})
\right)
}.
}
\label{equation:choice_prob_discrete_alt_gnn}
\end{equation}
\begin{example}
Attention Alt-GNN with MLP readout is a unique case of Alt-GNNs when the  hyperparameter space $\mathcal{H}$ takes the following form:
$L=1,\;
M_{i,j}(x_j^{(0)})=\mathrm{MLP}(x_j^{(0)}),\;
A(m_i)=\log\sum_{j\in\mathcal{V}}\alpha_{nij}\exp\!\big(M_{i,j}(x_j^{(0)})\big),\;
U(x_i,a_i)=\mathrm{CONCAT}\!\big(M_{i,i}(x_i^{(0)}),a_i\big),\;
R(x_i^{(1)})=\mathrm{MLP}^{(R)}(x_i^{(1)})$,
where $\alpha_{nij}$ is computed by an attention mechanism on projected node features.
\label{example:attentiongnn}
\end{example}

Beyond the complete-graph specification (\cref{sfig:complete_graph}), we further consider an attention-based GNN (\cref{sfig:attention_graph}) that learns importance weights over alternative edges from data.
Specifically, attention coefficients $\alpha_{nij} \in [0,1]$ are computed for each individual $n$ and each ordered pair of alternatives $(i,j)$ to modulate the contribution of alternative $j$ to the aggregated message of alternative $i$, where
\begin{equation}
\alpha_{nij}
=
\frac{\exp\!\left(\mathrm{LeakyReLU}\!\left(a^\top [\,h_i \,\|\, h_j\,]\right)\right)}
{\sum_{k \in \mathcal{V}} \exp\!\left(\mathrm{LeakyReLU}\!\left(a^\top [\,h_i \,\|\, h_k\,]\right)\right)}.
\label{attention_weight}
\end{equation}
Given a fixed message function $M^{\text{base}}_{i,j}(x_j^{(0)})$, attention weights are applied at the aggregation stage, such that the aggregated message is formed as an attention-weighted aggregation over all alternatives. \rev{The attention coefficients represent observation-specific information-aggregation weights that determine the relative contribution of neighboring alternative representations. Their role is therefore computational rather than a direct representation of random utility error correlation or behavioral substitution.}

This formulation allows the model to learn which alternative interactions receive greater weight during information aggregation while down-weighting less relevant ones in a data-driven manner.
The utility and the choice probability functions of Attention Alt-GNNs are:

\begin{equation}
V_{ni}
=
\mathrm{MLP}^{(R)}_{i}\!\Big(
\mathrm{CONCAT}\big(
\mathrm{MLP}(x_{ni}^{(0)}),
\;
\log \sum_{j \in \mathcal{V}}
\alpha_{nij}\,
\exp\!\big(\mathrm{MLP}(x_{nj}^{(0)})\big)
\big)
\Big),
\end{equation}
\begin{equation}
P_{ni}
=
\frac{\exp\!\Big(
\mathrm{MLP}^{(R)}_{i}\!\Big(
\mathrm{CONCAT}\big(
\mathrm{MLP}(x_{ni}^{(0)}),
\;
\log \sum_{j \in \mathcal{V}}
\alpha_{nij}\,
\exp\!\big(\mathrm{MLP}(x_{nj}^{(0)})\big)
\big)
\Big)
\Big)}
{\sum_{m \in \mathcal{V}}
\exp\!\Big(
\mathrm{MLP}^{(R)}_{m}\!\Big(
\mathrm{CONCAT}\big(
\mathrm{MLP}(x_{nm}^{(0)}),
\;
\log \sum_{j \in \mathcal{V}}
\alpha_{nmj}\,
\exp\!\big(\mathrm{MLP}(x_{nj}^{(0)})\big)
\big)
\Big)
\Big)}.
\label{eq:attentiongnn_probability}
\end{equation}

Overall, the seven Alt-GNN variants demonstrate the flexibility of the design space in the Alt-GNN model family.
The first three variants demonstrate how the Alt-GNN models reduce to the existing MNL, NL, and ASU-DNN models.  The Nested Alt-GNN extends the NL model by allowing more flexible aggregation functions and readout functions in higher-dimensional representations while preserving the nested structure and substitution patterns (see details in Section \ref{sec:methodology_beh_constraints}).
In contrast, the Complete Alt-GNN specifications relax the structural and readout constraints imposed by Nested Alt-GNN, allowing unrestricted interactions among alternatives and serving as general-purpose baselines without substitution restrictions. Attention Alt-GNN further learns individual-specific edge weights, allowing the strength of interactions to vary across observations.

\subsection{Interpretation of Alt-GNNs}
\label{sec:methodology_alt_gnn_interpretability}

\subsubsection{Validity of behavioral interpretation in Alt-GNNs}

\rev{Alt-GNNs are connected to logit-based discrete choice models through
the shared softmax mapping from deterministic utilities to choice
probabilities.} Let $V(\mathbf{x}_{ni},\mathbf{z}_n)$ denote the deterministic utility of alternative $i$ for individual $n$, where $\mathbf{x}_{ni}$ represents alternative-specific attributes and $\mathbf{z}_n$ individual characteristics. The total utility is defined as
$U_{ni}=V(\mathbf{x}_{ni},\mathbf{z}_n)+\epsilon_{ni},$
where $\epsilon_{ni}$ captures unobserved factors. Under the RUM framework \citep{manskiStructureRandomUtility1977,mcfaddenConditionalLogitAnalysis1974}, the probability that individual $n$ chooses alternative $i$ is
$P_{ni}=\text{Prob}(V_{ni}+\epsilon_{ni}>V_{nj}+\epsilon_{nj}, \forall j\in\mathcal{V}, j\neq i).$
When $\epsilon_{ni}$ follows an i.i.d.\ Gumbel distribution, the resulting choice probability takes the softmax form
$P_{ni}=\frac{e^{V_{ni}}}{\sum_{j\in\mathcal{V}} e^{V_{nj}}}.$ \rev{Alt-GNNs adopt the same mapping, allowing their pre-softmax outputs
to be interpreted as deterministic utilities. This connection
provides a deterministic utility interpretation of Alt-GNNs,
without identifying their learned graph structure with dependence
among unobserved utility components.}

\revII{
Unlike the conventional MNL specification,
Alt-GNNs can incorporate information from
other alternatives into the deterministic utility
of alternative $i$ through message passing.
} The closest framework in the DCM literature is the Mother Logit model \citep{timmermansMotherLogitAnalysis1991}, which allows the utility of an alternative to depend on attributes of other alternatives. In Mother Logit, the utility of alternative $i$ is defined as
$U_i = V_i + \sum_{j \neq i} \gamma_{ij} g(x_{nj}, z_n) + \varepsilon_i,$
where $V_i$ is the baseline utility determined by the attributes of alternative $i$, $g(x_{nj}, z_n)$ represents the utility contribution from alternative $j$ to alternative $i$, and $\gamma_{ij}$ captures cross-alternative effects. Assuming $\{\varepsilon_i\}$ follow an i.i.d.\ Type-I Extreme Value distribution, the resulting choice probability again takes the softmax form
\begin{equation}
P_{ni} =
\frac{\exp\!\left( V_{ni} + \sum_{j \neq i} \gamma_{ij} g(x_{nj}, z_n) \right)}
{\sum_{m \in \mathcal{V}} \exp\!\left( V_{nm} + \sum_{j \neq m} \gamma_{mj} g(x_{nj}, z_n) \right)}.
\label{equation:mother_logit}
\end{equation}

\rev{We differentiate between the random utility foundation and the deterministic utility interpretation of a choice model, as the former is an assumption about the joint distribution of the unobserved utility components, while the latter is an interpretation of the closed-form choice probability function. GEV models are defined at the level of the random utility foundation: dependence among alternatives is induced by the joint distribution of the error terms, and that assumption does not necessarily imply that the systematic utility of one alternative interacts with that of another. The resulting probability function, however, admits a deterministic utility interpretation. Under this reading, the Mother Logit model operates on a fully connected alternative graph in which each alternative receives utility contributions from all others. As shown by \cref{eq:nsu_gnn_probability_nl}, the NL utility function can likewise be interpreted as receiving contributions from the alternative itself and from all others within its nest. \revII{
Therefore, the theoretical contribution of this study
focuses on establishing a connection between classical
discrete choice models and Alt-GNNs through
deterministic utility representation, rather than
establishing a general equivalence in their random
utility foundations. In particular, we do not seek
to characterize the GEV generating functions or
random error dependence structures associated with
general Alt-GNN specifications.
}} 


\subsubsection{Behavior-informed graph design over Alt-GNNs}
\label{sec:methodology_beh_constraints}
While many studies found that DNNs present more flexible economic information than classical DCMs, a major challenge resides in how to exert prior knowledge over such economic information to avoid overly irregular economic patterns. The ASU-DNN model presents one example but by adopting an overly restrictive approach, as the researchers imposed the full IIA constraint to DNNs. Here the Alt-GNNs leverage the alternative graph and message passing algorithms to enable model design. To demonstrate the process, here we use the Nested Alt-GNN as a leading example with its general utility and choice probability functions as:
\begin{equation}
    V_{ni} = \phi(x_{ni}) + A( \left\{ \mathbf{m}_{i,j}^{(1)}, \forall j \in \mathcal{N}(i) \right\}),
\end{equation}
\begin{equation}
P_{ni} =
\frac{
\exp \left(
\phi(x_{ni}) +
A\left(
\left\{
\mathbf{m}_{i,j}^{(1)}
:
j \in \mathcal{N}(i)
\right\}
\right)
\right)
}{
\sum_{m \in \mathcal{V}}
\exp \left(
\phi(x_{nm}) +
A\left(
\left\{
\mathbf{m}_{m,j}^{(1)}
:
j \in \mathcal{N}(m)
\right\}
\right)
\right)
}.
\label{eq:nsu_gnn_probability}
\end{equation}

This Nested Alt-GNN model is characterized by its additive nest utility with a general aggregation function $A$. The first term represents the alternative $i$'s self-utility and the second term the nest utility from $i$'s neighbors. Since each subnetwork in \cref{fig:nesting_structures} is fully connected within each nest, even one-step neighborhood aggregation here can summarize all the information within a nest. Different from NL, the Nested Alt-GNN has a more general aggregation function $A$, rather than being limited to the LSE function only. Nonetheless, the rather general Nested Alt-GNNs can impose behavioral constraints by \revII{preserving selected NL-related probability-ratio properties}, as shown by the property below.  

\begin{property_Nested_Alt-GNN}
The Nested Alt-GNN model is characterized by its
\revII{two-level probability-ratio structure}:
\revII{a within-neighborhood probability ratio
in which the shared neighborhood component cancels
for alternatives $i$ and $j$ when
$j\in\mathcal{N}(i)$, and a between-neighborhood
probability ratio that retains neighborhood-specific
components when $j\notin\mathcal{N}(i)$.}
\end{property_Nested_Alt-GNN}

The mathematical proof is straightforward. Let $\mathcal{N}(i)$
denote the set of alternatives in the same nest as alternative $i$,
including $i$ itself. When $j \in \mathcal{N}(i)$, the ratio of
choice probabilities is
\begin{equation}
\frac{P_{ni}}{P_{nj}} = \frac{e^{\phi(x_{ni})}}{e^{\phi(x_{nj})}},
\label{eq:sub_gnn_1}
\end{equation}
which does not depend on any alternatives other than $i$ and $j$.
When $j \notin \mathcal{N}(i)$, the ratio of choice probabilities
equals:
\begin{equation}
\frac{P_{ni}}{P_{nj}} = \frac
{\exp \left( \phi(x_{ni}) + A\left( \left\{ \mathbf{m}_{i,m}^{(1)} : m \in \mathcal{N}(i) \right\} \right) \right)}
{\exp \left( \phi(x_{nj}) + A\left( \left\{ \mathbf{m}_{j,m}^{(1)} : m \in \mathcal{N}(j) \right\} \right) \right)}.
\label{eq:sub_gnn_2}
\end{equation}

\rev{This property also applies to Nested Alt-GNN specifications with high-dimensional linear or MLP message functions, multiple message-passing layers, and MLP readout functions. The key restriction is that the final readout must be additively separable into an alternative-specific scalar component and a neighborhood-specific scalar component shared by all alternatives in the same neighborhood. Under this restriction, the shared contextual component cancels from the probability ratio of two alternatives in the same neighborhood, regardless of the dimensionality of the intermediate representations. The complete $M$--$A$--$U$--$R$ specification and proof for the high-dimensional, multi-layer case are provided in Appendix \ref{sec:appendix_nested_alt_gnn_high_dimension}.}

This property highlights how Nested Alt-GNNs enable researchers to design models through the structure of the alternative graph. \revII{
In particular, this probability-ratio structure
resembles a selected property of NL
(see Appendix \ref{sec:appendix_substitution_nl}).
However, preserving this property does not
necessarily imply preserving the full substitution
structure of NL. This structural restriction
can make Nested Alt-GNN more interpretable
than an unconstrained DNN.
}

\revII{
Despite the similarity in selected probability-ratio
properties, NL and Nested Alt-GNN still differ
in several important respects.
} \rev{First, the LSE aggregation of the scalar messages corresponds to the inclusive value only when the messages, nest-specific scale parameter, and utility update satisfy the NL specification. Under these conditions, the resulting probabilities reproduce the conditional and marginal probability decomposition of NL.} Second, in the NL model, the message passed from alternative $j$ to alternative $i$ is restricted to a predefined scalar form. In contrast, Nested Alt-GNN allows the message function to be parameterized by a learnable nonlinear mapping, such as a multilayer perceptron, enabling more flexible transformations of alternative features during message passing. Third, the NL model uses a fixed LSE form for neighborhood aggregation, different from the mean, max, sum, and element-wise LSE aggregation functions that can be used in a general GNN model. Lastly, the NL model is limited to a one-layer GNN aggregation, rather than repeatedly using GNN aggregations for multiple steps, as commonly done in GNN models. 

\section{Experiment Design}
\label{sec:experiment_setup} 
\subsection{Datasets from Chicago and London} 
We use two travel survey datasets from Chicago and London to empirically test Alt-GNNs. The CMAP dataset is derived from the My Daily Travel Survey \citep{HouseholdTravelSurvey}. \rev{Trips are reconstructed from consecutive GPS-derived place records for the same individual on the same survey day.} \rev{Travel times are obtained from archived Google Directions API records originally assembled by \citet{fengDeepNeuralNetworks2024} and supplemented through subsequent data-collection rounds. Only trips with complete and
plausible route attributes for all four travel modes are retained.}
\rev{Because observed CMAP travel costs are extremely sparse, travel costs are
excluded from the primary CMAP specification.} The LPMC dataset provides a fully constructed mode choice dataset derived from the London Travel Demand Survey \citep{hillelRecreatingPassengerMode2018}. In contrast to the CMAP dataset, the LPMC dataset directly includes pre-constructed alternative-specific level-of-service attributes obtained via an online journey planner, as well as cost variables for public transit and driving. 

These two datasets represent complementary travel behavior contexts, differing in geographic setting (Europe vs. U.S.), traveler characteristics, travel features, and market shares. To demonstrate the details, we select age, gender, number of vehicles in the household, and household size as individual-level variables, and travel time and travel cost as alternative-specific attributes. The observed choice is one of four alternatives: automobile, transit, walking, and biking. The two datasets exhibit similar mean age, household size, and gender distributions, but differ in vehicle ownership, with CMAP showing higher car ownership. As shown in Table \ref{tab:summary_stats_twopanel}, the age distributions are similar across the two cities, with mean ages around 39 years. The gender composition is also comparable, with male travelers accounting for 47\% of the LPMC sample and 45\% of the CMAP sample. However, CMAP households exhibit higher vehicle ownership (mean = 1.39) compared to LPMC (mean = 0.98), reflecting stronger automobile dependence. 

\begin{table}[htbp]
\centering
\small
\caption{Summary Statistics of the LPMC and CMAP Samples
($N_{\mathrm{LPMC}}=81{,}086$; $N_{\mathrm{CMAP}}=20{,}796$)}
\label{tab:summary_stats_twopanel}
\resizebox{\linewidth}{!}{%
\begin{tabular}{l|rrrrrrr|rrrrrrr}
\toprule
& \multicolumn{7}{c|}{LPMC}
& \multicolumn{7}{c}{CMAP} \\
Variable
& Mean & Std & Min & 25\% & 50\% & 75\% & Max
& Mean & Std & Min & 25\% & 50\% & 75\% & Max \\
\midrule

Age (years)
& 39.46 & 19.23 & 5.00 & 25.00 & 38.00 & 52.00 & 99.00
& 39.64 & 13.11 & 13.00 & 30.00 & 37.00 & 48.00 & 84.00 \\

Male (0/1)
& 0.47 & 0.50 & 0.00 & 0.00 & 0.00 & 1.00 & 1.00
& 0.45 & 0.50 & 0.00 & 0.00 & 0.00 & 1.00 & 1.00 \\

Number of vehicles
& 0.98 & 0.75 & 0.00 & 0.00 & 1.00 & 2.00 & 2.00
& 1.39 & 0.98 & 0.00 & 1.00 & 1.00 & 2.00 & 8.00 \\

Household size
& 2.42 & 1.24 & 1.00 & 1.00 & 2.00 & 3.00 & 9.00
& 2.47 & 1.37 & 1.00 & 1.00 & 2.00 & 3.00 & 12.00 \\

\midrule

Transit time (min)
& 27.99 & 18.61 & 0.43 & 13.52 & 23.30 & 38.52 & 163.80
& 59.44 & 72.01 & 1.48 & 20.85 & 37.83 & 67.22 & 599.67 \\

Walking time (min)
& 67.74 & 67.09 & 1.52 & 21.03 & 43.38 & 90.83 & 556.70
& 125.83 & 151.98 & 5.38 & 33.02 & 72.00 & 146.50 & 1,484.55 \\

Driving time (min)
& 16.93 & 15.14 & 0.02 & 6.45 & 11.50 & 22.13 & 123.65
& 13.80 & 9.54 & 0.85 & 7.22 & 10.67 & 17.52 & 78.57 \\

Biking time (min)
& 21.73 & 21.12 & 0.37 & 7.02 & 13.93 & 29.08 & 183.12
& 35.50 & 41.79 & 1.13 & 10.87 & 20.08 & 41.10 & 424.92 \\

\midrule

Transit cost (GBP)
& 1.56 & 1.54 & 0.00 & 0.00 & 1.50 & 2.40 & 13.49
& -- & -- & -- & -- & -- & -- & -- \\

Driving cost (GBP)
& 1.90 & 3.49 & 0.00 & 0.29 & 0.57 & 1.29 & 17.16
& -- & -- & -- & -- & -- & -- & -- \\

\bottomrule
\end{tabular}%
}
\end{table}

The two datasets also differ in travel characteristics and market shares.
The CMAP sample has a longer average transit travel time (59.44 minutes) than
the LPMC sample (27.99 minutes), while their average driving times are of a
similar magnitude (13.80 and 16.93 minutes, respectively). Average walking and
biking times are both higher in CMAP than in LPMC. 
As shown in Table \ref{tab:market_share_transposed}, automobile trips account
for 65.26\% of the CMAP sample. The LPMC sample has a more balanced mode
distribution, with automobile and transit shares of 44.16\% and 35.28\%,
respectively. These differences provide distinct empirical settings for
assessing the robustness of Alt-GNNs.

\begin{table}[htbp]
\centering
\small
\caption{Market Shares by Travel Mode in the LPMC and CMAP Samples}
\label{tab:market_share_transposed}
\begin{tabular}{l|cccc|c}
\toprule
Dataset & Automobile & Transit & Walking & Biking & Total trips ($N$) \\
\midrule

LPMC Data
& 44.16\% & 35.28\% & 17.60\% & 2.97\% & 81{,}086 \\

CMAP Data
& 65.26\% & 13.38\% & 19.45\% & 1.91\% & 20{,}796 \\

\bottomrule
\end{tabular}
\end{table}

\subsection{Specifying hyperparameter space of Alt-GNNs}
We hypothesize that the alternative graph will serve as the most important hyperparameter influencing the predictive performance and behavioral interpretation of Alt-GNNs. Therefore, we examine three nested alternative graphs for Nested Alt-GNNs by assigning the four travel modes: automobile, transit, bike, and walking to different subgraphs (\cref{fig:nesting_structures}). For reference, alternative graph 0 represents the no-nesting baseline, in which alternatives are modeled independently without cross-alternative
edges, as shown in \cref{sfig:nonenest,sfig:nonegraph}. The three graph
structures included in the Nested Alt-GNN hyperparameter space are shown in
\cref{sfig:0001graph,sfig:0012graph,sfig:0011graph}. \rev{We also test the Complete Alt-GNN, which relaxes the nesting constraints by fully connecting all alternatives (see \cref{sfig:complete_graph}); the Top-K Alt-GNN, which learns a population-level sparse graph shared by all individuals; and the Attention Alt-GNN, which uses the complete graph as the interaction structure and learns individual-specific attention weights over its edges, allowing the interaction among alternatives to vary across individuals.}

Besides the alternative-graph specification, we compare mean, max, and logsum aggregation for Nested, Complete, Top-K, and Attention Alt-GNNs. Attention Alt-GNN uses the corresponding attention-weighted aggregation functions. We consider linear and MLP message functions, linear and MLP readouts, one or two message-passing layers, and hidden dimensions of 32 or 64. Configurations exceeding the common 40,000-parameter limit are excluded. All models use a fixed learning rate of 0.001. \Cref{tab:model_config_summary} summarizes the graph, message, aggregation, update, readout, and architecture settings for each model family.

\rev{The benchmark models include MNL, NL, ASU-DNN, ResLogit, and TasteNet-MNL. MNL, NL, and ASU-DNN represent restricted specifications of the Alt-GNN framework, while ResLogit introduces residual utility learning and TasteNet-MNL captures systematic taste heterogeneity through individual-specific utility parameters. To improve comparability, all neural model configurations are restricted to at most 40,000 trainable parameters. This constraint is implemented by limiting the number of hidden layers and selecting hidden-unit sizes so that each candidate configuration remains below the common parameter cap. The actual parameter count of each selected model is reported separately.}

As summarized in \cref{tab:model_config_summary}, the hyperparameter space contains 370 model specifications: 126 Nested Alt-GNN configurations across three nested graph structures, 42 Complete Alt-GNN configurations, 126 Top-K Alt-GNN configurations, 42 Attention Alt-GNN configurations, 14 ASU-DNN configurations, 12 TasteNet-MNL configurations, one MNL specification, three NL specifications, and four ResLogit specifications with 2, 4, 8, and 16 residual layers. 

\begin{table}[htbp]
\centering
\small
\setlength{\tabcolsep}{3pt}
\renewcommand{\arraystretch}{1.1}
\caption{Summary of model configurations in the experiments.}
\label{tab:model_config_summary}

\begin{tabular}{
>{\centering\arraybackslash}m{2.5cm}|
>{\raggedright\arraybackslash}m{2.3cm}|
>{\raggedright\arraybackslash}m{6.4cm}|
>{\centering\arraybackslash}m{2cm}|
>{\centering\arraybackslash}m{1.0cm}}
\hline

\textbf{Model Type} &
\multicolumn{2}{c|}{\textbf{Hyperparameter Configuration}} &
\textbf{Parameter Range} &
\textbf{Total}\\
\hline

MNL &
\multicolumn{2}{c|}{--} &
27 &
1\\
\hline

NL &
Nesting structure &
0011/0001/0012 &
29--30 &
3\\
\hline

\multirow{2}{2.7cm}{\centering ASU-DNN} &
DNN depth &
L = 1/2/3 &
\rev{\multirow{2}{2.2cm}{\centering 1,988--39,684}} &
\multirow{2}{1.0cm}{\centering 14}\\
\cline{2-3}
&
Hidden units &
\makecell[l]{
L = 1: 64/128/256/512/768/1024/1152/1280\\
L = 2: 32/64/80\\
L = 3: 32/48/64
} &
&\\
\hline

\multirow{1}{2.7cm}{\centering\rev{ResLogit}} &
\rev{Residual depth} &
\rev{2/4/8/16 layers} &
\rev{84--308} &
\rev{4}\\
\hline

\multirow{2}{2.7cm}{\centering\rev{TasteNet-MNL}} &
\rev{TasteNet depth} &
\rev{1--2 hidden layers} &
\multirow{2}{2.2cm}{\centering\rev{373--39,189}} &
\multirow{2}{1.0cm}{\centering\rev{12}}\\
&
\rev{Hidden width} &
\rev{32/64/128/160/176/192 units} &
&\\
\hline

\multirow{6}{2.7cm}{\centering Nested Alt-GNN} &
Graph &
Nested graphs 0011/0001/0012 &
\multirow{6}{2.2cm}{\centering 993--26,627} &
\multirow{6}{1.0cm}{\centering 126}\\
&
Message &
Linear/MLP &
&\\
&
Aggregation &
Mean/logsum/max &
&\\
&
Update &
Concat &
&\\
&
Readout &
Linear/MLP &
&\\
&
Architecture &
1--2 layers; 32/64 units &
&\\
\hline

\multirow{6}{2.7cm}{\centering Complete Alt-GNN} &
Graph &
Complete graph &
\multirow{6}{2.2cm}{\centering 1,123--34,948} &
\multirow{6}{1.0cm}{\centering 42}\\
&
Message &
Linear/MLP &
&\\
&
Aggregation &
Mean/logsum/max &
&\\
&
Update &
Concat &
&\\
&
Readout &
Linear/MLP &
&\\
&
Architecture &
1--2 layers; 32/64 units &
&\\
\hline

\multirow{6}{2.7cm}{\centering\rev{Top-K Alt-GNN}} &
\rev{Graph} &
\rev{Population-level top-$k$ graph, $k=1/2/3$} &
\multirow{6}{2.2cm}{\centering\rev{1,139--34,964}} &
\multirow{6}{1.0cm}{\centering\rev{126}}\\
&
\rev{Message} &
\rev{Linear/MLP} &
&\\
&
\rev{Aggregation} &
\rev{Mean/logsum/max} &
&\\
&
\rev{Update} &
\rev{Concat} &
&\\
&
\rev{Readout} &
\rev{Linear/MLP} &
&\\
&
\rev{Architecture} &
\rev{1--2 layers; 32/64 units} &
&\\
\hline

\multirow{6}{2.7cm}{\centering Attention Alt-GNN} &
Graph &
Individual-specific attention graph &
\multirow{6}{2.2cm}{\centering 2,211--39,172} &
\multirow{6}{1.0cm}{\centering 42}\\
&
Message &
Linear/MLP &
&\\
&
Aggregation &
Mean/logsum/max &
&\\
&
Update &
Concat &
&\\
&
Readout &
Linear/MLP &
&\\
&
Architecture &
1--2 layers; 32/64 units &
&\\
\hline

\multicolumn{4}{l|}{\textbf{Total}} &
\rev{\textbf{370}}\\
\hline

\end{tabular}
\end{table}

\subsection{Model training and evaluation}

The dataset is randomly divided into the development and testing subsets with a ratio of 80:20, with the split performed at the household level based on household IDs. The development set is used to train model parameters, while the test set is used only for model evaluation. All models are trained by maximizing the log-likelihood function
\begin{equation}
\mathcal{LL}(\theta)
=
\sum_{n=1}^{N}
\sum_{i \in C_n}
y_{ni} \log P_{ni}(\theta),
\end{equation}
where $y_{ni}$ indicates the observed choice and $P_{ni}(\theta)$ is the choice probability defined in the previous section. Here, the training process is more similar to the classical DCMs because individual decision makers maintain the IID structure, which differs from the studies considering social network effects \citep{villarragaDesigningGraphConvolutional2025a}.

\rev{Within the development set, five-fold stratified cross-validation is used to compare all 370 model specifications defined in \cref{tab:model_config_summary}. All specifications are estimated using the common fixed learning rate of 0.001. In each fold, the training portion is used to fit all data-dependent preprocessing operations and estimate the model parameters, while the validation portion is used to evaluate the candidate configurations based on the mean validation log-likelihood per observation. The best configuration in each model family is selected according to this validation metric. Each selected configuration is then refitted, together with its preprocessing pipeline, using the complete development set. Finally the model performance is evaluated on the independent test set.}

We adopt various optimization strategies depending on the parameter scale. For models with a small number of parameters, such as MNL and NL, we use the L-BFGS optimizer with full-batch estimation, which corresponds to the standard maximum likelihood estimation procedure in DCMs. For deep learning models, we use the Adam optimizer with mini-batch stochastic gradient descent, where the log-likelihood is evaluated on mini-batches of observations. All training experiments were conducted on the HiPerGator high-performance
computing cluster at the University of Florida. \rev{The experiments were executed entirely on 54 CPUs, with 64GB of allocated memory. The cross-validation stage required approximately 313.3 CPU-hours
for LPMC and 79.0 CPU-hours for CMAP. Refitting the selected configurations and
evaluating them on the independent test sets required an additional 1.5 CPU-hours across the two datasets.}

Model performance is evaluated on the independent test set using
log-likelihood per observation ($LL_{\mathrm{obs}}$), prediction accuracy
(ACC), balanced accuracy (BACC), macro F1-score ($F1_{\mathrm{macro}}$), and
the multiclass Brier score (BS). The log-likelihood per observation is
computed using the same log-likelihood formulation as the training objective,
divided by the number of test observations:
\begin{equation}
LL_{\mathrm{obs}}
=
\frac{1}{N}
\sum_{n=1}^{N}
\log p_{n y_n},
\end{equation}
where $p_{n y_n}$ is the probability assigned by the model to the observed
choice of observation $n$.

Prediction accuracy measures the proportion of correctly predicted choices:
\begin{equation}
ACC =
\frac{1}{N}
\sum_{n=1}^{N}
\mathbf{1}\left(\hat{y}_n=y_n\right),
\end{equation}
where $\hat{y}_n$ denotes the predicted choice and $y_n$ denotes the observed
choice.

To account for class imbalance across alternatives, we report the macro
F1-score:
\begin{equation}
F1_{\mathrm{macro}}
=
\frac{1}{|C|}
\sum_{i\in C}
\frac{2\,Precision_i\,Recall_i}
{Precision_i+Recall_i},
\end{equation}
where
\begin{equation}
Precision_i
=
\frac{TP_i}{TP_i+FP_i},
\qquad
Recall_i
=
\frac{TP_i}{TP_i+FN_i}.
\end{equation}
The quantities $TP_i$, $FP_i$, and $FN_i$ are defined as
\begin{equation}
TP_i
=
\sum_{n=1}^{N}
\mathbf{1}(\hat{y}_n=i \land y_n=i),
\end{equation}
\begin{equation}
FP_i
=
\sum_{n=1}^{N}
\mathbf{1}(\hat{y}_n=i \land y_n\neq i),
\end{equation}
and
\begin{equation}
FN_i
=
\sum_{n=1}^{N}
\mathbf{1}(\hat{y}_n\neq i \land y_n=i).
\end{equation}

\rev{Because ordinary accuracy may be dominated by frequently chosen
alternatives, we additionally report balanced accuracy, defined as the
unweighted mean recall across alternatives:
\begin{equation}
BACC
=
\frac{1}{|C|}
\sum_{i\in C}
Recall_i.
\end{equation}}

\rev{Since discrete choice models produce choice probabilities rather than
only class labels, we also evaluate probabilistic prediction quality using
the multiclass Brier score:
\begin{equation}
BS
=
\frac{1}{N}
\sum_{n=1}^{N}
\sum_{i\in C}
\left[
p_{ni}-\mathbf{1}(y_n=i)
\right]^2,
\end{equation}
where $p_{ni}$ is the predicted probability that observation $n$ chooses
alternative $i$. Lower values of $BS$ indicate more accurate probabilistic
predictions.}

\section{Results}
\label{sec:results}
Here we investigate to what extent the Alt-GNNs can enhance predictive performance (RQ2) in Section \ref{sec:results_prediction}, support graph-based model design and extend the substitution patterns of benchmark DCMs (RQ3) in Section \ref{sec:results_interpretation_substitution}, and enable graph-based utility interpretation beyond classical GEV models (RQ4) in Section \ref{sec:results_interpretation_graph}. While we conducted experiments using both the LPMC and CMAP datasets, here, for simplicity, we present the results of only the LPMC dataset. The CMAP dataset yields similar results, which are presented in Appendix~\ref{sec:chicago_results}. 

\subsection{Alt-GNNs outperforming benchmark models}
\label{sec:results_prediction}
\cref{tab:best_model_comparison,tab:top5_models} summarize the predictive performance of 370 model configurations across nine model families. \Cref{tab:best_model_comparison} presents the best-performing model selected from each of the nine model families. The cross-validation columns report the
mean validation performance across five stratified folds, while the test
columns report performance after retraining each selected model on the full
development set and evaluating it once on the untouched test set. \cref{tab:top5_models} presents the five models achieving the highest performance in the test set along with their corresponding hyperparameters. The two tables yield three major findings. 

\begin{table*}[htbp]
\centering
\footnotesize
\renewcommand{\arraystretch}{1.12}
\setlength{\tabcolsep}{3.5pt}
\caption{Cross-validation and test set performance (LPMC data)}
\label{tab:best_model_comparison}

\resizebox{\textwidth}{!}{%
\begin{tabular}{lccccc|ccccc}
\hline
& \multicolumn{5}{c|}{\textbf{Five-fold cross-validation set (mean value)}}
& \multicolumn{5}{c}{\textbf{Independent test set}} \\
\cline{2-11}

\textbf{Model}
& \textbf{LL/obs. $\uparrow$}
& \textbf{Macro F1 $\uparrow$}
& \textbf{ACC $\uparrow$}
& \textbf{BACC $\uparrow$}
& \textbf{BS $\downarrow$}
& \textbf{LL/obs. $\uparrow$}
& \textbf{Macro F1 $\uparrow$}
& \textbf{ACC $\uparrow$}
& \textbf{BACC $\uparrow$}
& \textbf{BS $\downarrow$} \\
\hline

Attention Alt-GNN
& -0.6868
& 0.5429
& \textbf{0.7333}
& 0.5437
& 0.3806
& \textbf{-0.6962}
& 0.5414
& 0.7307
& 0.5429
& 0.3822 \\

Complete Alt-GNN
& -0.6846
& 0.5430
& 0.7330
& 0.5441
& 0.3802
& -0.6968
& 0.5404
& 0.7301
& 0.5417
& 0.3833 \\

\rev{Top-K Alt-GNN}
& \rev{\textbf{-0.6845}}
& \rev{0.5431}
& \rev{0.7328}
& \rev{0.5450}
& \rev{\textbf{0.3798}}
& \rev{-0.6966}
& \rev{\textbf{0.5419}}
& \rev{\textbf{0.7322}}
& \rev{\textbf{0.5434}}
& \rev{\textbf{0.3820}} \\

Nested Alt-GNN
& -0.6845
& \textbf{0.5434}
& 0.7326
& \textbf{0.5454}
& 0.3799
& -0.6981
& 0.5410
& 0.7312
& 0.5416
& 0.3835 \\

ASU-DNN
& -0.6881
& 0.5425
& 0.7323
& 0.5440
& 0.3807
& -0.6988
& 0.5410
& 0.7306
& 0.5424
& 0.3832 \\

\rev{ResLogit}
& \rev{-0.7028}
& \rev{0.5376}
& \rev{0.7245}
& \rev{0.5378}
& \rev{0.3916}
& \rev{-0.7105}
& \rev{0.5373}
& \rev{0.7237}
& \rev{0.5396}
& \rev{0.3926} \\

\rev{TasteNet-MNL}
& \rev{-0.7119}
& \rev{0.5409}
& \rev{0.7263}
& \rev{0.5490}
& \rev{0.3907}
& \rev{-0.7275}
& \rev{0.5361}
& \rev{0.7200}
& \rev{0.5458}
& \rev{0.3970} \\

MNL
& -0.7336
& 0.5327
& 0.7154
& 0.5413
& 0.4054
& -0.7508
& 0.5312
& 0.7131
& 0.5409
& 0.4109 \\

NL
& -0.7334
& 0.5335
& 0.7161
& 0.5420
& 0.4051
& -0.7504
& 0.5306
& 0.7125
& 0.5403
& 0.4108 \\

\hline
\end{tabular}%
}
\end{table*}

First, Alt-GNNs consistently outperform all benchmark models as shown in both \cref{tab:best_model_comparison,tab:top5_models}. In \cref{tab:best_model_comparison}, all Alt-GNNs, including Attention Alt-GNN, Complete Alt-GNN, \rev{Top-K Alt-GNN}, and Nested Alt-GNN, consistently outperform the ASU-DNN, \rev{ResLogit, TasteNet-MNL}, MNL, and NL models. For example, the highest-ranked Attention Alt-GNN achieves a test log-likelihood per observation of $-0.6962$ and an accuracy of $0.7307$, compared with $-0.7508$ and $0.7131$ for MNL, and $-0.7504$ and $0.7125$ for NL, respectively. These results correspond to reductions of approximately $7.3\%$ and $7.2\%$ in the absolute value of the negative log-likelihood per observation and relative improvements of $2.5\%$ and $2.6\%$ in prediction accuracy compared with MNL and NL, respectively. In \cref{tab:top5_models}, the five highest-ranked model specifications on the independent test set all belong to the Alt-GNN families. Two are Top-K Alt-GNN specifications, while the Attention, Complete, and Nested Alt-GNN families each contribute one specification. 
\revII{
To further assess the statistical uncertainty of
predictive performance, we conducted a household-cluster
bootstrap analysis using 5,000 resamples of the
independent test set
(Appendix~\ref{app:test_uncertainty}).
The results indicate that Alt-GNNs generally
outperform MNL, NL, and TasteNet-MNL on both
datasets, with statistically significant
improvements in log-likelihood per observation.
However, the advantages over stronger neural
baselines are less consistent. On LPMC, none of
the four Alt-GNN families significantly
outperforms ASU-DNN, whereas on CMAP, Complete
and Top-K Alt-GNN achieve statistically
significant improvements over ASU-DNN.
These findings support the predictive potential
of Alt-GNNs while highlighting that performance
advantages depend on the benchmark model and
dataset.
}

\begin{table}[htbp]
\small
\centering
\caption{Top five Alt-GNNs ranked by log-likelihood (LPMC data)}
\label{tab:top5_models}
\resizebox{\linewidth}{!}{%
\begin{tabular}{p{3.5cm}p{2.5cm}p{2.5cm}p{2.5cm}p{2.5cm}p{2.5cm}}
\hline
\textbf{Model rank}
& \textbf{1st}
& \textbf{2nd}
& \textbf{3rd}
& \textbf{4th}
& \textbf{5th} \\
\hline

\multicolumn{6}{l}{\textbf{Model specification}} \\
\hline

Model family
& Attention Alt-GNN
& Top-K Alt-GNN
& Top-K Alt-GNN
& Complete Alt-GNN
& Nested Alt-GNN \\

Graph structure
& Attention graph
& Learned binary graph ($k=3$)
& Learned binary graph ($k=3$)
& Complete graph
& Nested graph 3 ($0011$) \\

Message function
& MLP
& MLP
& MLP
& MLP
& MLP \\

Aggregation function
& LSE
& Mean
& Max
& Mean
& Mean \\

Update function
& Concat
& Concat
& Concat
& Concat
& Concat \\

Readout function
& Linear
& Linear
& Linear
& Linear
& Linear \\

Message-passing layers
& 2
& 1
& 1
& 2
& 2 \\

Hidden units
& 32
& 64
& 64
& 32
& 64 \\
\hline

\multicolumn{6}{l}{\textbf{Independent test-set performance}} \\
\hline

LL/observation
& \textbf{-0.696170}
& -0.696507
& -0.696595
& -0.696823
& -0.698132 \\

Macro F1
& 0.541368
& \textbf{0.542360}
& 0.541939
& 0.540400
& 0.540964 \\

Accuracy
& 0.730691
& \textbf{0.732848}
& 0.732170
& 0.730136
& 0.731184 \\

Balanced accuracy
& 0.542918
& \textbf{0.543507}
& 0.543352
& 0.541735
& 0.541564 \\

Brier score
& 0.382190
& 0.382036
& \textbf{0.382026}
& 0.383255
& 0.383474 \\
\hline
\end{tabular}%
}
\end{table}

Second, the high predictive performance of Alt-GNNs is supported by the flexibility of alternative graph design. \rev{For example, as shown in \cref{tab:top5_models}, the top five models include two Top-K Alt-GNN specifications, one Nested Alt-GNN specification, one Attention Alt-GNN specification, and one Complete Alt-GNN specification.} The strong performance
of several distinct graph structures suggests that the representation of
inter-alternative relationships is consequential, while no single graph
family dominates across all specifications and evaluation metrics. Complete graphs enable information exchange among all alternatives, observation-specific attention graphs allow interaction weights to vary across observations, and Top-K graphs identify a sparse, population-shared interaction structure by retaining the top-\(k\) learned edges for each alternative. These results
therefore support the value of flexible graph representations, rather than
establishing the universal superiority of one graph family. \revII{
To verify the robustness of the learned Top-K
graph structures, we conducted a stability analysis
of the two highest-ranked Top-K specifications on
each dataset using 20 random seeds
(Appendix~\ref{app:graph_stability}).
The learned graphs exhibited statistically significant
agreement beyond random Top-K selection, although
the chance-adjusted agreement remained modest.
These findings suggest that Top-K Alt-GNN can learn
partially reproducible graph structures, while
individual selected edges may still be sensitive
to random initialization and model configuration.
} Besides, \cref{tab:best_model_comparison} demonstrates that Nested Alt-GNN outperforms ASU-DNN, while NL outperforms MNL, also highlighting the contribution of capturing alternative dependence to model performance. In both Nested Alt-GNN and NL, message passing across alternatives allows the model to capture dependence among alternatives, while ASU-DNN and MNL treat alternatives independently. Therefore, these results demonstrate the advantages of our Alt-GNN models. By incorporating an alternative graph, Alt-GNN captures alternative dependence through flexible alternative graphs, enriching the utility representation beyond what NL and ASU-DNN can provide. 


Third, the high performance of the Alt-GNN models partly can be attributed to the hyperparameters deviating from those standard ones in classical DCMs. As shown in \cref{tab:best_model_comparison}, while NL is a unique case of Nested Alt-GNN, the hyperparameters of the highest-performing Nested Alt-GNN are quite different from the NL model. Specifically, the optimum hyperparameters include MLP message function, Mean
aggregation function, concatenation updating function, and linear readout
function - every function is different from the predefined hyperparameters
in NL models. Consequently, its messages remain multidimensional throughout message
passing and are mapped to scalar alternative utilities only at the readout
stage. In contrast, the message function in NL is a linear transformation that directly produces a scalar. A similar pattern can be observed when comparing ASU-DNN and MNL. MNL represents each
alternative's systematic utility as a linear function of observed
attributes, whereas ASU-DNN uses alternative-specific neural networks to
capture nonlinear relationships among the inputs. Overall, the rich hyperparameter space in Alt-GNNs provides more flexible functional forms than the standard specifications in classical DCMs, which contributes to their performance differences.

\subsection{Elasticities and substitution patterns}
\label{sec:results_interpretation_substitution}
This subsection addresses RQ3 by examining how alternative graph design in Alt-GNNs enables flexible yet behaviorally constrained elasticity and substitution patterns. We present a sequence of models from MNL and ASU-DNN to NL and Nested Alt-GNN, and finally Complete and Attention Alt-GNNs, revealing the process of imposing and gradually relaxing behavioral constraints. 

The elasticity results are summarized in \cref{tab:elasticity_mnl_dnn,tab:elasticity_nl_nested_altgnn,tab:elasticity_attention_complete_altgnn}. Each row corresponds to an alternative-specific variable (e.g., automobile time, transit cost), and each column represents a particular travel mode. \rev{For each observation in the independent test set, the point elasticity of the predicted probability of choosing alternative $i$ with respect to
attribute $x_j$ is calculated as $E_{ij} = \frac{\partial P_i}{\partial x_j}
\frac{x_j}{P_i}$, evaluated at that observation's observed attribute values and predicted
choice probabilities. Because the models use min--max normalized inputs,
the derivatives are transformed back to the original units before the
elasticities are calculated. The tables report the mean individual-level
elasticity across test observations, with the corresponding standard
deviation in parentheses.} The elasticities can be revealed by visualizing the substitution patterns, as shown in \cref{fig:auto_time_mnl_dnn,fig:auto_time_nl_gnn,fig:auto_time_attention_complete_altgnn}. 
\rev{To construct these curves, all other inputs are fixed at their test-set sample means, except
for automobile travel time, which is varied over its observed test-set
range. The solid lines show the resulting predicted choice probabilities,
whereas the dashed lines show the probability ratios between pairs of
non-automobile travel modes.} \rev{As a supplementary behavioral diagnostic, the value of travel time (VOTT) was calculated for each model as the ratio of the marginal utility of travel time to the marginal utility of travel cost, converted to GBP per hour (Appendix~\ref{app:behavioral_analysis}).} 

Consistent with past studies, ASU-DNN resembles the elasticities and substitution patterns of MNL. As shown in \cref{tab:elasticity_mnl_dnn}, the elasticities of automobile, bike, and walking with respect to transit time are consistently 0.68 (std = 0.74) in the MNL model and 1.05 (std = 1.04) in the ASU-DNN model. Similar patterns appear across other variables. For example, increases in automobile time lead to equal elasticities of 0.68 for transit, bike, and walking in MNL, and 1.05 in ASU-DNN. This pattern arises because both models lack mechanisms for inter-mode interaction. Changes in a mode-specific variable (e.g., transit time) affect all other modes in a uniform and symmetric manner, reflecting the IIA property in both MNL and ASU-DNN. Similar to the elasticity values, ASU-DNN resembles MNL's substitution pattern as shown by the flat lines of choice probability ratios. In \cref{fig:auto_time_mnl_dnn}, as automobile travel time increases, the predicted probability of choosing an automobile (the blue solid line) decreases, while the probabilities of choosing transit, bike, and walking (the yellow, green, and red solid lines, respectively) all increase. Importantly, the relative proportions among the three non-automobile modes remain unchanged: the choice probability ratios between transit, bike, and walking stay constant across the full range of automobile costs, as visually shown by the flat blue, yellow, and green dashed lines. This invariance property holds for both MNL and ASU-DNN models, reflecting the shared IIA constraint. Although these findings mainly replicate results from past studies \citep{wangDeepNeuralNetworks2020a}, here we present both models as special cases within the Alt-GNN model family, rather than contrasting them as separate DCM and DNN paradigms. 

\begin{table}[htbp]
\centering
\caption{Elasticities of MNL and ASU-DNN (LPMC data)}
\resizebox{\textwidth}{!}{%
\begin{tabular}{lcccc|cccc}
\toprule
& \multicolumn{4}{c|}{\textbf{MNL}}
& \multicolumn{4}{c}{\textbf{ASU-DNN}} \\
& Automobile & Transit & Bike & Walking
& Automobile & Transit & Bike & Walking \\
\midrule
Automobile time
& -1.06 (1.44) & \textbf{0.68 (0.74)} & \textbf{0.68 (0.74)} & \textbf{0.68 (0.74)}
& -1.12 (1.05) & \textbf{1.05 (1.04)} & \textbf{1.05 (1.04)} & \textbf{1.05 (1.04)} \\

Automobile cost
& -0.23 (0.52) & \textbf{0.07 (0.10)} & \textbf{0.07 (0.10)} & \textbf{0.07 (0.10)}
& -0.02 (0.85) & \textbf{-0.12 (0.24)} & \textbf{-0.12 (0.24)} & \textbf{-0.12 (0.24)} \\

Transit time
& \textbf{0.57 (0.72)} & -0.77 (0.67) & \textbf{0.57 (0.72)} & \textbf{0.57 (0.72)}
& \textbf{0.64 (0.83)} & -0.94 (0.92) & \textbf{0.64 (0.83)} & \textbf{0.64 (0.83)} \\

Transit cost
& \textbf{0.16 (0.28)} & -0.21 (0.23) & \textbf{0.16 (0.28)} & \textbf{0.16 (0.28)}
& \textbf{-0.01 (0.26)} & -0.11 (0.27) & \textbf{-0.01 (0.26)} & \textbf{-0.01 (0.26)} \\

Bike time
& \textbf{0.05 (0.07)} & \textbf{0.05 (0.07)} & -1.98 (1.98) & \textbf{0.05 (0.07)}
& \textbf{0.05 (0.05)} & \textbf{0.05 (0.05)} & -2.32 (1.44) & \textbf{0.05 (0.05)} \\

Walking time
& \textbf{0.42 (0.47)} & \textbf{0.42 (0.47)} & \textbf{0.42 (0.47)} & -8.96 (9.68)
& \textbf{0.56 (0.76)} & \textbf{0.56 (0.76)} & \textbf{0.56 (0.76)} & -4.46 (2.71) \\

\bottomrule
\end{tabular}
}
\label{tab:elasticity_mnl_dnn}
\end{table}

\begin{figure}[H]
\centering
\includegraphics[width=\textwidth]
{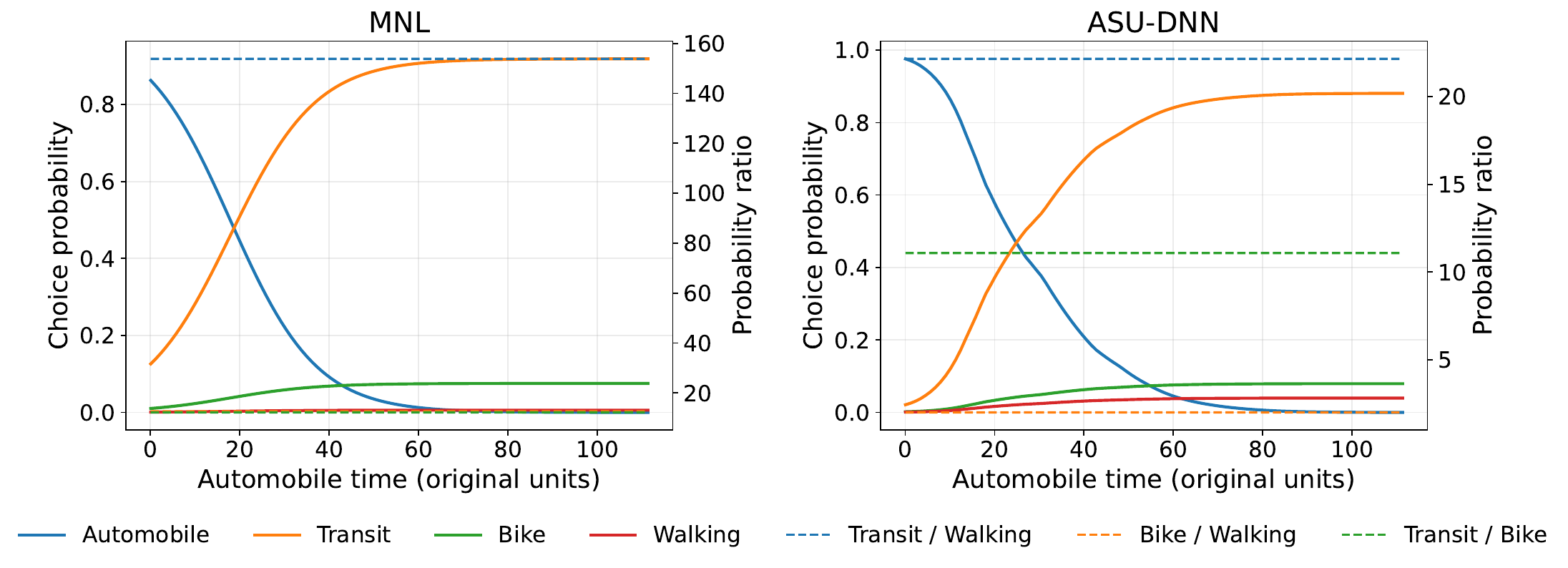}
\caption{Substitution patterns of MNL and ASU-DNN}
\label{fig:auto_time_mnl_dnn}
\end{figure}

Nested Alt-GNNs resemble the elasticities, substitution patterns, and the optimal nesting structures in the NL models, demonstrating a successful process of relaxing the IIA constraints in ASU-DNN while imposing behavioral knowledge. As shown in \cref{tab:elasticity_nl_nested_altgnn}, in the NL model, the elasticity of bike and walking (both in the non-motorized component) with respect to automobile time is 0.02 (std = 0.02), whereas transit, which belongs to the same component as automobile, has a much higher elasticity of 0.73 (std = 0.72). In the Nested Alt-GNN model, a similar elasticity pattern is observed: bike and walking have equal elasticities of 0.69 (std = 0.79) with respect to automobile time, while transit shows a higher elasticity of 1.25 (std = 1.04). Therefore, in both NL and Nested Alt-GNN, the alternatives within the same component, i.e., nest in NL or subnetwork in Nested Alt-GNN, tend to have identical elasticities with respect to an alternative-specific variable in another component, while alternatives' elasticities in the same component are different. Visually, the Nested Alt-GNN resembles NL in its unique two-layer substitution pattern: when the mode-specific variable in one component (e.g., automobile cost) changes, the probability ratio among modes within the other component (e.g., bike/walking) remains constant. As reflected in \cref{fig:auto_time_nl_gnn}, the cross-nest probability ratios, such as transit/walking and transit/bike, vary with automobile travel time, while the within-nest bike/walking ratio remains relatively stable. \rev{After the main transition region, the decline in automobile probability is associated with a larger increase in transit probability than in bike or walking probability, which is consistent with the stronger within-nest substitution between automobile and transit.} This pattern is also consistent with the elasticity results in \cref{tab:elasticity_nl_nested_altgnn}. \rev{These results show that the nested alternative graph can preserve the substitution structure of NL while allowing a broader model design space, with NL represented as a particular case.}


\begin{table}[htbp]
\centering
\caption{Elasticities of NL and Nested Alt-GNN (LPMC data)}
\resizebox{\textwidth}{!}{%
\begin{tabular}{lcccc|cccc}
\toprule
& \multicolumn{4}{c|}{\textbf{NL (0011)}}
& \multicolumn{4}{c}{\textbf{Nested Alt-GNN (0011)}} \\
& Automobile & Transit & Bike & Walking
& Automobile & Transit & Bike & Walking \\
\midrule

Automobile time
& -0.97 (1.44)
& 0.73 (0.72)
& \textbf{0.02 (0.02)}
& \textbf{0.02 (0.02)}
& -1.03 (0.98)
& 1.25 (1.04)
& \textbf{0.69 (0.79)}
& \textbf{0.69 (0.79)} \\

Automobile cost
& -0.22 (0.50)
& 0.08 (0.12)
& \textbf{0.00 (0.00)}
& \textbf{0.00 (0.00)}
& -0.09 (0.83)
& -0.03 (0.15)
& \textbf{-0.06 (0.20)}
& \textbf{-0.06 (0.20)} \\

Transit time
& 0.67 (0.76)
& -0.74 (0.74)
& \textbf{0.02 (0.02)}
& \textbf{0.02 (0.02)}
& 0.67 (0.88)
& -1.21 (1.00)
& \textbf{0.77 (0.80)}
& \textbf{0.77 (0.80)} \\

Transit cost
& 0.13 (0.20)
& -0.13 (0.16)
& \textbf{0.00 (0.01)}
& \textbf{0.00 (0.01)}
& -0.05 (0.21)
& -0.15 (0.31)
& \textbf{0.05 (0.20)}
& \textbf{0.05 (0.20)} \\

Bike time
& \textbf{0.01 (0.02)}
& \textbf{0.01 (0.02)}
& -0.57 (0.57)
& 0.01 (0.02)
& \textbf{-0.01 (0.03)}
& \textbf{-0.01 (0.03)}
& 0.40 (0.50)
& 0.09 (0.30) \\

Walking time
& \textbf{0.34 (0.37)}
& \textbf{0.34 (0.37)}
& 0.34 (0.37)
& -7.18 (7.76)
& \textbf{0.60 (0.77)}
& \textbf{0.60 (0.77)}
& -2.43 (2.55)
& -4.70 (3.19) \\

\bottomrule
\end{tabular}%
}
\label{tab:elasticity_nl_nested_altgnn}
\end{table}

\begin{figure}[H]
\centering
\includegraphics[width=\textwidth]
{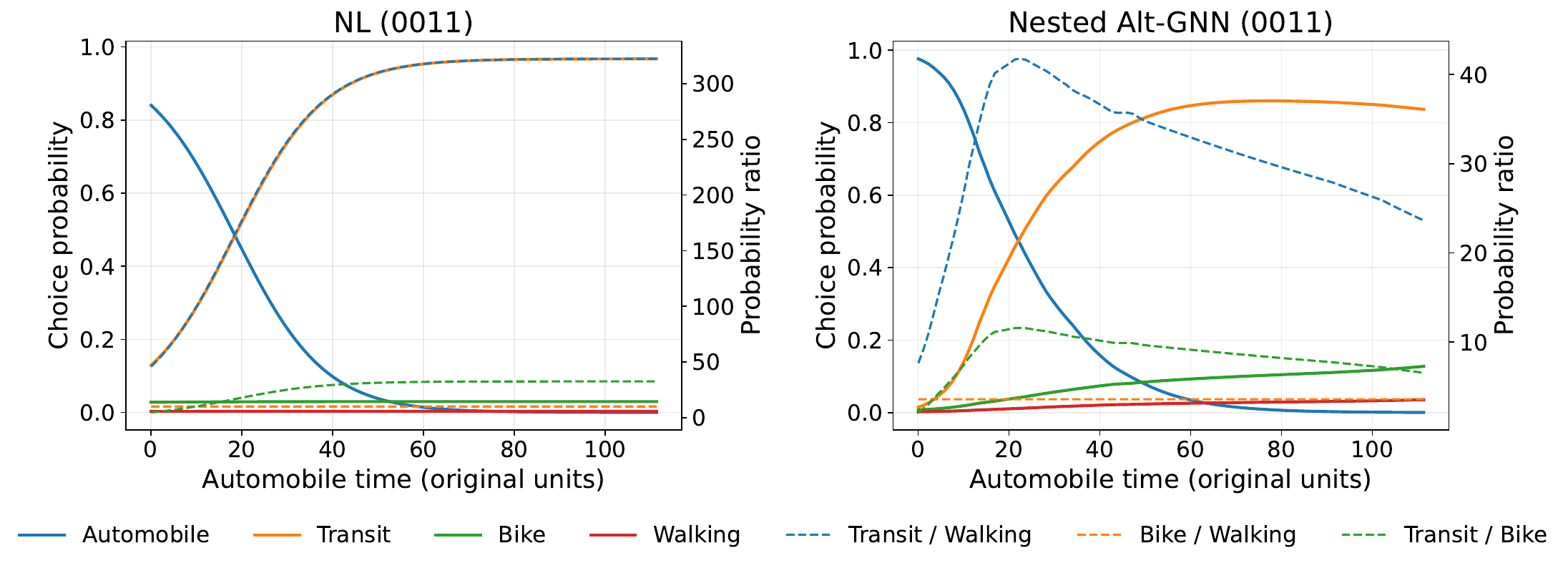}
\caption{Substitution patterns of NL and Nested Alt-GNN}
\label{fig:auto_time_nl_gnn}
\end{figure}

Complete and Attention Alt-GNNs further relax the behavioral constraints by reaching the most flexible forms of Alt-GNNs. Overall, both models reveal reasonable behavior patterns. In Table \ref{tab:elasticity_attention_complete_altgnn}, both Complete Alt-GNN and Attention Alt-GNN models present negative self-elasticities for travel time and overall positive cross-elasticities, indicating substitution effects among travel modes. Similarly, in \cref{fig:auto_time_attention_complete_altgnn}, both Attention and Complete Alt-GNNs illustrate reasonable substitution patterns when automobile travel time increases. Specifically, the probability of choosing an automobile decreases sharply, while the probability of choosing transit increases significantly, indicating a strong substitution effect from automobile to transit. In contrast, the probabilities of bike and walking change only slightly, suggesting that increases in automobile travel time mainly shift demand toward transit rather than active modes. The elasticities and substitution patterns of both Complete and Attention Alt-GNNs somehow resemble those from the NL and Nested Alt-GNN (\cref{fig:auto_time_nl_gnn}), reinforcing the behavioral validity of the highly flexible Alt-GNN models. 

\rev{Compared with the more structured Nested Alt-GNN, the Attention and Complete Alt-GNNs exhibit more irregular substitution patterns.} As shown in \cref{fig:auto_time_attention_complete_altgnn}, the predicted automobile probability decreases and the transit probability increases under both models, while several probability ratios display nonlinear changes over the observed automobile-time range. Several probability ratios become large toward the upper end of the automobile-time range because the corresponding denominator approaches zero. These large ratios should therefore not be interpreted as equally large absolute transfers of demand between alternatives. \rev{This comparison highlights the role of graph design in controlling model flexibility. Nested Alt-GNN builds on the nested structure of the NL model and preserves its within-nest and between-nest substitution relationships, whereas Attention and Complete Alt-GNN relax these structural restrictions and allow more flexible substitution patterns after estimation. Nonetheless, such flexibility also leads to irregular behavioral patterns, since such a graph design approach doesn't impose any hard constraint on behavioral regularity.}

\begin{table}[htbp]
\centering
\caption{Elasticities of Attention Alt-GNN and Complete Alt-GNN (LPMC data)}
\resizebox{\textwidth}{!}{%
\begin{tabular}{lcccc|cccc}
\toprule
& \multicolumn{4}{c|}{\textbf{Attention Alt-GNN (LSE)}}
& \multicolumn{4}{c}{\textbf{Complete Alt-GNN (mean)}} \\
& Automobile & Transit & Bike & Walking
& Automobile & Transit & Bike & Walking \\
\midrule

Automobile time
& \textbf{-1.07 (1.12)}
& 1.16 (1.01)
& 0.50 (0.71)
& 0.40 (0.77)
& \textbf{-1.01 (1.00)}
& 1.14 (0.99)
& 0.91 (0.85)
& 0.37 (0.69) \\

Automobile cost
& \textbf{-0.23 (0.85)}
& 0.02 (0.14)
& 0.02 (0.27)
& 0.03 (0.29)
& \textbf{-0.11 (0.68)}
& -0.04 (0.16)
& -0.01 (0.17)
& 0.01 (0.19) \\

Transit time
& 0.71 (0.87)
& \textbf{-1.22 (0.91)}
& 0.81 (0.83)
& 0.82 (0.88)
& 0.63 (0.80)
& \textbf{-1.05 (0.90)}
& 0.62 (0.84)
& 0.88 (0.94) \\

Transit cost
& -0.00 (0.27)
& \textbf{-0.17 (0.34)}
& -0.02 (0.23)
& -0.02 (0.27)
& -0.05 (0.25)
& \textbf{-0.19 (0.34)}
& 0.08 (0.25)
& 0.05 (0.18) \\

Bike time
& 0.06 (0.32)
& 0.08 (0.15)
& \textbf{-0.84 (1.10)}
& -0.04 (0.33)
& -0.21 (0.24)
& 0.24 (0.23)
& \textbf{-0.36 (0.83)}
& -0.02 (0.23) \\

Walking time
& 0.59 (0.70)
& 0.62 (1.10)
& -0.94 (1.20)
& \textbf{-4.27 (3.62)}
& 0.80 (0.66)
& 0.43 (0.91)
& -1.70 (2.36)
& \textbf{-5.10 (4.49)} \\

\bottomrule
\end{tabular}%
}
\label{tab:elasticity_attention_complete_altgnn}
\end{table}

\begin{figure}[H]
\centering
\includegraphics[width=\textwidth]
{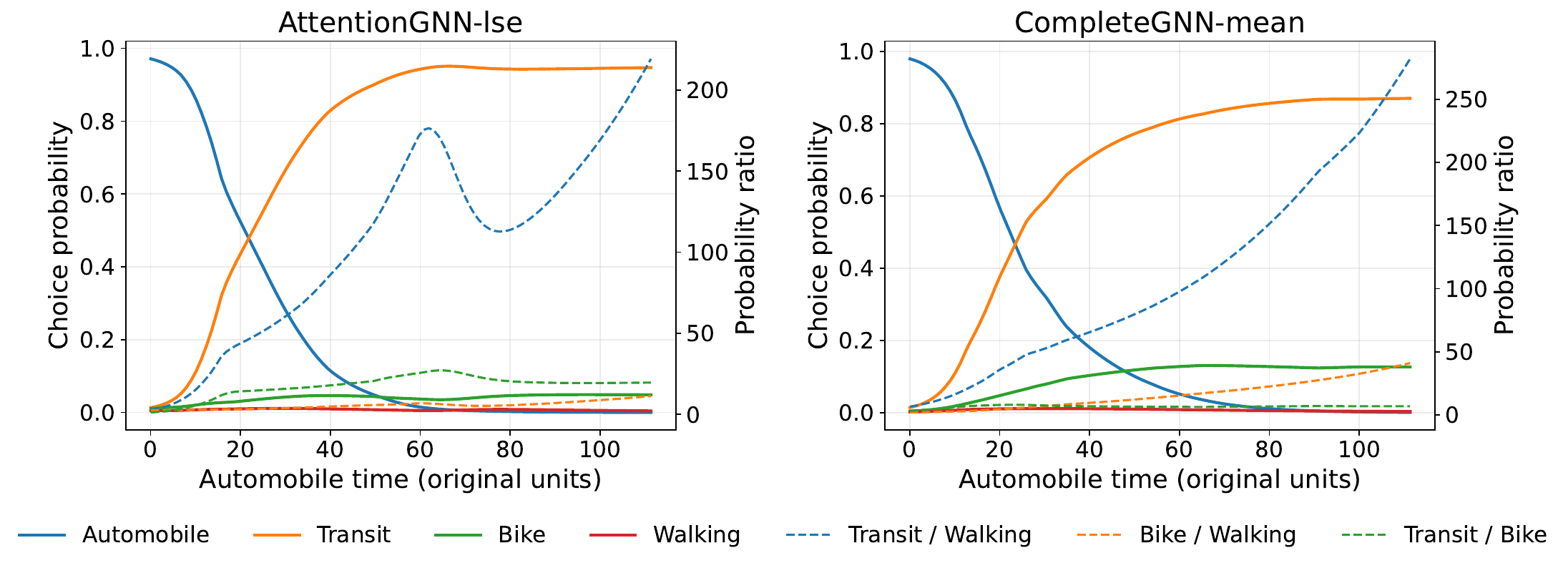}
\caption{Substitution patterns of Attention Alt-GNN (LSE) and Complete
Alt-GNN (Mean)}
\label{fig:auto_time_attention_complete_altgnn}
\end{figure}

Overall, our empirical results demonstrate how to impose behavioral constraints to varying degrees using alternative graphs. The MNL and ASU-DNN are the most restrictive following the classical IIA constraint; the NL and Nested Alt-GNNs relax some constraints by following the two-layer substitution patterns in the classical NL model; the Complete and Attention Alt-GNNs are the most flexible forms with still reasonable elasticity patterns. Meanwhile, Alt-GNNs present significant flexibility in model specification due to the high-dimensional hyperparameter space. While we tested only three NL specifications, we can easily examine hundreds of Nested Alt-GNN models, all of which retain the core substitution properties of NL models. 

\subsection{Graph-based model interpretation}
\label{sec:results_interpretation_graph}
\rev{This subsection addresses RQ4 by examining how predefined and learned alternative graph representations can be interpreted in Alt-GNNs. We interpret these graph representations as components that govern information flow and utility interactions among alternatives, rather than as direct measures of behavioral substitution, causal relationships, or unobserved utility correlation.
} In particular, the \rev{Top-K} and Attention Alt-GNNs provide graph-based interpretations characterized by four non-traditional properties: (1) automatic learning of alternative graphs rather than requiring modelers to pre-specify them, (2)  non-nested topologies that go beyond classical nesting structures, (3) directed structures that allow asymmetric utility dependencies across alternatives, and (4) heterogeneous dependence patterns that vary across individuals. Specifically, \rev{\cref{fig:discrete_learned_graphs} presents the alternative graphs from the two highest-ranked Top-K Alt-GNNs, illustrating non-nested interaction structures and asymmetric utility dependencies.} \cref{fig:attention_weight_distributions_layer_1} presents the test-set distributions of the learned attention weights, $\alpha_{nij}$, across directed mode pairs, showing how the Attention Alt-GNN differentiates interaction strengths across alternative pairs and individuals. Finally, we apply hierarchical clustering to the individual-level attention-weight vectors. The resulting dendrogram is presented in \cref{fig:attention_dendrogram_layer_1}, and the representative cluster-average graphs are visualized in \cref{fig:attention_clusters_layer1}; together, they reveal systematic heterogeneity in the alternative dependence structures learned across the population.

\begin{figure}[htbp]
\centering

\begin{subfigure}[t]{0.4\textwidth}
    \centering
    \includegraphics[width=\textwidth]
    {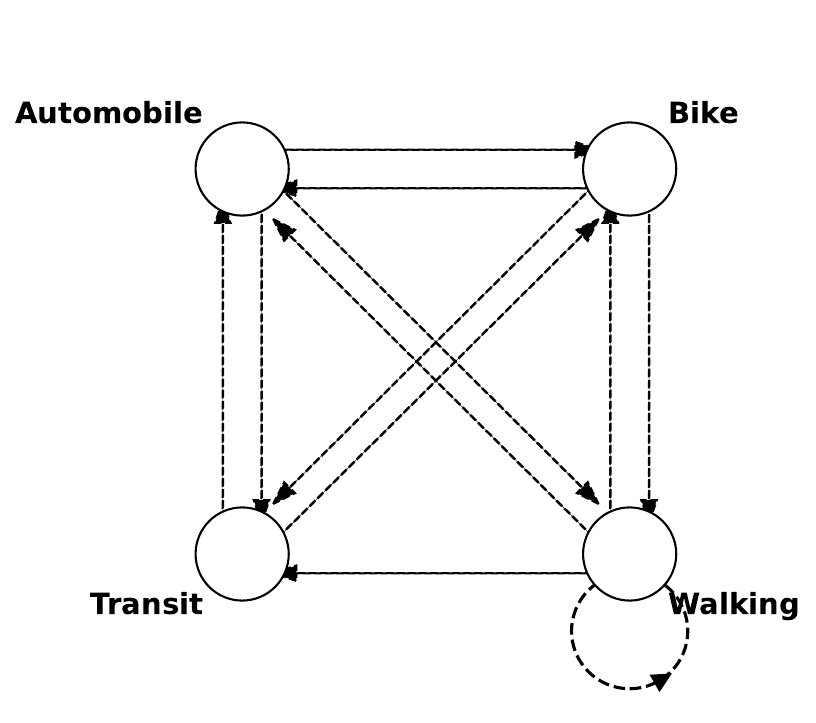}
    \caption{Top-K Alt-GNN ($k=3$, mean)}
    \label{fig:discrete_graph_k3_mean}
\end{subfigure}
\hfill
\begin{subfigure}[t]{0.4\textwidth}
    \centering
    \includegraphics[width=\textwidth]
    {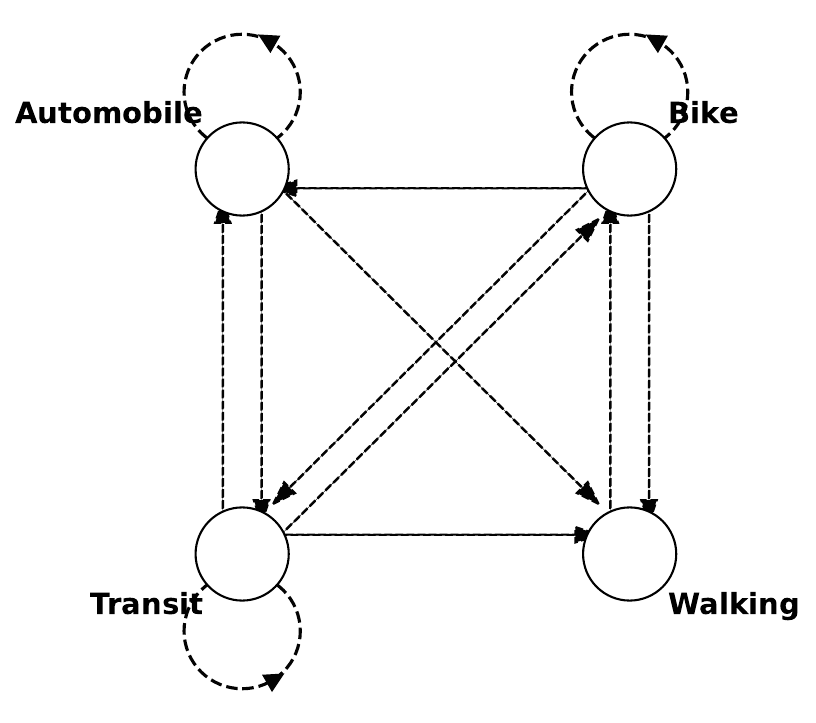}
    \caption{Top-K Alt-GNN ($k=3$, max)}
    \label{fig:discrete_graph_k3_max}
\end{subfigure}

\caption{Alternative graphs from the two highest-ranked Top-K Alt-GNNs (LPMC data).}
\label{fig:discrete_learned_graphs}
\end{figure}

Property~(1) - automatic learning of the alternative graph - is demonstrated
through \rev{both the discrete edge-selection scores $\theta_{ij}$ of the
Top-K Alt-GNN and} the attention weights $\alpha_{nij}$
\rev{of the Attention Alt-GNN}. 
\rev{For the Attention Alt-GNN,} $\alpha_{nij}$ \rev{is} defined in
Equation~\eqref{attention_weight} as a normalized function \rev{of the
projected representations of target alternative $i$ and source alternative
$j$ for each individual $n$}. Unlike the nested alternative graphs in
\cref{fig:nesting_structures}, where modelers pre-specify which
alternatives exchange messages and with what topology, \rev{the Top-K and Attention Alt-GNNs learn alternative dependence structures from data through discrete edge selection and continuous edge weighting}. \cref{fig:attention_weight_distributions_layer_1} shows the
distribution of $\alpha_{nij}$ across all test-set individuals, with each
subplot corresponding to the edge weights of a directed mode pair (e.g.,
walking to transit or transit to walking). The edge weights differ
substantially across individuals: some pairs consistently receive high
attention weights close to one while others remain concentrated near zero.

\begin{figure}[htbp]
\centering
\includegraphics[width=0.9\textwidth]
{figures/attention\_weight\_distributions\_layer\_1.pdf}
\caption{Distribution of attention weights in the Attention Alt-GNNs}
\label{fig:attention_weight_distributions_layer_1}
\end{figure}

Properties~(2) and~(3) - directed asymmetric utility dependencies and non-nested topologies - are demonstrated through \rev{\cref{fig:discrete_learned_graphs}, which presents the alternative graphs from the two highest-ranked Top-K Alt-GNN specifications. For example, in \cref{fig:discrete_graph_k3_mean}, the edge from walking to transit is
selected, whereas the reverse edge from transit to walking is not. Similarly,
in \cref{fig:discrete_graph_k3_max}, bike sends a message to automobile, whereas automobile does not send a message to bike. These unidirectional edges show that one alternative's representation can shape the aggregated message of another without requiring the converse interaction.} Such asymmetric dependencies are structurally excluded from the nested logit framework, where alternatives within the same nest exchange messages through symmetric and undirected edges of equal weight, imposing proportional substitution patterns within each nest. Moreover, neither of the learned discrete graphs in \cref{fig:discrete_learned_graphs} conforms to any of the four nested alternative graphs in \cref{fig:nesting_structures}, confirming that the Alt-GNN can capture non-nested topologies that go beyond predefined nesting structures. 

Property~(4) - heterogeneous dependence patterns across individuals - is demonstrated by the substantial variation in $\alpha_{nij}$ visible across the marginal weight distributions in \cref{fig:attention_weight_distributions_layer_1}, where different directed mode pairs show markedly different distributional shapes. To further characterize this heterogeneity, we apply hierarchical clustering with cosine distance to individual attention-weight vectors. \cref{fig:attention_dendrogram_layer_1} presents the resulting dendrogram, which reveals a clear separation at a relatively high cosine distance, indicating three major clusters of interaction structures.

\begin{figure}[htbp]
    \centering
    \includegraphics[width=0.9\linewidth]
    {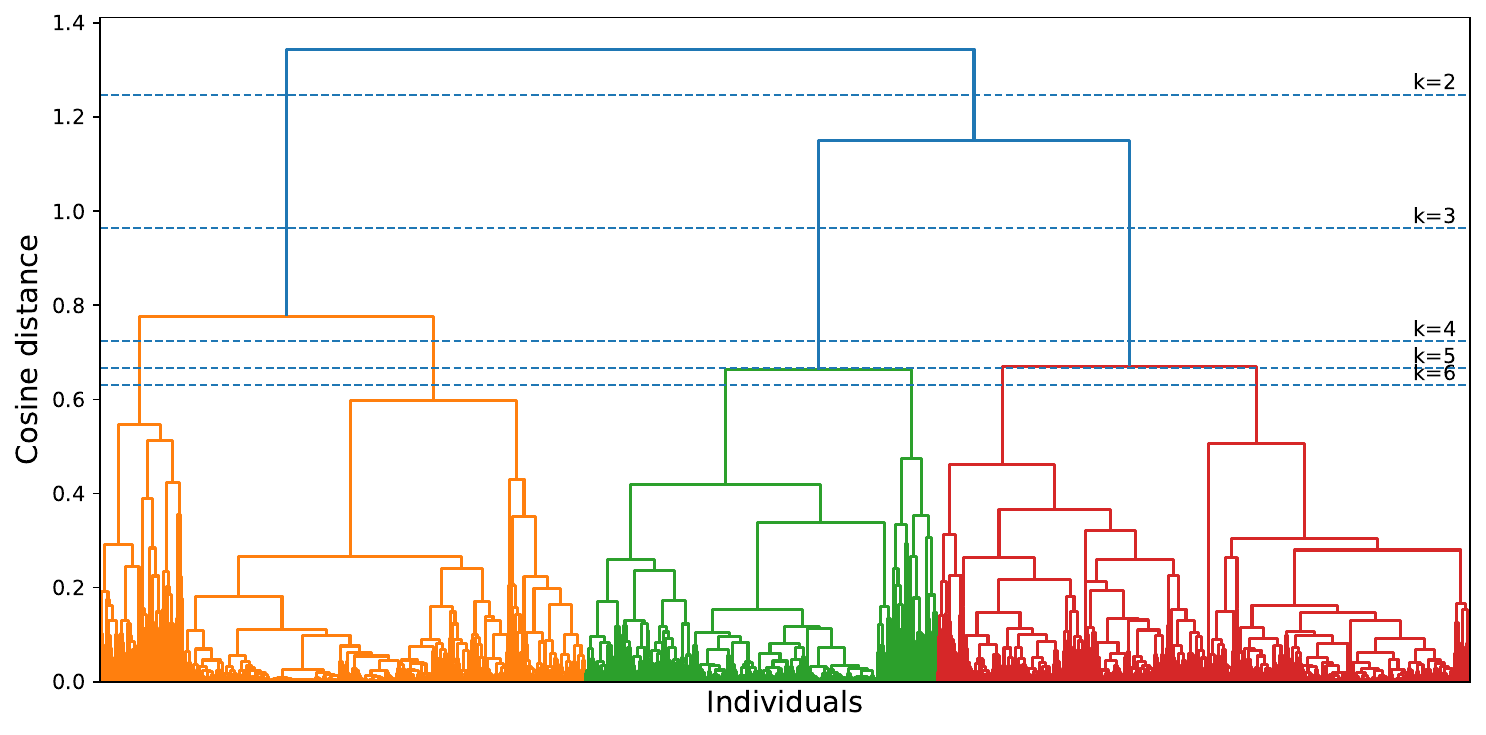}
    \caption{Hierarchical clustering dendrogram of individual attention weight patterns.}
    \label{fig:attention_dendrogram_layer_1}
\end{figure}

\begin{figure}[H]
\centering

\begin{subfigure}[t]{0.32\textwidth}
    \centering
    \includegraphics[width=\textwidth]
    {figures/cluster\_1\_centroid\_layer\_1.pdf}
    \caption{Cluster 1 ($n=5{,}748$)}
    \label{fig:attention_cluster1_layer1}
\end{subfigure}
\hfill
\begin{subfigure}[t]{0.32\textwidth}
    \centering
    \includegraphics[width=\textwidth]
    {figures/cluster\_2\_centroid\_layer\_1.pdf}
    \caption{Cluster 2 ($n=6{,}302$)}
    \label{fig:attention_cluster2_layer1}
\end{subfigure}
\hfill
\begin{subfigure}[t]{0.32\textwidth}
    \centering
    \includegraphics[width=\textwidth]
    {figures/cluster\_3\_centroid\_layer\_1.pdf}
    \caption{Cluster 3 ($n=4{,}173$)}
    \label{fig:attention_cluster3_layer1}
\end{subfigure}

\caption{Cluster-centroid attention graphs from Attention Alt-GNNs}
\label{fig:attention_clusters_layer1}
\end{figure}

The three clusters correspond to behaviorally distinct travel contexts and produce three interpretable alternative graphs. As shown in \cref{tab:cluster_profile_lpmc}, Cluster~1 ($n = 5{,}748$) contains the longest trips and oldest individuals, with the highest male share and household vehicle ownership. Automobile is the dominant observed mode. Its attention graph shows relatively large weights on messages originating from bike. Cluster~2 ($n = 6{,}302$) contains the shortest trips and middle-aged individuals, with walking as the most frequently observed mode. Its graph places relatively large weights on messages originating from walking. Cluster~3 ($n = 4{,}173$) represents intermediate-distance trips and the youngest individuals, with transit as the dominant observed mode. Its graph shows relatively large weights on messages originating from automobile. These differences indicate that individuals with different demographic and trip characteristics are associated with distinct patterns of relational weighting across alternatives.

\begin{table}[htbp]
\centering
\small
\caption{Behavioral characteristics of the three clusters identified from attention patterns}
\label{tab:cluster_profile_lpmc}

\begin{tabular}{lccccc}
\toprule
& Age & Male (\%) & Vehicles & Household Size & Distance (m) \\
\midrule
Cluster 1 & 50.4 & 62.2 & 1.32 & 2.20 & 7,944 \\
Cluster 2 & 39.5 & 38.8 & 0.88 & 2.46 & 1,238 \\
Cluster 3 & 24.5 & 39.0 & 0.71 & 2.63 & 4,930 \\
\midrule
& Automobile Time & Transit Time & Walking Time & Bike Time & \\
\midrule
Cluster 1 & 26.8 & 41.2 & 115.1 & 36.5 & \\
Cluster 2 & 6.1  & 12.4 & 19.6  & 6.6  & \\
Cluster 3 & 19.1 & 32.4 & 72.8  & 23.4 & \\
\midrule
& Automobile (\%) & Transit (\%) & Walking (\%) & Bike (\%) & \\
\midrule
Cluster 1 & 57.6 & 37.9 & 2.2  & 2.3 & \\
Cluster 2 & 37.0 & 19.5 & 39.9 & 3.6 & \\
Cluster 3 & 36.5 & 55.5 & 5.1  & 2.9 & \\
\bottomrule
\end{tabular}
\end{table}

Together, the four properties illustrate how the two kinds of self-learned graph Alt-GNNs provide graph-based model interpretation, by learning population-level or individual-specific, directed, and non-nested alternative graphs. Such properties are unavailable in classical GEV models, which are constrained to pre-specified and symmetric nesting structures.

\section{Conclusions}
\label{sec:conclusions}
Travel demand modeling has long relied on GEV models to represent structured dependence among alternatives, yet these models require alternative dependence to be predefined, symmetric, and shared uniformly across all individuals. The requirement for pre-specified nesting structures, combined with computational burdens in large-scale network-based choice problems, remains an enduring challenge for the DCM community. \rev{DNNs offer greater flexibility in utility specification but often lack an explicit representation of dependence among alternatives. ASU-DNN introduces alternative-specific subnetworks to better structure utility functions, but it still does not explicitly model inter-alternative relationships. Existing studies combining DCMs and DNNs mainly rely on feedforward architectures and have not explored how graph representations of alternative dependence can bridge GEV models and GNN architectures.} Therefore, we introduce the Alt-GNN framework that achieves this synergy along with multi-faceted goals. Corresponding to the four research questions, this study yields four major findings.

First, we make theoretical contributions by proposing Alternative Graph Neural Networks, a GNN framework that captures alternative dependence through an alternative graph and message passing algorithms. \rev{Alt-GNNs establish connections with existing models by representing MNL, NL, and ASU-DNN as special cases, while also enabling new model designs such as Nested Alt-GNN, Complete Alt-GNN, Top-K Alt-GNN, and Attention Alt-GNN.} \revII{
The Alt-GNN framework provides a deterministic utility
representation through alternative-level message passing,
paralleling the Mother Logit formulation and enabling
flexible utility specifications through its hyperparameter
space (e.g., alternative graphs and aggregation functions).} Second, empirically, Alt-GNNs significantly improve the predictive performance of benchmark models in terms of all the evaluation metrics (e.g., accuracy and log-likelihood). This improvement stems primarily from the flexible hyperparameter space of Alt-GNNs, particularly the alternative graph design transcending a pre-specified nesting structure. \revII{Third, Nested Alt-GNN provides a behavior-informed
graph design that incorporates the nested structure
of NL while allowing more flexible utility
specifications. The empirical results show that
Nested Alt-GNN produces broadly reasonable
substitution patterns, as reflected in the
estimated cross-elasticities. From a theoretical perspective,
Nested Alt-GNN preserves the
probability-ratio properties of NL, although these
properties do not guarantee the same full substitution
structure as NL.} Fourth, Alt-GNNs provide graph-based interpretations with four non-traditional properties: they automatically learn the alternative graph from data, allow directed and asymmetric dependencies across alternatives, move beyond nested structures, and capture heterogeneous dependence patterns across individuals. Overall, Alt-GNNs demonstrate how GNNs and GEV models can be connected by converting pre-specified nesting structures to flexible alternative graphs, improving predictive performance, allowing behavioral priors to be encoded through alternative graph design, and enabling novel graph-based utility interpretations.

Despite these contributions, several major limitations warrant acknowledgment. \rev{\revII{
While this study establishes a connection between
classical discrete choice models and Alt-GNNs
at the level of deterministic utility representation,
it does not seek to establish a general correspondence
between graph-based message passing and GEV generating
functions or random error dependence structures.
Our theoretical contribution therefore focuses on
deterministic utility interpretation, rather than
establishing an equivalence between Alt-GNNs and
GEV models at the level of their random utility foundations.
} For example, the correspondence between nesting structures and nested alternative graphs is clear in \cref{fig:nesting_structures}, but it remains unclear how to design alternative graphs and Alt-GNN architectures that correspond to cross-nested logit, generalized nested logit, paired combinatorial logit, and other GEV models.} \rev{Secondly, Alt-GNN specifications exhibit a prediction--interpretability trade-off. More flexible graph structures, such as Complete Alt-GNN, Top-K Alt-GNN, and Attention Alt-GNN, can achieve stronger predictive performance, but their learned interactions are harder to interpret through standard behavioral structures. In contrast, Nested Alt-GNN retains a clearer nested substitution structure, although it is not necessarily the best-performing specification. \revII{Thirdly, Alt-GNN does not guarantee key economic properties such as monotonicity or well-defined willingness-to-pay. 
Although Nested Alt-GNN preserves the probability-ratio properties of NL, these properties
do not guarantee the exactly same substitution patterns
as NL models in choice probabilities,
elasticities, or the VOTT estimates.
} Flexible alternative graphs could also lead to
non-monotonic probability functions,
\revII{behaviorally inconsistent elasticities},
or extreme probability ratios.
These behavioral irregularities have also been
reported in DNN-based choice models and may be
addressed by incorporating additional behavioral
restrictions
\citep{fengDeepNeuralNetworks2024,
kimNewFlexiblePartially2024}. Fourthly, the predictive gains of Alt-GNN over the benchmark models are moderate. It is potentially caused by the small choice set of only four travel modes, which may limit its predictive value. Alt-GNNs can be examined for larger choice sets, potentially achieving higher performance. Lastly, extending this evaluation to larger choice sets or social groups also leads to scalability challenges. Larger alternative sets may substantially increase the complexity of graph construction, message passing, and model estimation. Our work considers only alternative graphs without explicitly constructing a social network, which could further increase model complexity in future studies.}

Nonetheless, the Alt-GNN framework opens numerous research opportunities, as individual choices in urban transportation systems nearly always involve a network. While classical GEV models are limited to pre-defined nesting structures, the alternative graph concept can adapt to highly diverse spatial and alternative networks. For example, Alt-GNNs can be applied to residential and work location decisions using spatial alternative graphs, to route choice by treating routes as chain graphs, or to hierarchical spatial decisions using multiscale alternative graphs. \rev{Future work can further examine Alt-GNNs with larger and more complex alternative graphs, where scalability, sparse graph design, and efficient message passing become increasingly important. Another extension is to combine alternative-level graphs with individual-level social networks, allowing the model to represent both dependence among alternatives and interactions among decision makers.} Through its capability to learn individual-specific alternative graphs, the Alt-GNN framework also provides a way to examine heterogeneous decision mechanisms. Alt-GNNs can readily be integrated with existing synergistic approaches: for instance, enhancing behavioral regularity and interpretability by incorporating behavioral constraints. Introducing economic constraints and statistical inference tools can help ensure correct elasticity signs, behaviorally consistent substitution patterns, and reliable effect estimates in Alt-GNNs. Alt-GNNs can also be combined with causal inference approaches, such as Double Machine Learning, to identify causal effects on spatial or alternative graphs. \revII{
Finally, future research may examine how graph-based
deterministic utility specifications can be combined
with GEV generating functions and different random
error dependence structures.
} Overall, this work opens broad opportunities to advance the methodological frontier, investigate diverse applications, and strengthen the theoretical foundation for the synergy of DNNs and DCMs.

\section*{Contributions of the authors}
Yuqi Zhou: Writing -- review \& editing, Writing -- original draft, Visualization, Validation, Methodology, Investigation, Formal analysis, Data curation. Zhanhong Cheng: Writing -- review \& editing, Writing -- original draft, Validation, Methodology, Investigation. Dingyi Zhuang: Writing -- review \& editing, Conceptualization. Lingqian Hu: Writing -- review \& editing, Supervision, Project administration. Yuheng Bu: Writing -- review \& editing. Shenhao Wang: Writing -- review \& editing, Writing -- original draft, Validation, Supervision, Project administration, Methodology, Investigation, Conceptualization.

\section*{Acknowledgment}
The authors acknowledge the support from the Research Opportunity Seed Fund 2023 at the University of Florida and the U.S. Department of Energy’s Office of Energy Efficiency and Renewable Energy (EERE) under the Vehicle Technology Program Award Number DE-EE0011186. The views expressed herein do not necessarily represent the views of the U.S. Department of Energy or the United States Government. The authors acknowledge the early discussions with Dr. Kara Kockelman, Dr. Joan Walker, and Dr. Jinhua Zhao in the research seminars at UT Austin, UC Berkeley, and MIT. The authors also acknowledge the support from Dr. Yao Rui and Siqi Feng at the early stage of this research. 

\bibliographystyle{elsarticle-harv}
\bibliography{Paper_GNN_Mode_Choice_Yuqi_260331}


\begin{appendices}
\section{Choice Probabilities in Nested Logit Models}
\label{sec:appendix_a}
\begin{proposition}\label{prop:nl_gnn}
The classical formula of the NL model given by \cref{eq:nest_probability_1} is equivalent to \cref{eq:nest_probability_2}--a special type of GNN.
\end{proposition}

\begin{proof}
The NL model in \cref{eq:nest_probability_1} can be reorganized as
\begin{align*}
P_{ni} &= P(i|B_k) P(B_k|\mathcal{V})\\
    &=\frac{\exp \left(V_{n i}/\mu_k\right)}{\sum_{j \in B_k} \exp \left(V_{n j}/\mu_k\right)} \times \frac{\left(\sum_{j \in B_k} \exp \left(V_{n j}/\mu_k\right)\right)^{\mu_k}}{\sum_{l=1}^{K_n}\left(\sum_{j \in B_l} \exp \left(V_{n j}/\mu_l\right)\right)^{\mu_l}}\\
    &= \frac{\exp \left(V_{n i}/\mu_k\right)\left(\sum_{j \in B_k} \exp \left(V_{n j}/\mu_k\right)\right)^{\mu_k -1}}
    {\sum_{l=1}^{K_n}
    \left(\sum_{j \in B_l} \exp \left(V_{n j}/\mu_l\right)\right)^{\mu_l}
    }\\
    &= \frac{\exp \left(V_{n i}/\mu_k\right)\left(\sum_{j \in B_k} \exp \left(V_{n j}/\mu_k\right)\right)^{\mu_k -1}}
    {\sum_{l=1}^{K_n}
    \left(
    \frac{\sum_{m \in B_l} \exp \left(V_{n m}/\mu_l\right)}{\sum_{m \in B_l} \exp \left(V_{n m}/\mu_l\right)}
    \left(\sum_{j \in B_l} \exp \left(V_{n j}/\mu_l\right)\right)^{\mu_l}
    \right)} \\
    &= \frac{\exp \left(V_{n i}/\mu_k\right)\left(\sum_{j \in B_k} \exp \left(V_{n j}/\mu_k\right)\right)^{\mu_k -1}}
    {\sum_{l=1}^{K_n} \sum_{m \in B_l}
    \left(
    \frac{ \exp \left(V_{n m} / \mu_l\right)}{\sum_{m \in B_l} \exp \left(V_{n m} / \mu_l\right)}
    \left(\sum_{j \in B_l} \exp \left(V_{n j}/\mu_l\right)\right)^{\mu_l}
    \right)} \\
    & = \frac{\exp \left(V_{n i}/\mu_k\right)\left(\sum_{j \in B_k} \exp \left(V_{n j}/\mu_k\right)\right)^{\mu_k -1}}
    {\sum_{l=1}^{K_n} \sum_{m \in B_l}
    \left(
    \exp \left(V_{n m} / \mu_l\right)
    \left(\sum_{j \in B_l} \exp \left(V_{n j}/\mu_l\right)\right)^{\mu_l - 1}
    \right)}\\
    & = \frac{\exp \left(V_{n i}/\mu_k + (\mu_k -1)\log\left( \sum_{j \in B_k} \exp \left(V_{n j}/\mu_k\right)\right)
    \right)}
    {\sum_{m \in \mathcal{V}}
    \exp \left( V_{n m} / \mu_l + (\mu_l - 1)
    \log (\textstyle\sum_{j \in B_l} \exp \left(V_{n j}/\mu_l \right) )\right)}.
\end{align*}

\end{proof}

This yields to formula in \cref{eq:nest_probability_2}.

\section{Substitution patterns in nested logit model}
\label{sec:appendix_substitution_nl}
The NL model enhances the flexibility in the substitution pattern by relaxing the restrictive IIA constraint in MNL \citep{mcfaddenModelingChoiceResidential1978}. Its choice probability function is specified as:
\begin{equation}
P_{ni} = P(i|B_k) P(B_k|\mathcal{V}) = \frac{e^{V_{ni}/\mu_k}}{\sum_{j\in B_k} e^{V_{nj}/\mu_k}} \times \frac{(\sum_{j\in B_k} e^{V_{nj}/\mu_k})^{\mu_k}}{\sum_{l = 1}^K(\sum_{j\in B_l} e^{V_nj/\mu_l})^{\mu_l}},
\label{eq:nest_probability_1}
\end{equation}
where $B_k$ and $B_l$ represent the alternative sets in nests $k$ and $l$, and $\mu_k$ and $\mu_l$ represent their corresponding scale factors \citep{mcfaddenConditionalLogitAnalysis1974}. Although  \cref{eq:nest_probability_1} is the most common NL form, its terms can be reorganized as \cref{eq:nest_probability_2}. The equivalence between \cref{eq:nest_probability_1} and \cref{eq:nest_probability_2} has been formally established in our ongoing work \citep{chengGraphNeuralNetworks2025} and is detailed in Appendix~\ref{sec:appendix_a}.
Here we would like to reinterpret NL by highlighting two interrelated properties as shown in \cref{eq:nest_probability_2}. 
\begin{equation}
P_{ni} = \frac{\exp \left( V_{ni}/\mu_k + (\mu_k - 1) \log ( \sum_{j \in B_k} \exp (V_{nj}/\mu_k) ) \right)}{\sum_{m \in \mathcal{V}} \exp \left( V_{nm}/\mu_l + (\mu_l - 1) \log (\sum_{j \in B_l} \exp (V_{nj}/\mu_l) ) \right)}.
\label{eq:nest_probability_2}
\end{equation}

The NL model is characterized by its additive nest utility with log-sum-exponential (LSE) aggregation. This is a straightforward description of the numerator in \cref{eq:nest_probability_2}, as its first term $V_{ni}/\mu_k$ represents the alternative $i$'s self-utility and the second term $(\mu_k - 1) \log ( \sum_{j \in B_k} \exp (V_{nj}/\mu_k) )$ is added to the first term by aggregating all the utilities in nest $B_k$ using a log-sum-exponential (LSE) form with $\mu_k$ for linear scaling. This property will be leveraged to design our Nested Alt-GNN models in \cref{sec:methodology}. 


\smallskip 
\begin{property_nl}
\textbf{Two-Layer Substitution.} \textit{The NL model is characterized by its two-layer substitution patterns: proportional substitution of alternatives $i$ and $j$ within every nest $B_k$ when $i, j \in B_k$, and disproportional substitution of alternatives $i$ and $j$ across nests $B_k$ and $B_l$ when $i \in B_k$ and $j \in B_l$}. 
\end{property_nl}

\noindent Mathematically, the two-layer substitution property can be shown straightforwardly. When $i, j \in B_k$, the ratio of choice probabilities 
\begin{equation}
\frac{P_{ni}}{P_{nj}} = \frac{e^{V_{ni}/\mu_k}}{e^{V_{nj}/\mu_k}},
\label{eq:sub_1}
\end{equation}
which does not depend on any alternatives other than $i$ and $j$. When $i \in B_k$ and $j \in B_l$, the ratio of choice probabilities equals: 
\begin{equation}
\frac{P_{ni}}{P_{nj}} = \frac{\exp \left( V_{ni}/\mu_k + (\mu_k -1) \log ( \sum_{m \in B_k} e^{V_{nm}/\mu_k} ) \right)} {\exp \left( V_{nj}/\mu_l + (\mu_l -1) \log ( \sum_{m \in B_l} e^{V_{nm}/\mu_l} ) \right) }.
\label{eq:sub_2}
\end{equation}

This two-layer substitution property characterizes the uniqueness of NL because it captures how the NL model retains the proportional substitution pattern of MNL within each nest and yet relaxes the IIA restriction across nests. \cref{eq:sub_2} relaxes the IIA constraint only partially because the choice probability ratio depends on all the alternatives within nests $B_k$ and $B_l$, but is still independent of the alternatives beyond the two nests. Comparing the two properties, Property 2 can be seen as a consequence of Property 1 because the additive LSE aggregation cancels out in \cref{eq:sub_1} and remains in \cref{eq:sub_2}.

\section{Nested Alt-GNN with high dimension and multiple layers}
\label{sec:appendix_nested_alt_gnn_high_dimension}

\rev{Nested Alt-GNN with high-dimensional MLP message passing and multi-layer aggregation can be specified as an example of Alt-GNN when its hyperparameter space $\mathcal H$ takes the following form:}
\rev{$L\geq1,
M_{i,j}(x_j^{(0)})=\mathrm{MLP}_{j}^{(M)}(x_j^{(0)}),
A^{(l)}(m_i^{(l)})=\mathrm{AGGR}_{j\in\mathcal N(i)}
\left(x_j^{(l-1)}\right),
U^{(l)}(x_i^{(l-1)},a_i^{(l)}) =\mathrm{CONCAT}\left(x_i^{(l-1)},a_i^{(l)}\right),
R_i(x_i^{(L)})
=\mathrm{MLP}_{i}^{(\mathrm{self})}
\left(M_i(x_i^{(0)})\right)
+\mathrm{MLP}_{\mathcal N(i)}^{(\mathrm{context})}
\left(a_i^{(L)}\right).$}

\rev{Here, $\mathrm{MLP}_{j}^{(M)}$ denotes the message MLP and can
produce a high-dimensional representation. The operator
$\mathrm{AGGR}$ can be mean, max, sum, or element-wise LSE, and
$a_i^{(l)}$ denotes the resulting high-dimensional aggregated
representation at layer $l$. The readout contains an
alternative-specific MLP and a context-specific MLP. Both readout
MLPs may have high-dimensional hidden layers, but their final outputs
are scalar. The parameters and scalar output of
$\mathrm{MLP}_{\mathcal N(i)}^{(\mathrm{context})}$ are shared by all
alternatives in the same neighborhood.}

\rev{For alternatives $i$ and $j$ in the same neighborhood,
$a_i^{(L)}=a_j^{(L)}$ and the same context-specific readout is used.
Their probability ratio is therefore}
\begin{equation}
\rev{\frac{P_{ni}}{P_{nj}}
=
\frac{
\exp\left[
\mathrm{MLP}_{i}^{(\mathrm{self})}
\left(M_i(x_{ni}^{(0)})\right)
\right]
}{
\exp\left[
\mathrm{MLP}_{j}^{(\mathrm{self})}
\left(M_j(x_{nj}^{(0)})\right)
\right]}
}.
\label{eq:nested_alt_gnn_high_dimension_ratio}
\end{equation}

\rev{Although a change in another alternative may affect every dimension of the high-dimensional aggregated representation $\mathbf a_{ni}^{(L)}$, the contextual MLP transforms this representation into a scalar shared by alternatives $i$ and $j$ in the same neighborhood. This shared scalar cancels from their utility difference and probability ratio. In contrast, alternatives in different neighborhoods have different contextual representations and contextual readout functions, allowing their substitution responses to differ. Therefore, the dimensionality of the intermediate message and aggregated representations does not affect the within-subnetwork probability-ratio result, and the high-dimensional, multi-layer Nested Alt-GNN retains the stated substitution property under the additively separable readout restriction.}

\section{Value of Time Analysis}
\label{app:behavioral_analysis}

\rev{Table~\ref{tab:vott_results} reports individual-level VOTT estimates on the independent test set. VOTT is summarized only for observations with negative marginal utilities of both travel time and
travel cost and a finite, non-zero cost derivative.} \rev{Overall, the median VOTT estimates of most models are of a similar order of magnitude to those of MNL and NL, particularly for transit. This suggests that the central tendency of the estimated VOTT remains broadly comparable to conventional discrete choice benchmarks. However, several neural-network-based models exhibit large standard deviations and substantial differences between the mean and median, indicating the presence of extreme VOTT estimates. In addition, the proportion of observations satisfying the validity conditions varies considerably across models, especially for automobile. These results provide further evidence of the prediction--interpretability trade-off discussed in the main text: more flexible specifications can achieve stronger predictive performance, but may also produce less stable marginal utilities and economic measures such as VOTT.}

\begin{table*}[htbp]
\centering
\caption{\rev{Individual-level value of travel time (VOTT; GBP/hour) on the
independent test set. Statistics are calculated over valid observations with
negative marginal utilities of both travel time and travel cost and a finite,
non-zero cost derivative.}}

\label{tab:vott_results}
\small
\setlength{\tabcolsep}{5pt}
\begin{tabular}{lrrrrrr}
\toprule
& \multicolumn{3}{c}{Automobile}
& \multicolumn{3}{c}{Transit} \\
\cmidrule(lr){2-4}
\cmidrule(lr){5-7}
Model
& Mean (SD)
& Median
& Valid (\%)
& Mean (SD)
& Median
& Valid (\%) \\
\midrule
MNL
& 39.68 (0.00)
& 39.68
& 100.0
& 12.09 (0.00)
& 12.09
& 100.0 \\

NL
& 39.90 (0.00)
& 39.90
& 100.0
& 18.27 (0.00)
& 18.27
& 100.0 \\

ResLogit
& 188.08 (6853.46)
& 43.22
& 93.4
& 15.10 (8.97)
& 13.99
& 100.0 \\

TasteNet-MNL
& 58.59 (34.15)
& 47.21
& 100.0
& 16.13 (10.56)
& 13.93
& 100.0 \\

ASU-DNN
& 97.88 (503.76)
& 28.56
& 13.8
& 34.91 (235.68)
& 9.73
& 78.8 \\

Nested Alt-GNN
& 138.78 (1799.71)
& 25.57
& 14.6
& 83.72 (2211.20)
& 14.14
& 87.0 \\

Complete Alt-GNN
& 293.00 (417.28)
& 86.96
& 24.2
& 27.20 (111.70)
& 10.82
& 83.2 \\

Top K Alt-GNN
& 115.83 (878.91)
& 33.08
& 15.9
& 80.61 (1446.88)
& 13.87
& 82.7 \\

Attention Alt-GNN
& 1065.39 (17742.54)
& 121.05
& 34.3
& 38.03 (804.08)
& 15.31
& 82.0 \\
\bottomrule
\end{tabular}
\end{table*}

\section{Experiment results for the CMAP Dataset}
\label{sec:chicago_results}
The empirical findings from the CMAP dataset are broadly consistent with those
reported for the LPMC dataset in Section~\ref{sec:results}. Across both
five-fold cross-validation and the independent test set, the best-performing
Alt-GNN specifications outperform all benchmark DCM and DNN models in
log-likelihood, accuracy, and macro F1 score
(\cref{tab:best_model_comparison_chicago}). On the independent test set,
the Complete Alt-GNN achieves the highest log-likelihood, the Top K Alt-GNN achieves the highest accuracy, and the Attention Alt-GNN achieves the highest
macro F1 and balanced accuracy. Moreover, all five highest-ranked model
specifications are Alt-GNN variants
(\cref{tab:top5_models_chicago}). \rev{Table~\ref{tab:class_performance_comparison_chicago} shows that Alt-GNN variants achieve the highest automobile F1 and transit recall and F1, with Attention Alt-GNN also attaining the highest walking recall; however, like the benchmark models, they fail to correctly identify bicycle trips, highlighting a remaining limitation under severe class imbalance.} The Nested Alt-GNN continues to produce elasticity and substitution patterns closely aligned with the NL model, confirming that the alternative graph design reliably encodes behavioral constraints across datasets. The MNL and ASU-DNN models follow the IIA constraint as expected, while the Complete and Attention Alt-GNNs exhibit more flexible substitution structures (Tables~\ref{tab:elasticity_asu_mnl_chicago}--\ref{tab:elasticity_attention_complete_altgnn_cmap}). The attention-weight distributions and clustering results similarly reveal
heterogeneous yet structured interaction patterns across individuals. Bike
serves as the most heavily weighted attention source in the shorter and more
multimodal cluster, whereas transit becomes the dominant attention source
across alternatives in the longer and automobile-dominated cluster
(\cref{fig:cmap_attention_clusters}).

\begin{table*}[htbp]
\centering
\footnotesize
\renewcommand{\arraystretch}{1.12}
\setlength{\tabcolsep}{3.5pt}
\caption{Cross-validation and independent test-set performance of the
selected models from the nine model families on the CMAP dataset.}
\label{tab:best_model_comparison_chicago}

\resizebox{\textwidth}{!}{%
\begin{tabular}{lccccc|ccccc}
\hline
& \multicolumn{5}{c|}{\textbf{Five-fold cross validation set (mean value)}}
& \multicolumn{5}{c}{\textbf{Independent test set}} \\
\cline{2-11}

\textbf{Model}
& \textbf{LL/obs. $\uparrow$}
& \textbf{Macro F1 $\uparrow$}
& \textbf{ACC $\uparrow$}
& \textbf{BACC $\uparrow$}
& \textbf{BS $\downarrow$}
& \textbf{LL/obs. $\uparrow$}
& \textbf{Macro F1 $\uparrow$}
& \textbf{ACC $\uparrow$}
& \textbf{BACC $\uparrow$}
& \textbf{BS $\downarrow$} \\
\hline

Attention Alt-GNN
& -0.6703
& \textbf{0.4749}
& 0.7420
& \textbf{0.4646}
& 0.3598
& -0.6527
& \textbf{0.5126}
& 0.7651
& \textbf{0.5085}
& \textbf{0.3402} \\

Complete Alt-GNN
& \textbf{-0.6671}
& 0.4728
& \textbf{0.7445}
& 0.4598
& \textbf{0.3583}
& \textbf{-0.6459}
& 0.5061
& 0.7634
& 0.4962
& 0.3413 \\

Top-K Alt-GNN
& -0.6674
& 0.4736
& 0.7429
& 0.4604
& 0.3595
& -0.6495
& 0.5079
& \textbf{0.7665}
& 0.4965
& 0.3408 \\

Nested Alt-GNN
& -0.6824
& 0.4489
& 0.7346
& 0.4365
& 0.3690
& -0.6672
& 0.4872
& 0.7562
& 0.4767
& 0.3508 \\

ASU-DNN
& -0.6880
& 0.4486
& 0.7363
& 0.4360
& 0.3699
& -0.6688
& 0.4893
& 0.7562
& 0.4796
& 0.3513 \\

ResLogit
& -0.6923
& 0.4559
& 0.7289
& 0.4393
& 0.3760
& -0.6702
& 0.4901
& 0.7451
& 0.4835
& 0.3607 \\

TasteNet-MNL
& -0.7629
& 0.3926
& 0.7089
& 0.3785
& 0.4018
& -0.7402
& 0.4178
& 0.7091
& 0.4028
& 0.3923 \\

MNL
& -0.7625
& 0.3862
& 0.7051
& 0.3745
& 0.4042
& -0.7491
& 0.3773
& 0.6966
& 0.3694
& 0.3976 \\

NL
& -0.7532
& 0.3968
& 0.7068
& 0.3807
& 0.4028
& -0.7454
& 0.4005
& 0.7040
& 0.3843
& 0.3958 \\

\hline
\end{tabular}%
}
\end{table*}

\begin{table*}[htbp]
\centering
\footnotesize
\renewcommand{\arraystretch}{1.12}
\setlength{\tabcolsep}{3.5pt}
\caption{\rev{Class-specific recall and F1 scores of the selected models
from the nine model families on the CMAP independent test set.}}
\label{tab:class_performance_comparison_chicago}

\begin{tabular}{lcc|cc|cc|cc}
\hline
& \multicolumn{2}{c|}{\textbf{Automobile}}
& \multicolumn{2}{c|}{\textbf{Transit}}
& \multicolumn{2}{c|}{\textbf{Bike}}
& \multicolumn{2}{c}{\textbf{Walking}} \\
\cline{2-9}

\textbf{Model}
& \textbf{Recall $\uparrow$}
& \textbf{F1 $\uparrow$}
& \textbf{Recall $\uparrow$}
& \textbf{F1 $\uparrow$}
& \textbf{Recall $\uparrow$}
& \textbf{F1 $\uparrow$}
& \textbf{Recall $\uparrow$}
& \textbf{F1 $\uparrow$} \\
\hline

Attention Alt-GNN
& 0.8714 & 0.8471
& \textbf{0.4901} & \textbf{0.5460}
& 0.0000 & 0.0000
& \textbf{0.6724} & 0.6574 \\

Complete Alt-GNN
& 0.8836 & 0.8463
& 0.4524 & 0.5316
& 0.0000 & 0.0000
& 0.6489 & 0.6466 \\

Top-K Alt-GNN
& 0.8884 & \textbf{0.8475}
& 0.4399 & 0.5274
& 0.0000 & 0.0000
& 0.6576 & 0.6568 \\

Nested Alt-GNN
& 0.8880 & 0.8418
& 0.3537 & 0.4498
& 0.0000 & 0.0000
& 0.6650 & 0.6573 \\

ASU-DNN
& 0.8847 & 0.8416
& 0.3662 & 0.4564
& 0.0000 & 0.0000
& 0.6675 & \textbf{0.6593} \\

ResLogit
& 0.8622 & 0.8314
& 0.4291 & 0.5010
& 0.0000 & 0.0000
& 0.6428 & 0.6280 \\

TasteNet-MNL
& 0.8954 & 0.8165
& 0.2388 & 0.3459
& 0.0000 & 0.0000
& 0.4771 & 0.5086 \\

MNL
& 0.9108 & 0.8106
& 0.1329 & 0.2161
& 0.0000 & 0.0000
& 0.4339 & 0.4825 \\

NL
& \textbf{0.9112} & 0.8134
& 0.2047 & 0.3056
& 0.0000 & 0.0000
& 0.4215 & 0.4830 \\

\hline
\end{tabular}%

\medskip
\begin{minipage}{\textwidth}
\scriptsize
\textit{Notes:}
The independent test set contains 4,159 observations:
2,714 automobile trips (65.26\%),
557 transit trips (13.39\%),
79 bicycle trips (1.90\%), and
809 walking trips (19.45\%).
Metrics are computed using the highest-probability class prediction.
Bold values indicate the best performance in each column,
except for the bicycle metrics, for which all models score zero.
\end{minipage}
\end{table*}
\begin{table}[htbp]
\small
\centering
\caption{Top five selected model specifications ranked by log-likelihood per
observation on the CMAP test set.}
\label{tab:top5_models_chicago}

\resizebox{\linewidth}{!}{%
\begin{tabular}{p{3.5cm}p{2.5cm}p{2.5cm}p{2.5cm}p{2.5cm}p{2.5cm}}
\hline
\textbf{Model rank}
& \textbf{1st}
& \textbf{2nd}
& \textbf{3rd}
& \textbf{4th}
& \textbf{5th} \\
\hline

\multicolumn{6}{l}{\textbf{Model specification}} \\
\hline

Model family
& Complete Alt-GNN
& Top-K Alt
& Attention Alt-GNN
& Top-K Alt-GNN
& Nested Alt-GNN \\

Graph structure
& Complete graph
& Learned binary graph ($k=3$)
& Attention graph
& Learned binary graph ($k=2$)
& Nested graph ($0001$) \\

Message function
& MLP
& MLP
& Linear
& Linear
& MLP \\

Aggregation function
& Max
& Max
& LSE
& Max
& Mean \\

Update function
& Concat
& Concat
& Concat
& Concat
& Concat \\

Readout function
& Linear
& Linear
& MLP
& MLP
& Linear \\

Message-passing layers
& 1
& 2
& 1
& 2
& 2 \\

Hidden units
& 64
& 32
& 32
& 32
& 32 \\
\hline

\multicolumn{6}{l}{\textbf{Independent test-set performance}} \\
\hline

LL/observation
& \textbf{-0.645946}
& -0.649483
& -0.652665
& -0.654134
& -0.667249 \\

Macro F1
& 0.506125
& 0.507923
& \textbf{0.512616}
& 0.507144
& 0.487211 \\

Accuracy
& 0.763405
& \textbf{0.766530}
& 0.765088
& 0.761241
& 0.756191 \\

Balanced accuracy
& 0.496235
& 0.496454
& \textbf{0.508492}
& 0.501850
& 0.476672 \\

Brier score
& 0.341324
& 0.340781
& \textbf{0.340198}
& 0.342110
& 0.350842 \\
\hline
\end{tabular}%
}
\end{table}

\begin{table}[htbp]
\centering
\caption{Mean individual-level point elasticities of MNL and ASU-DNN on the
CMAP independent test set. Standard deviations
across test observations are reported in parentheses.}
\resizebox{\textwidth}{!}{%
\begin{tabular}{lcccc|cccc}
\toprule
& \multicolumn{4}{c|}{\textbf{MNL}}
& \multicolumn{4}{c}{\textbf{ASU-DNN}} \\
& Automobile & Transit & Bike & Walking
& Automobile & Transit & Bike & Walking \\
\midrule

Automobile time
& -0.54 (0.58)
& \textbf{1.25 (1.11)}
& \textbf{1.25 (1.11)}
& \textbf{1.25 (1.11)}
& -0.72 (0.66)
& \textbf{1.62 (1.13)}
& \textbf{1.62 (1.13)}
& \textbf{1.62 (1.13)} \\

Transit time
& \textbf{0.20 (0.32)}
& -1.73 (2.26)
& \textbf{0.20 (0.32)}
& \textbf{0.20 (0.32)}
& \textbf{0.26 (0.45)}
& -1.79 (1.95)
& \textbf{0.26 (0.45)}
& \textbf{0.26 (0.45)} \\

Bike time
& \textbf{0.02 (0.04)}
& \textbf{0.02 (0.04)}
& -1.80 (2.23)
& \textbf{0.02 (0.04)}
& \textbf{0.03 (0.05)}
& \textbf{0.03 (0.05)}
& -2.49 (3.38)
& \textbf{0.03 (0.05)} \\

Walking time
& \textbf{0.28 (0.21)}
& \textbf{0.28 (0.21)}
& \textbf{0.28 (0.21)}
& -3.99 (5.47)
& \textbf{0.51 (0.60)}
& \textbf{0.51 (0.60)}
& \textbf{0.51 (0.60)}
& -2.22 (1.06) \\

\bottomrule
\end{tabular}%
}
\label{tab:elasticity_asu_mnl_chicago}
\end{table}

\begin{figure}[H]
\centering
\includegraphics[width=\textwidth]
{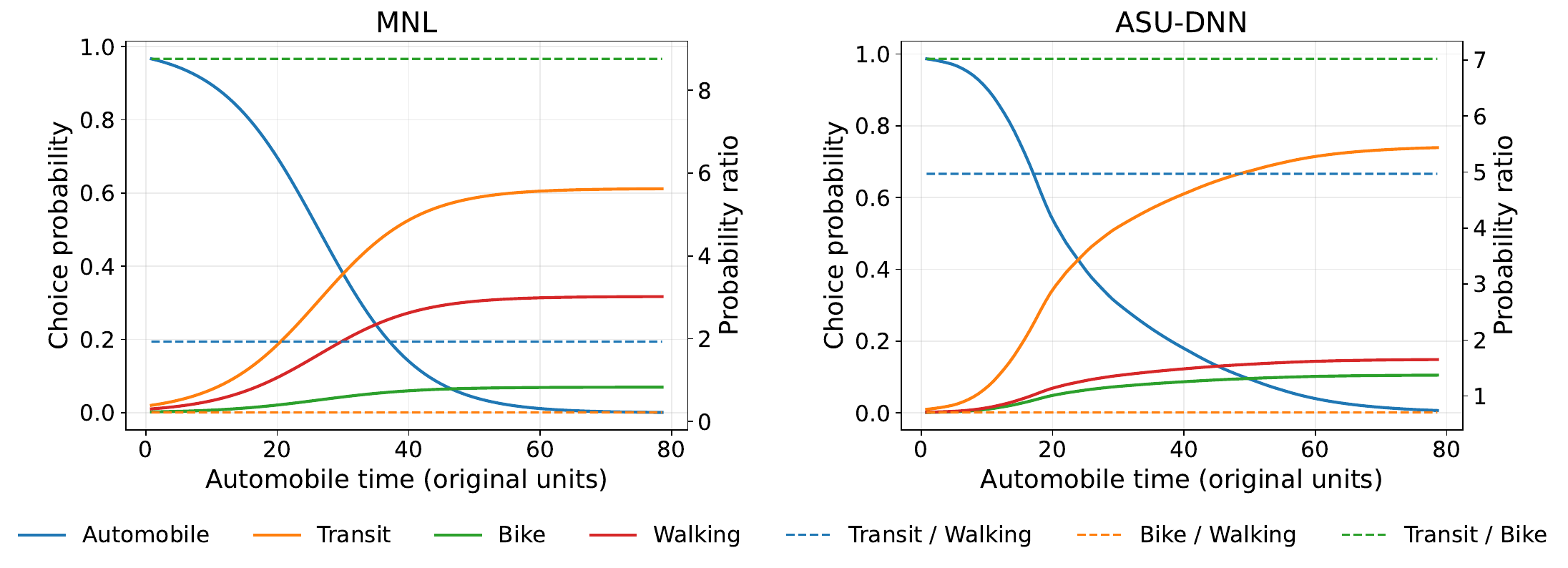}
\caption{Predicted substitution patterns of MNL and ASU-DNN as automobile
travel time varies over its observed CMAP test-set range for a representative
sample-mean profile.}
\label{fig:auto_time_mnl_dnn_chicago}
\end{figure}

\begin{table}[htbp]
\centering
\caption{Mean individual-level point elasticities of NL (0001) and
Nested Alt-GNN (0001, mean) on the CMAP independent test set. Standard deviations across test observations are reported in
parentheses.}
\resizebox{\textwidth}{!}{%
\begin{tabular}{lcccc|cccc}
\toprule
& \multicolumn{4}{c|}{\textbf{NL (0001)}}
& \multicolumn{4}{c}{\textbf{Nested Alt-GNN (0001, mean)}} \\
& Automobile & Transit & Bike & Walking
& Automobile & Transit & Bike & Walking \\
\midrule

Automobile time
& -0.49 (0.82)
& \textbf{1.76 (1.31)}
& \textbf{1.76 (1.31)}
& 0.01 (0.01)
& -0.65 (0.65)
& \textbf{1.72 (1.10)}
& \textbf{1.72 (1.10)}
& 0.94 (0.87) \\

Transit time
& \textbf{0.30 (0.46)}
& -2.26 (3.02)
& \textbf{0.30 (0.46)}
& 0.00 (0.00)
& \textbf{0.31 (0.45)}
& -1.87 (2.43)
& \textbf{0.31 (0.45)}
& -0.38 (0.94) \\

Bike time
& \textbf{0.03 (0.05)}
& \textbf{0.03 (0.05)}
& -2.02 (2.51)
& 0.00 (0.00)
& \textbf{0.01 (0.04)}
& \textbf{0.01 (0.04)}
& -2.39 (2.74)
& 0.30 (0.35) \\

Walking time
& \textbf{0.18 (0.12)}
& \textbf{0.18 (0.12)}
& \textbf{0.18 (0.12)}
& -2.49 (3.42)
& \textbf{0.46 (0.56)}
& \textbf{0.46 (0.56)}
& \textbf{0.46 (0.56)}
& -1.65 (0.89) \\

\bottomrule
\end{tabular}%
}
\label{tab:elasticity_nestgnn_nl_chicago}
\end{table}

\begin{figure}[H]
\centering
\includegraphics[width=\textwidth]
{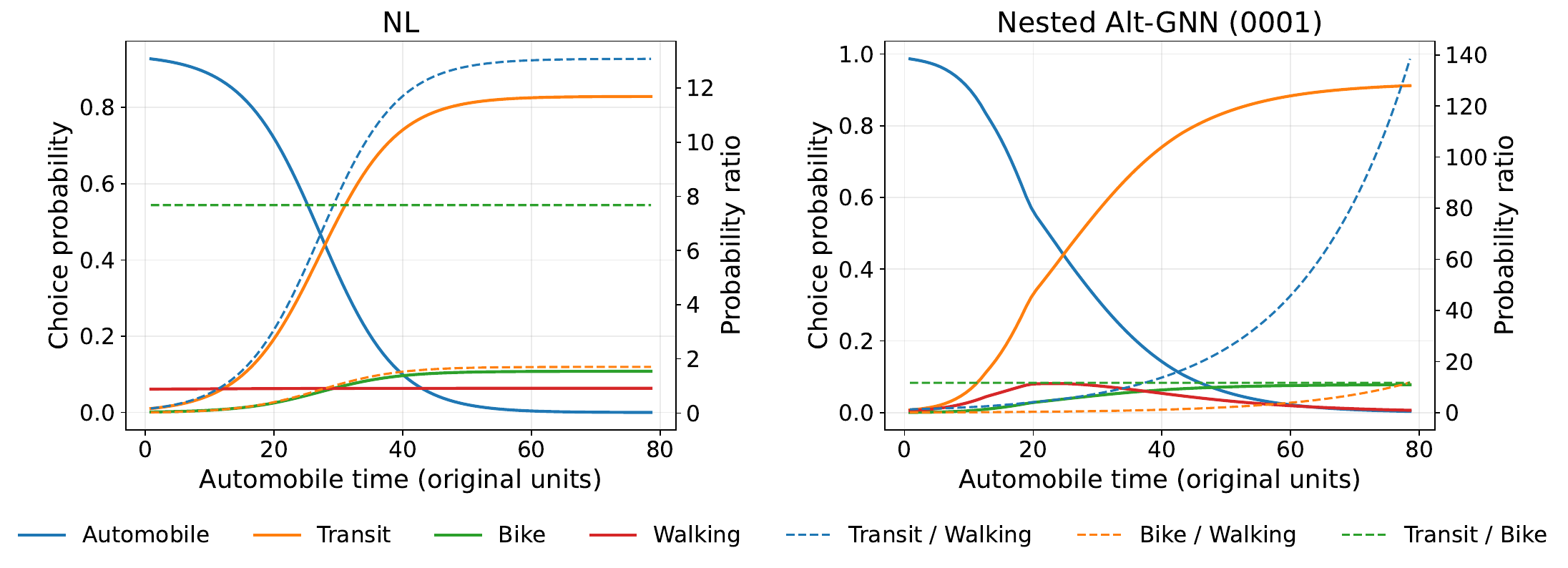}
\caption{Predicted substitution patterns of NL (0001) and
Nested Alt-GNN (0001, mean) as automobile travel time varies over its observed CMAP
test-set range for a representative sample-mean profile on the CMAP
dataset.}
\label{fig:auto_time_nl_nestgnn_chicago}
\end{figure}

\begin{table}[htbp]
\centering
\caption{Mean individual-level point elasticities of Attention and Complete
Alt-GNNs on the CMAP independent test set.
Standard deviations across test observations are reported in parentheses.}
\resizebox{\textwidth}{!}{%
\begin{tabular}{lcccc|cccc}
\toprule
& \multicolumn{4}{c|}{\textbf{Attention Alt-GNN}}
& \multicolumn{4}{c}{\textbf{Complete Alt-GNN}} \\
& Automobile & Transit & Bike & Walking
& Automobile & Transit & Bike & Walking \\
\midrule

Automobile time
& \textbf{-0.40 (0.98)}
& 0.70 (0.89)
& 1.22 (2.54)
& 0.27 (1.37)
& \textbf{-0.33 (0.33)}
& 0.71 (0.69)
& 1.33 (0.72)
& 0.64 (0.73) \\

Transit time
& 0.61 (1.99)
& \textbf{-2.16 (2.05)}
& -0.63 (2.83)
& -0.54 (1.76)
& 0.53 (0.79)
& \textbf{-2.13 (1.53)}
& -0.36 (0.85)
& -0.43 (0.79) \\

Bike time
& -0.61 (1.18)
& 0.83 (0.88)
& \textbf{-2.94 (5.14)}
& 0.71 (0.59)
& -0.34 (0.56)
& 0.95 (0.89)
& \textbf{-1.31 (2.00)}
& 0.75 (0.92) \\

Walking time
& 0.66 (0.95)
& 0.78 (0.90)
& 0.53 (1.10)
& \textbf{-1.57 (1.49)}
& 0.42 (0.55)
& 0.62 (0.62)
& 0.05 (0.89)
& \textbf{-1.77 (0.85)} \\

\bottomrule
\end{tabular}%
}
\label{tab:elasticity_attention_complete_altgnn_cmap}
\end{table}

\begin{figure}[H]
\centering
\includegraphics[width=\textwidth]
{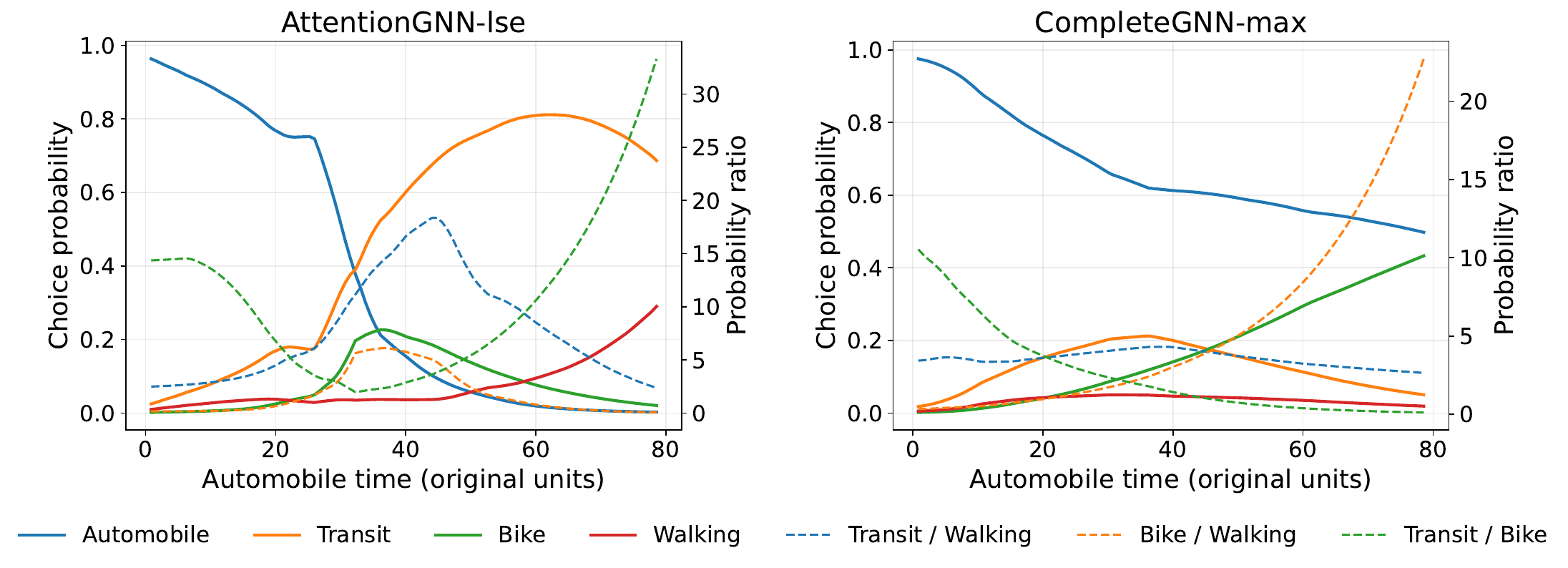}
\caption{Predicted substitution patterns of Attention Alt-GNN (LSE) and Complete
Alt-GNN (Max) as automobile travel time varies over its observed CMAP test-set range for a representative sample-mean profile on the CMAP dataset.}
\label{fig:auto_time_attention_complete_altgnn_chicago}
\end{figure}

\begin{figure}[htbp]
\centering

\begin{subfigure}[t]{0.4\textwidth}
    \centering
    \includegraphics[width=\textwidth]
    {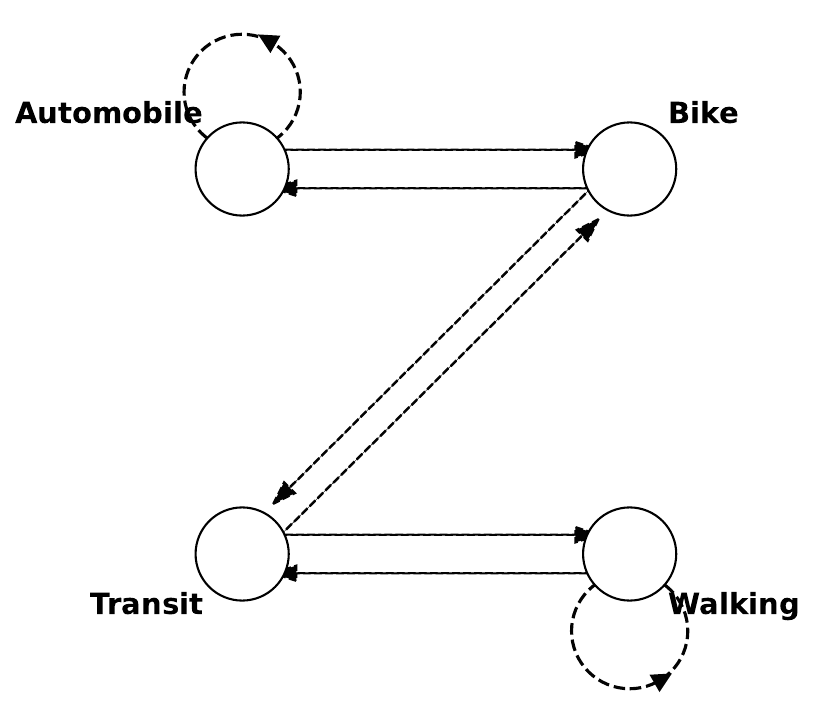}
    \caption{Top-K Alt-GNN ($k=2$, max)}
    \label{fig:discrete_graph_cmap_k2_max}
\end{subfigure}
\hfill
\begin{subfigure}[t]{0.4\textwidth}
    \centering
    \includegraphics[width=\textwidth]
    {figures/Chicago/learned\_binary\_graph\_k3-max.pdf}
    \caption{Top-K Alt-GNN ($k=3$, max)}
    \label{fig:discrete_graph_cmap_k3_max}
\end{subfigure}

\caption{Population-shared directed binary graphs learned by the two
highest-ranked Top-K Alt-GNN specifications on the CMAP dataset.}
\label{fig:discrete_learned_graphs_cmap}
\end{figure}

\begin{figure}[H]
\centering
\includegraphics[width=0.8\textwidth]{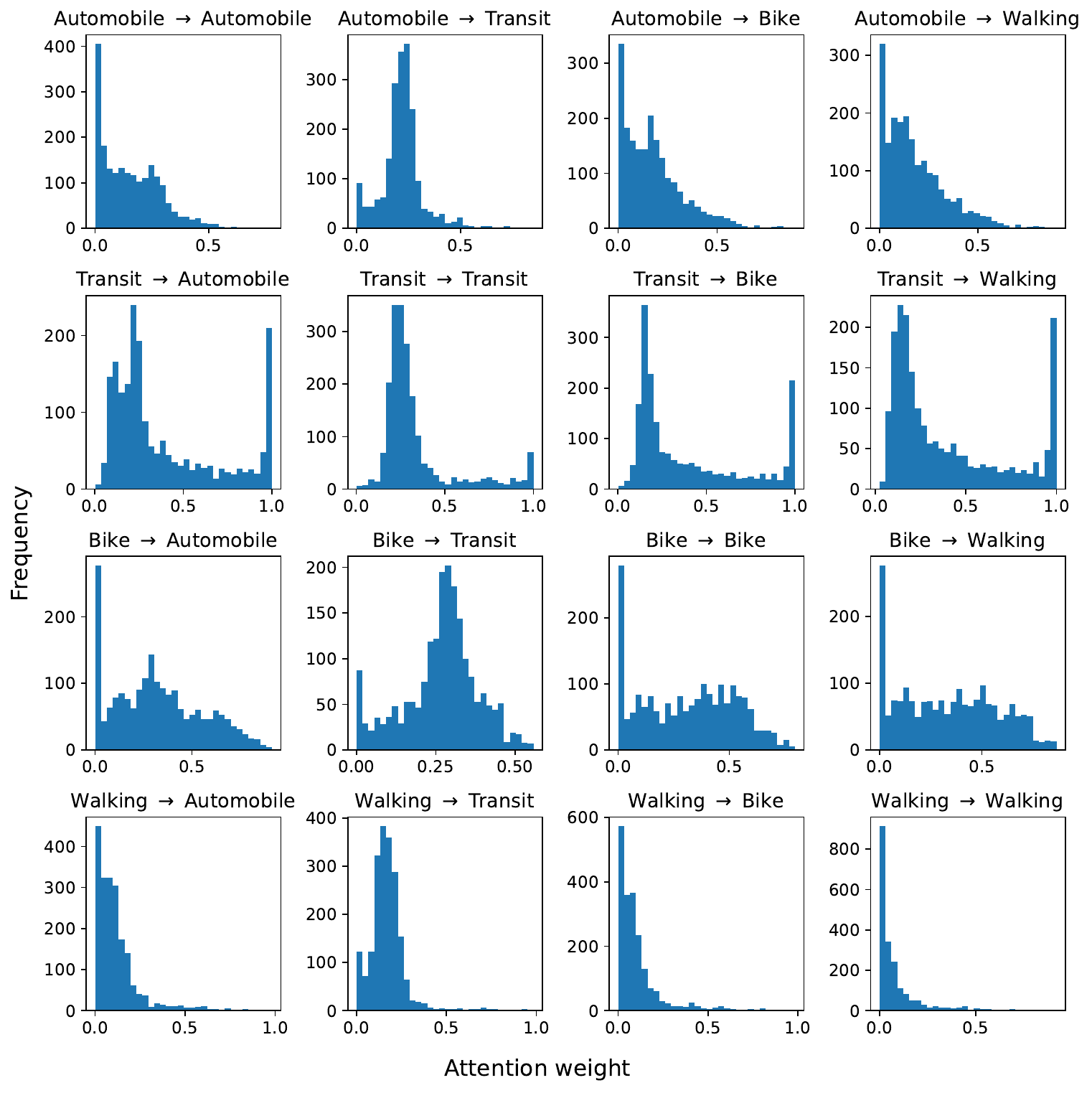}
\caption{Distribution of attention weights learned by the Attention Alt-GNN on the CMAP test set.}
\end{figure}

\begin{figure}[htbp]
    \centering
    \includegraphics[width=0.8\linewidth]{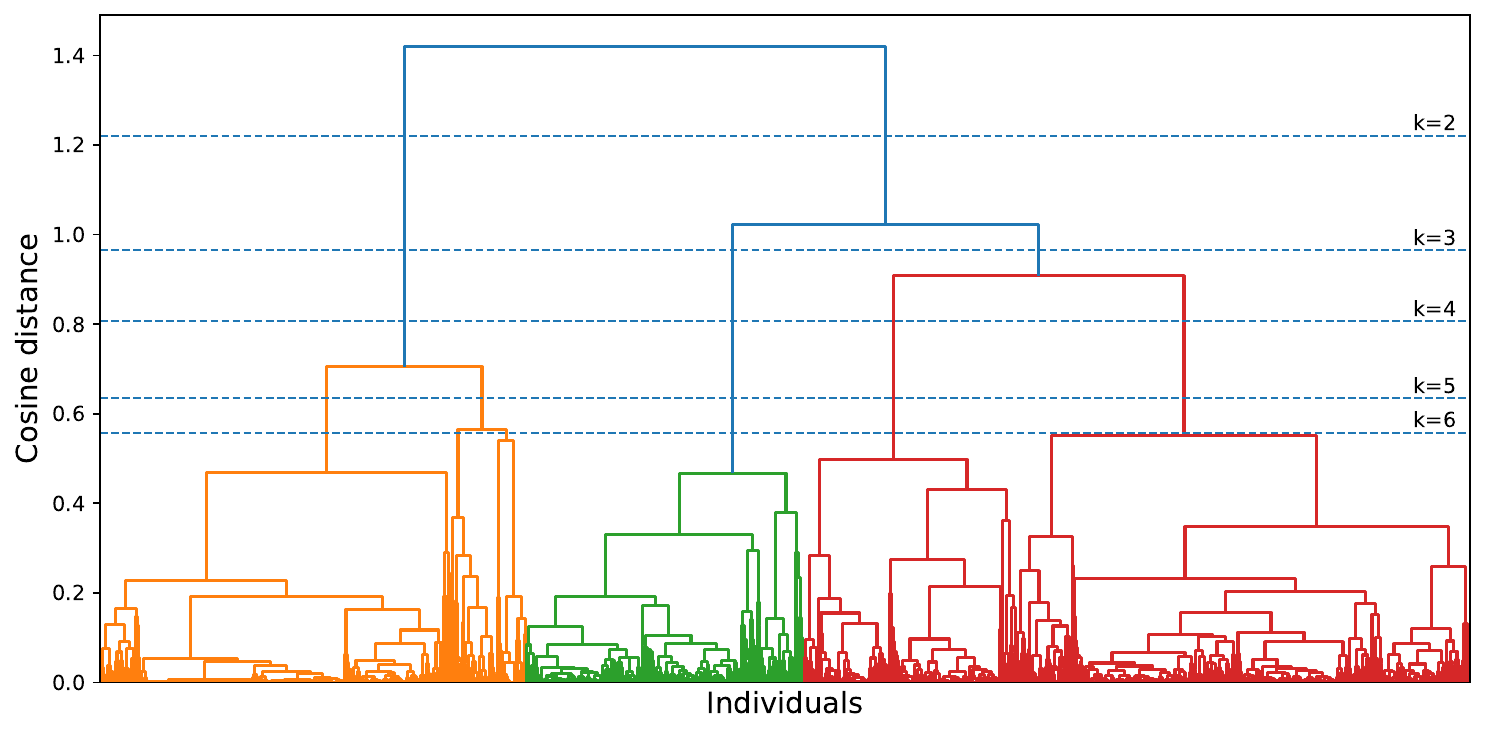}
    \caption{Hierarchical clustering dendrogram of individual attention weight patterns for CMAP Data.}
\end{figure}

\begin{figure}[H]
\centering

\begin{subfigure}[t]{0.42\textwidth}
    \centering
    \includegraphics[width=\textwidth]
    {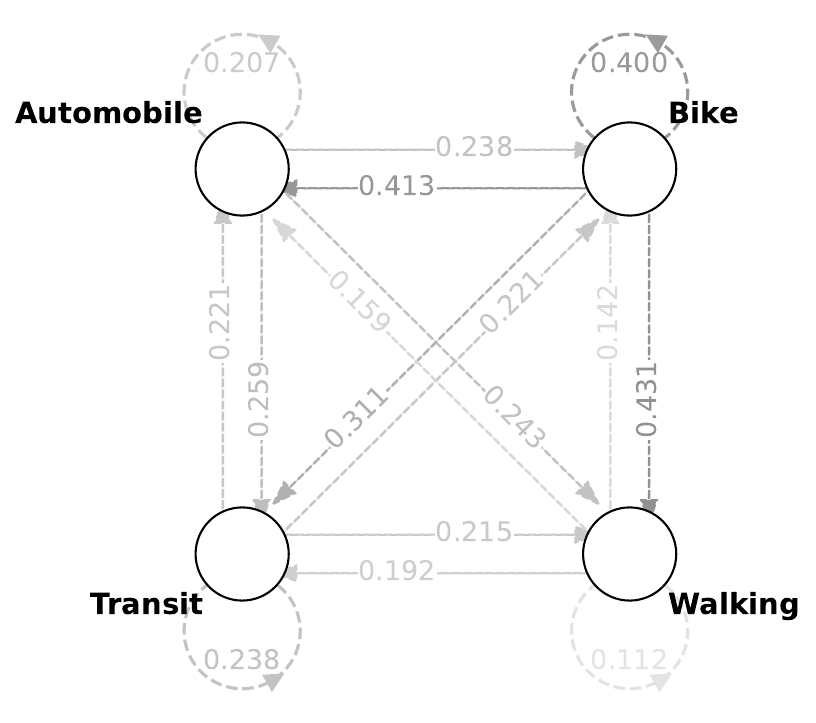}
    \caption{Cluster 1 ($n=2{,}865$)}
    \label{fig:cmap_attention_cluster1}
\end{subfigure}
\hfill
\begin{subfigure}[t]{0.42\textwidth}
    \centering
    \includegraphics[width=\textwidth]
    {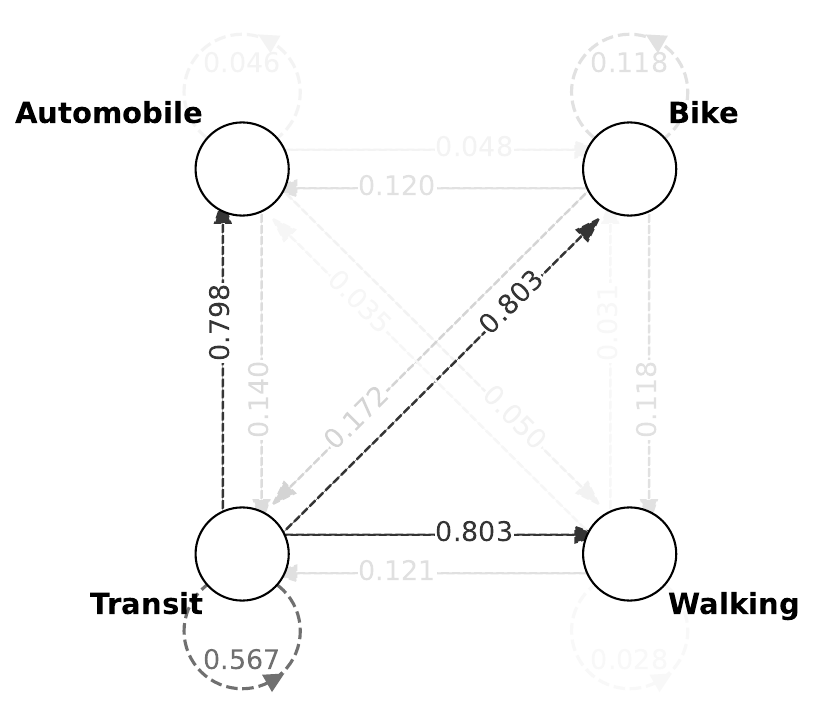}
    \caption{Cluster 2 ($n=1{,}294$)}
    \label{fig:cmap_attention_cluster2}
\end{subfigure}

\caption{Representative attention-graph structures obtained by clustering
individual-level attention weights for all observations in
the CMAP independent test set.}
\label{fig:cmap_attention_clusters}
\end{figure}

\begin{table}[htbp]
\centering
\small
\caption{Behavioral characteristics of the two clusters identified from
individual-level attention patterns in the CMAP independent test set}
\label{tab:cluster_profile_cmap}

\begin{tabular}{lccccc}
\toprule
& Age & Male (\%) & Vehicles & Household Size & Distance (m) \\
\midrule
Cluster 1 & 37.8 & 45.5 & 1.20 & 2.58 & 5040 \\
Cluster 2 & 48.0 & 44.0 & 1.69 & 2.53 & 12968 \\
\midrule
& Automobile Time & Transit Time & Walking Time & Bike Time & \\
\midrule
Cluster 1 & 11.2 & 31.0  & 81.7  & 24.0 & \\
Cluster 2 & 18.7 & 120.6 & 218.2 & 58.7 & \\
\midrule
& Automobile (\%) & Transit (\%) & Walking (\%) & Bike (\%) & \\
\midrule
Cluster 1 & 53.1 & 17.2 & 27.4 & 2.3 & \\
Cluster 2 & 92.2 & 4.9  & 1.9  & 1.0 & \\
\bottomrule
\end{tabular}
\end{table}

\section{\revII{Stability of the learned Top-K graphs}}
\label{app:graph_stability}

\revII{To verify that the Top-K graphs learned by the Alt-GNN model are more stable than randomly generated Top-K graphs, we calculated the Jaccard similarity distributions for the randomly generated Top-K graphs and the point estimate for the learned Top-K graphs. The Jaccard similarity is a metric used to calculate the similarity between a single pairwise of graphs:}

\begin{equation}
\revII{J(E_a,E_b) = \frac{|E_a\cap E_b|}{|E_a\cup E_b|}.}
\label{eq:graph_jaccard}
\end{equation} 

\revII{First, for the same Top-K specification, we trained with 20 random seeds to obtain 20 learned graphs, and calculated the average pairwise Jaccard similarity $J_1$ for these 20 graphs. Then, we use simulation to construct the distribution of random Top-K graph stability. Each simulation follows the same Top-K constraints, allowing each node to randomly select $k$ allowed edges, generating 20 random graphs. The average pairwise Jaccard similarity of these 20 graphs is calculated, yielding a $J_0^{(r)}$. This process is repeated 500,000 times to form the null distribution of the random graph's stability. Finally, the observed graph stability $J_1$ was compared against the null distribution of $J_0^{(r)}$.
The mean of the null distribution was calculated as
$\bar{J}_0 = \frac{1}{R}\sum_{r=1}^{R}J_0^{(r)}$,
and the chance-adjusted Jaccard similarity was defined as
$J_{\mathrm{adj}} = (J_1-\bar{J}_0)/(1-\bar{J}_0)$
to quantify the stability beyond random agreement.
The results are reported in \cref{tab:graph_stability}. Let $R$ denote the number of Monte Carlo simulations.
The corresponding right-tailed Monte Carlo $p$-value was calculated as:}

\begin{equation}
\revII{
p =
\frac{
1+\displaystyle\sum_{r=1}^{R}
\mathbb{I}\left(J_0^{(r)}\geq J_1\right)
}{
R+1
}.
}
\label{eq:graph_stability_pvalue}
\end{equation}

\revII{ The learned
Top-K graphs exhibit statistically significant but
modest stability beyond random edge selection,
providing evidence of some reproducibility across
random seeds. As shown in \cref{tab:graph_stability}, all four Top-K
specifications exhibited higher average Jaccard similarities
than their corresponding random graph baselines, with
statistically significant differences. For example, for the
first-ranked Top-K specification on LPMC, the observed
Jaccard similarity ($J_1=0.656$) exceeded the random
baseline ($\bar{J}_0=0.606$), with a right-tailed Monte Carlo
$p$-value below 0.001. The adjusted Jaccard similarity
($J_{\mathrm{adj}}=0.127$) indicates that the
observed agreement exceeded the random baseline
by 12.7\% of the remaining possible agreement.
However, the observed Jaccard similarity remains
well below 1, indicating that the learned graph
structures are not identical across random seeds.}

\begin{table}[htbp]
\centering
\caption{\revII{Stability of the two highest-ranked distinct Top-K specifications across 20 random seeds.}}
\label{tab:graph_stability}
\begin{tabular}{llcccc}
\toprule
\revII{Dataset} & \revII{Specification}
& \revII{$J_1$ (SD)}
& \revII{$\bar{J}_0$}
& \revII{$J_{\mathrm{adj}}$}
& \revII{$p$-value} \\
\midrule
\revII{LPMC} & \revII{1st-ranked: k3-max}
     & \revII{0.656 (0.105)} & \revII{0.606}
     & \revII{0.127} & \revII{$<0.001$} \\
\revII{LPMC} & \revII{2nd-ranked: k3-mean}
     & \revII{0.672 (0.120)} & \revII{0.606}
     & \revII{0.168} & \revII{$<0.001$} \\
\revII{CMAP} & \revII{1st-ranked: k3-max}
     & \revII{0.648 (0.112)} & \revII{0.606}
     & \revII{0.108} & \revII{$<0.001$} \\
\revII{CMAP} & \revII{2nd-ranked: k2-max}
     & \revII{0.384 (0.149)} & \revII{0.346}
     & \revII{0.058} & \revII{0.002} \\
\bottomrule
\end{tabular}
\end{table}

\section{\revII{Uncertainty in independent test-set performance}}
\label{app:test_uncertainty}

\revII{
To evaluate the stability of the performance of the models on the test set, we resampled the independent test set 5,000 times and calculated the performance metrics for each bootstrap resampling. The independent test set included 3,525 households in the LPMC and 562 households in the CMAP. In each bootstrap replicate, we sampled the same number of households as in the original test set with replacement and calculated the LL/obs, Macro F1, ACC, BACC, and BS values using the trained models' fixed predictions on the resampled observations.
This process was repeated 5,000 times to obtain the empirical distribution for each performance metric, and 95\% confidence intervals were calculated using the 2nd and 97.5th percentiles, as shown in Table ~\ref{tab:lpmc_test_uncertainty} and ~\ref{tab:cmap_test_uncertainty}.} 

\begin{table*}[htbp]
\centering
\footnotesize
\setlength{\tabcolsep}{4pt}
\renewcommand{\arraystretch}{1.12}
\caption{\revII{Household-cluster bootstrap 95\% confidence intervals for independent-test-set performance on the LPMC data.}}
\label{tab:lpmc_test_uncertainty}
\resizebox{\textwidth}{!}{%
\begin{tabular}{lccccc}
\hline
\revII{Model}
& \revII{LL/obs.}
& \revII{Macro F1}
& \revII{ACC}
& \revII{BACC}
& \revII{BS} \\
\hline
\revII{Attention Alt-GNN}
& \revII{$[-0.7206,-0.6725]$}
& \revII{$[0.5324,0.5499]$}
& \revII{$[0.7197,0.7415]$}
& \revII{$[0.5336,0.5516]$}
& \revII{$[0.3698,0.3951]$} \\
\revII{Top-K Alt-GNN ($k=3$, mean)}
& \revII{$[-0.7210,-0.6729]$}
& \revII{$[0.5334,0.5509]$}
& \revII{$[0.7217,0.7435]$}
& \revII{$[0.5343,0.5523]$}
& \revII{$[0.3694,0.3950]$} \\
\revII{Top-K Alt-GNN ($k=3$, max)}
& \revII{$[-0.7213,-0.6727]$}
& \revII{$[0.5330,0.5503]$}
& \revII{$[0.7211,0.7428]$}
& \revII{$[0.5341,0.5520]$}
& \revII{$[0.3695,0.3951]$} \\
\revII{Complete Alt-GNN}
& \revII{$[-0.7209,-0.6736]$}
& \revII{$[0.5313,0.5489]$}
& \revII{$[0.7191,0.7407]$}
& \revII{$[0.5323,0.5504]$}
& \revII{$[0.3708,0.3962]$} \\
\revII{Nested Alt-GNN}
& \revII{$[-0.7228,-0.6744]$}
& \revII{$[0.5320,0.5494]$}
& \revII{$[0.7201,0.7419]$}
& \revII{$[0.5324,0.5502]$}
& \revII{$[0.3709,0.3964]$} \\
\revII{ASU-DNN}
& \revII{$[-0.7232,-0.6755]$}
& \revII{$[0.5321,0.5493]$}
& \revII{$[0.7194,0.7412]$}
& \revII{$[0.5333,0.5509]$}
& \revII{$[0.3709,0.3962]$} \\
\revII{ResLogit}
& \revII{$[-0.7332,-0.6888]$}
& \revII{$[0.5282,0.5461]$}
& \revII{$[0.7123,0.7347]$}
& \revII{$[0.5302,0.5485]$}
& \revII{$[0.3805,0.4049]$} \\
\revII{TasteNet-MNL}
& \revII{$[-0.7510,-0.7053]$}
& \revII{$[0.5271,0.5447]$}
& \revII{$[0.7086,0.7308]$}
& \revII{$[0.5364,0.5546]$}
& \revII{$[0.3853,0.4092]$} \\
\revII{NL}
& \revII{$[-0.7737,-0.7276]$}
& \revII{$[0.5213,0.5391]$}
& \revII{$[0.7010,0.7235]$}
& \revII{$[0.5309,0.5490]$}
& \revII{$[0.3987,0.4229]$} \\
\revII{MNL}
& \revII{$[-0.7740,-0.7279]$}
& \revII{$[0.5219,0.5398]$}
& \revII{$[0.7016,0.7240]$}
& \revII{$[0.5314,0.5497]$}
& \revII{$[0.3989,0.4231]$} \\
\hline
\end{tabular}%
}
\end{table*}

\begin{table*}[htbp]
\centering
\footnotesize
\setlength{\tabcolsep}{4pt}
\renewcommand{\arraystretch}{1.12}
\caption{\revII{Household-cluster bootstrap 95\% confidence intervals for independent-test-set performance on the CMAP data.}}
\label{tab:cmap_test_uncertainty}
\resizebox{\textwidth}{!}{%
\begin{tabular}{lccccc}
\hline
\revII{Model}
& \revII{LL/obs.}
& \revII{Macro F1}
& \revII{ACC}
& \revII{BACC}
& \revII{BS} \\
\hline
\revII{Complete Alt-GNN}
& \revII{$[-0.7085,-0.5867]$}
& \revII{$[0.4833,0.5261]$}
& \revII{$[0.7371,0.7890]$}
& \revII{$[0.4728,0.5179]$}
& \revII{$[0.3100,0.3730]$} \\
\revII{Top-K Alt-GNN ($k=3$, max)}
& \revII{$[-0.7133,-0.5899]$}
& \revII{$[0.4848,0.5283]$}
& \revII{$[0.7405,0.7914]$}
& \revII{$[0.4727,0.5179]$}
& \revII{$[0.3095,0.3723]$} \\
\revII{Attention Alt-GNN}
& \revII{$[-0.7159,-0.5919]$}
& \revII{$[0.4907,0.5328]$}
& \revII{$[0.7395,0.7905]$}
& \revII{$[0.4862,0.5298]$}
& \revII{$[0.3089,0.3720]$} \\
\revII{Top-K Alt-GNN ($k=2$, max)}
& \revII{$[-0.7168,-0.5943]$}
& \revII{$[0.4851,0.5278]$}
& \revII{$[0.7366,0.7853]$}
& \revII{$[0.4795,0.5235]$}
& \revII{$[0.3118,0.3723]$} \\
\revII{Nested Alt-GNN}
& \revII{$[-0.7317,-0.6069]$}
& \revII{$[0.4626,0.5090]$}
& \revII{$[0.7288,0.7824]$}
& \revII{$[0.4531,0.4981]$}
& \revII{$[0.3191,0.3827]$} \\
\revII{ASU-DNN}
& \revII{$[-0.7339,-0.6078]$}
& \revII{$[0.4654,0.5109]$}
& \revII{$[0.7294,0.7826]$}
& \revII{$[0.4564,0.5015]$}
& \revII{$[0.3192,0.3835]$} \\
\revII{ResLogit}
& \revII{$[-0.7301,-0.6128]$}
& \revII{$[0.4669,0.5117]$}
& \revII{$[0.7177,0.7716]$}
& \revII{$[0.4608,0.5057]$}
& \revII{$[0.3300,0.3911]$} \\
\revII{TasteNet-MNL}
& \revII{$[-0.8071,-0.6785]$}
& \revII{$[0.3919,0.4408]$}
& \revII{$[0.6773,0.7396]$}
& \revII{$[0.3793,0.4244]$}
& \revII{$[0.3607,0.4241]$} \\
\revII{NL}
& \revII{$[-0.8140,-0.6824]$}
& \revII{$[0.3746,0.4241]$}
& \revII{$[0.6726,0.7345]$}
& \revII{$[0.3615,0.4063]$}
& \revII{$[0.3641,0.4283]$} \\
\revII{MNL}
& \revII{$[-0.8184,-0.6855]$}
& \revII{$[0.3524,0.4009]$}
& \revII{$[0.6643,0.7283]$}
& \revII{$[0.3482,0.3902]$}
& \revII{$[0.3650,0.4305]$} \\
\hline
\end{tabular}%
}
\end{table*}

\revII{Furthermore, to compare the predictive performance of Alt-GNN
models with each baseline, we calculated the paired differences
in LL/obs, defined as
$\Delta\mathrm{LL}^{(b)}
= \mathrm{LL}_{\mathrm{Alt\text{-}GNN}}^{(b)}
- \mathrm{LL}_{\mathrm{baseline}}^{(b)}$,
where $b=1,\ldots,B$ denotes the bootstrap replicate
and $B=5{,}000$.
The corresponding two-sided bootstrap $p$-value
was calculated as follows:}

\begin{equation}
\revII{
p =
\frac{
1+\displaystyle\sum_{b=1}^{B}
\mathbb{I}\left(
\left|\Delta\mathrm{LL}^{(b)}
-\widehat{\Delta\mathrm{LL}}\right|
\geq
\left|\widehat{\Delta\mathrm{LL}}\right|
\right)
}{
B+1
},
}
\label{eq:bootstrap_pvalue}
\end{equation}

\revII{
where $\widehat{\Delta\mathrm{LL}}$ denotes the observed
difference in LL/obs between an Alt-GNN model and a
baseline on the original independent test set, and
$\Delta\mathrm{LL}^{(b)}$ denotes the corresponding
difference in the $b$th bootstrap resample.
Subtracting $\widehat{\Delta\mathrm{LL}}$ centers the
bootstrap distribution at zero, allowing the
sampling variability of the paired differences
to approximate their distribution under the null
hypothesis of no performance difference.
The indicator function $\mathbb{I}(\cdot)$ identifies
bootstrap replicates in which the absolute centered
difference is at least as large as the observed
absolute difference. The add-one correction prevents
the estimated $p$-value from being zero.
The results of the pairwise comparisons are reported
in Table~\ref{tab:altgnn_baseline_pvalues}.
}

\begin{table*}[htbp]
\centering
\small
\setlength{\tabcolsep}{6pt}
\renewcommand{\arraystretch}{1.12}
\caption{\revII{Two-sided centered-null paired-bootstrap $p$-values for comparisons between Alt-GNN families and baseline models. Bold values indicate $p<0.05$. }}
\label{tab:altgnn_baseline_pvalues}
\begin{tabular}{llcccc}
\hline
\revII{Dataset}
& \revII{Baseline}
& \revII{Attention}
& \revII{Complete}
& \revII{Top-K}
& \revII{Nested} \\
\hline
\revII{LPMC}
& \revII{ASU-DNN}
& \revII{$0.1410$}
& \revII{$0.1632$}
& \revII{$0.0924$}
& \revII{$0.6311$} \\
\revII{LPMC}
& \revII{ResLogit}
& \revII{\textbf{$0.0002$}}
& \revII{\textbf{$0.0002$}}
& \revII{\textbf{$0.0002$}}
& \revII{\textbf{$0.0002$}} \\
\revII{LPMC}
& \revII{TasteNet-MNL}
& \revII{\textbf{$0.0002$}}
& \revII{\textbf{$0.0002$}}
& \revII{\textbf{$0.0002$}}
& \revII{\textbf{$0.0002$}} \\
\revII{LPMC}
& \revII{MNL}
& \revII{\textbf{$0.0002$}}
& \revII{\textbf{$0.0002$}}
& \revII{\textbf{$0.0002$}}
& \revII{\textbf{$0.0002$}} \\
\revII{LPMC}
& \revII{NL}
& \revII{\textbf{$0.0002$}}
& \revII{\textbf{$0.0002$}}
& \revII{\textbf{$0.0002$}}
& \revII{\textbf{$0.0002$}} \\
\hline
\revII{CMAP}
& \revII{ASU-DNN}
& \revII{$0.0568$}
& \revII{\textbf{$0.0002$}}
& \revII{\textbf{$0.0002$}}
& \revII{$0.6037$} \\
\revII{CMAP}
& \revII{ResLogit}
& \revII{\textbf{$0.0422$}}
& \revII{\textbf{$0.0030$}}
& \revII{\textbf{$0.0214$}}
& \revII{$0.7544$} \\
\revII{CMAP}
& \revII{TasteNet-MNL}
& \revII{\textbf{$0.0002$}}
& \revII{\textbf{$0.0002$}}
& \revII{\textbf{$0.0002$}}
& \revII{\textbf{$0.0002$}} \\
\revII{CMAP}
& \revII{MNL}
& \revII{\textbf{$0.0002$}}
& \revII{\textbf{$0.0002$}}
& \revII{\textbf{$0.0002$}}
& \revII{\textbf{$0.0002$}} \\
\revII{CMAP}
& \revII{NL}
& \revII{\textbf{$0.0002$}}
& \revII{\textbf{$0.0002$}}
& \revII{\textbf{$0.0002$}}
& \revII{\textbf{$0.0002$}} \\
\hline
\end{tabular}
\end{table*}

\revII{As shown in Table~\ref{tab:altgnn_baseline_pvalues}, the Alt-GNN models generally outperform traditional discrete choice models and several deep learning benchmarks, although this advantage is not statistically significant against every baseline. Specifically, all four Alt-GNN families consistently outperform MNL, NL, and TasteNet-MNL on both the LPMC and CMAP datasets, with all centered-null paired-bootstrap tests yielding $p=0.0002$. This is the minimum attainable $p$-value under the add-one correction with 5,000 bootstrap resamples and indicates that no centered bootstrap replicate was at least as extreme as the observed difference. Compared with ResLogit, all four Alt-GNN families achieve improvements in LL/obs ranging from 0.0124 to 0.0143 on LPMC ($p=0.0002$ for all comparisons). On CMAP, Attention, Complete, and Top-K Alt-GNN also outperform ResLogit, with centered-null bootstrap $p$-values of $0.0422$, $0.0030$, and $0.0214$, respectively. However, the advantages are less consistent relative to the stronger ASU-DNN benchmark. On LPMC, the improvements in LL/obs are small and not statistically significant ($p=0.0924$--$0.6311$). In contrast, Complete Alt-GNN and Top-K Alt-GNN outperform ASU-DNN by 0.0229 and 0.0194, respectively, on CMAP, with $p=0.0002$ for both comparisons. Overall, these results demonstrate that Alt-GNN offers competitive predictive performance and substantial improvements over traditional benchmark models, while its advantage over stronger deep learning models depends on the specific architecture and dataset.}

\end{appendices}

\end{document}